\documentclass{article}

\usepackage{authblk}
\usepackage{graphicx}
\usepackage{amssymb, amsmath, amsthm,mathtools}
\usepackage{epstopdf}
\usepackage{bbm}

\usepackage[marginratio=1:1,height=600pt,width=460pt,tmargin=100pt]{geometry}
\usepackage{layout, color}
\usepackage{bbm}
\usepackage{enumerate}

\usepackage{algorithm}
\usepackage{algpseudocode}

\usepackage{setspace}
\usepackage{enumitem}

\usepackage{array}
\usepackage{multirow}

\usepackage{hyperref}

\usepackage{url}
\usepackage{mdwlist}
\usepackage{multirow}
\usepackage{amsthm}
\usepackage{color}

\usepackage[font=scriptsize,labelfont=bf]{caption}
\usepackage[labelfont=scriptsize,labelfont=scriptsize]{subcaption}

\usepackage{cancel}

\newcommand{\Tr}{{\bf Tr}}
\newcommand{\calM}{ {\cal M}}

\newcommand{\calH}{{\cal H}}
\newcommand{\calL}{\mathcal{L}}
\newcommand{\R}{\mathbb{R}} 
\newcommand{\E}{\mathbb{E}}
\newcommand{\calN}{\mathcal{N}}
\newcommand{\Vol}{\mathrm{Vol}}
\newcommand{\Second}{\textup{I}\!\textup{I}}

\newtheorem{theorem}{Theorem}[section]

\newtheorem{assumption}{Assumption}

\newtheorem{lemma}[theorem]{Lemma}

\theoremstyle{remark}

\newtheorem{conjecture*}{Conjecture}
\theoremstyle{plain}

\usepackage{xcolor}

\begin{document}

\title{
Learning manifold diffusion semigroups from graph transition matrices
}

\author[1]{Xiuyuan Cheng\thanks{Email: xiuyuan.cheng@duke.edu.}}
\author[2]{Nan Wu}

\affil[1]{{\small Department of Mathematics, Duke University}}
\affil[2]{\small Department of Mathematical Sciences, The University of Texas at Dallas}

\date{\vspace{-30pt}}

\maketitle

\begin{abstract}

We consider graph diffusion processes constructed from finite i.i.d.\ samples drawn from an unknown manifold embedded in ambient Euclidean space, where the graph affinity is defined by an ambient Gaussian kernel matrix. We show that the manifold heat semigroup $Q_t = e^{t\Delta}$ can be approximated directly by iterating the graph transition matrix $P$, under only low regularity assumptions on the test function $f$, including the case $f \in L^\infty$. We bound $\| P^n f - Q_t f \|$ in $\infty$-norm, with the operator application to $f$ properly defined, and we recover the classical graph-Laplacian pointwise rate $O(N^{-2/(d+6)})$ up to logarithmic factors, for diffusion times $t $ up to $O(1)$ and longer. The rate holds for in-sample error as well as out-of-sample generalization, where the estimator of $Q_t f$ at a new point is defined via kernel convolution. To handle non-uniform sampling densities on the manifold, we introduce a right-normalization of the graph transition matrix; under the assumption that the sampling density $p$ is $C^3$ and bounded away from zero, the same convergence rates hold. We numerically demonstrate the performance of the proposed estimator on simulated data.

\end{abstract}

\section{Introduction}

\paragraph{Overview}

In many applications, observed data samples in high dimensional space lie on an embedded low-dimensional hidden manifold.
Learning manifold diffusion operators from sampled data is a fundamental problem in high-dimensional data analysis, machine learning, and graph-based solvers of partial differential equations on non-Euclidean domains.
Approximation of differential operators such as the Laplace-Beltrami operator $\Delta$ lies at the heart of spectral methods including Laplacian Eigenmaps \cite{belkin2003laplacian} and Diffusion Maps \cite{coifman2006diffusion}, which compute the eigen-decomposition of the graph Laplacian matrix.

In this work, we consider the heat semigroup $Q_t$ on $\calM$, which provides the solution map of the manifold heat equation, namely, $u(x,t) = Q_t f(x)$ solves the equation
\begin{equation}\label{eq:manifold-heat-eqn}
    \partial_t u = \Delta u,  \quad u|_{t=0} = f.
\end{equation}
One can formally write $Q_t = e^{t \Delta}$.
Let $\calH_t$ denote the manifold heat kernel,  the operator $Q_t$ can be expressed as
$
Q_t f(x) = \int_{\calM} \calH_t(x,y) f(y) dV(y),$
where $dV$ is the (local) Riemann volume form, and $(\calM, dV)$ gives a measure space. 
In other words, the manifold heat kernel $\calH_t$ provides the fundamental solution (Green's function) of the manifold heat equation.
Meanwhile, because the heat equation \eqref{eq:manifold-heat-eqn} coincides with the Fokker-Planck equation of the (no-drift) diffusion process, i.e. the Brownian motion, on $\calM$, we can also view $Q_t$ as the semigroup of the manifold diffusion process.

This work studies approximation of the manifold heat semigroup $Q_t$ by graph diffusion operators constructed from finite i.i.d. samples $\{x_i\}_{i=1}^N$ on $\calM$.
Specifically, we are to study
\begin{equation}\label{eq:operator-approximation-goal}
P^n \approx Q_t, \quad t = n \sigma^2, 
\end{equation}
where $P$ is a Markov transition matrix computed from $x_i$ using an affinity kernel of local scale $\sigma$, 
such that one step transition $P$ corresponds to a short-time diffusion of time $\epsilon = \sigma^2$. Specifically, $P = D^{-1} W$, where $D$ is a diagonal matrix normalizing the rows to be sum-one,
and $W$ is the graph affinity matrix constructed using a Gaussian kernel in the ambient Euclidean space, see details in Section \ref{subsec:set-up}.
Consequently, $n$-step transition $P^n$ approximates the semigroup at diffusion time $t = n\sigma^2$. 
Our theory covers the case when $t $ is up to $O(1)$ and even longer, technically, $t = O(\log^2 N)$.

We will prove the operator approximation \eqref{eq:operator-approximation-goal} in an operator {\it pointwise} sense, that is, for each test function $f$ on $\calM$, we will bound the difference  of  $\| P^n \rho_X f  -  \rho_X Q_{t } f \|_\infty$,
where  $\rho_X$ is the function evaluation operation, namely,
\[
\rho_X f := (f(x_1), \cdots, f(x_N)) \in \R^N.
\]
The problem is to establish quantitative convergence rates as the number of samples $N$ increases,
with $\sigma$ properly scaled (and $\sigma \to 0$ jointly).
Here, after restricting on samples, both $P^n \rho_X f $ and $\rho_X Q_{t } f $ are length-$N$ vectors,
 and  $\| \cdot \|_\infty$ is the vector infinity norm.
 We will also prove a bound at the same rate with out-of-sample extension, where the $\| \cdot \|_\infty$ is the functional $\infty$-norm on $\calM$.

A natural approach to the problem \eqref{eq:operator-approximation-goal} is through the convergence of graph Laplacians,
which has been intensively studied in literature, 
 including earlier works \cite{coifman2006diffusion, singer2006graph, hein2007graph, belkin2008towards},
 and later in \cite{singer2017spectral, trillos2020error, dunson2021spectral, calder2022improved, calder2022lipschitz, cheng2022eigen, trillos2025minimax}.
There are two types of results: the pointwise convergence and the spectral convergence. 
 The pointwise convergence 
 concerns how $L_N f \approx \calL f$ on a test function $f$, where $L_N$ is the graph Laplacian matrix
and $\calL$ is the manifold Laplacian operator, e.g., the Laplace-Beltrami operator $\Delta$.
Here, we omit the $\rho_X$ operator for notation brevity. 
The spectral convergence concerns how $\hat \lambda_l \approx \lambda_l$,
and $\hat u_l \approx \phi_l$, 
where  $\{ \hat \lambda_l, \hat u_l\}_l$ 
and $\{ \lambda_l, \phi_l\}_l$ 
are eigen-pairs of of the graph Laplacian matrix and the manifold Laplacian respectively.

Given the pointwise convergence of graph Laplacian, 
one may then attempt to derive convergence of graph diffusion from the approximation
\[
L_N := 
\frac{1}{\epsilon} (I- P)\approx -\Delta.
\]
However, pointwise graph Laplacian convergence requires differentiating the test function: 
the function $f$ at least needs to be in $C^2$ for $\Delta f$ to be well-defined.
Typically, existing convergence rates require high regularity assumptions such as $f$ is $C^{4}$ on $\calM$.
Iterating graph diffusion estimates while preserving optimal statistical rates would therefore require correspondingly strong regularity assumptions on $f$. 
On the other hand, $Q_t f$ is well-defined and in $C^\infty$ for any $t > 0$ even if $f$ only has low regularity (as long as $f$ is integrable). 
Thus, approximating the integral operator should require much lower regularity of the test $f$ than approximating the differential operator.

Our analysis closes this gap by  {\it directly} analyzing the finite-time diffusion semigroup $Q_t$ instead of approximating the infinitesimal generator $\Delta$.
We leverage the smoothing property of the heat semigroup and derives a multistep approximation strategy for graph diffusion operators,
avoiding differentiating the test function during the iteration procedure of the discrete diffusion process. 
Our theory allows the test function $f$ to be merely bounded, that is, $f \in L^\infty(\calM)$. 
Moreover, the resulting statistical convergence rates match those of classical pointwise graph Laplacian convergence.
We discuss in detail the attempt of using $\Delta f$ convergence rates to prove the $Q_t f$ approximation
in Section \ref{subsec:attempt-GL}.

Another possible route toward heat semigroup approximation is through spectral convergence of graph Laplacians \cite{dunson2021spectral,calder2022improved,cheng2022eigen,calder2022lipschitz,trillos2025minimax}. 
Recall that $\{ \lambda_l, \phi_l\}_l$ are the eigen-pair of  $-\Delta$, 
since the heat semigroup admits the spectral representation
\[
Q_t f
=
\sum_{l=0}^\infty e^{-\lambda_l t}\langle f, \phi_l\rangle_{L^2(\calM)} \phi_l.
\]
one may attempt to approximate $Q_t$ spectrally through graph Laplacian eigen-convergence.
However, controlling finite-time diffusion operators through this route requires accurate approximation over increasingly many eigenmodes as the target accuracy improves (as $N$ increases). 
Existing eigen-convergence theory is primarily formulated for finitely many low-lying eigenpairs, with constants depending on the spectral index. 
In contrast, our analysis establishes direct operator-level approximation of graph diffusion without requiring eigendecomposition or spectral truncation.
Another subtlety lies in the norm: 
existing sharp eigenvector convergence rates are primarily established in vector $2$-norm or weighted $L^2$ norms  \cite{trillos2025minimax}, whereas our result concerns approximation of $Q_t f$ in the $\infty$-norm.

In practice, our estimator only involves iterated multiplication of the $N$-by-$N$ transition matrix to the test function vector, without requiring eigen-decomposition. The matrix-vector updates can furthermore be implemented sparsely and computed on-the-fly with low memory cost.
While the focus of the current work is the statistical convergence, we demonstrate the numerical performance of the proposed estimator on simulated data.

\paragraph{Main results}
We establish bounds of $\|P^n f - Q_t f\|_\infty$ for diffusion times $t = n\sigma^2$ up to $O(1)$ and longer.
For notational simplicity, in the remainder of the introduction section, we omit the function evaluation operator $\rho_X$ 
in the expression.

Specifically, for general sampling density $p$ on $\calM$ that is $C^3$ and uniformly bounded from below, 
after introducing a density corrected transition matrix $\tilde P$, 
we prove in Theorem \ref{thm:pC3} that, as long as $ t = O( \log^2 N)$,
for any $f \in L^\infty(\calM)$, 
at large $N$ and with high probability,
\[
\|  \tilde P^n  f - Q_t f  \|_\infty = \| f\|_\infty O\Big( (t + \log n) \sigma^2 + (t + \sqrt{t}) \sqrt{ \frac{\log N}{N \sigma^{d+2}}} \Big),
\]
 where the constant in big-O only depends on $\calM$ and $p$.
Using $ t = O( \log^2 N)$
-- typically only $t \lesssim 1$ case is of interest --
the bound is further simplified into
\[
O \Big( \sigma^2 + \sqrt{ \frac{\log N}{N \sigma^{d+2}}} \Big) \text{ up to log factor,}
\]
which is $O(N^{- 2/(d+6)})$ when $\sigma\sim N^{-1/(d+6)}$.
 This matches the classical graph Laplacian pointwise convergence rate, which bounds $ \| L_N f - \Delta f \|_{\infty}$ for differentiable $f$.

 When the test function $f$ possesses additional regularity, the theoretical bound improves slightly.
Our analysis shows that when $f \in C^{0,\beta}(\calM)$, $ 0< \beta \le 1$,
the error bound improves by not having the $\log n$ factor (but with additional factors which are powers of $t$).
 The same final rate of $O(N^{- 2/(d+6)})$ is obtained by assuming $t \lesssim 1$.

 The above rates are first proved for on-sample errors, and furthermore, Theorem \ref{thm:out-of-sample} proves an out-of-sample extension result: 
 the graph diffusion approximation converges uniformly for new query points $x \in \calM$ not contained in the sampled dataset, again with the same statistical convergence rate. The functional estimator is naturally constructed using the graph  affinity kernel function and the extension analysis is standard.

For expository purposes,
our analysis first considers when the sampling density $p$ is uniform on $\calM$,  
 and prove the approximation rate of $\|  P^n  f - Q_t f  \|_\infty$  in Theorem \ref{thm:p-uniform}.
 The proof strategy then extends to the case of non-uniform $p$ in Section \ref{sec:theory-extend}.
Technically, key estimates include the pointwise approximation of the manifold heat kernel $\calH_\epsilon$ by the ambient Gaussian kernel $K_\epsilon$ on samples (in one-step, short time $\epsilon$), see Lemma \ref{lemma:Keps-g},
 and the manifold heat smoothing estimates (for multi-step analysis), see Lemma \ref{lemma:smoothing-estimates}.

\paragraph{Notation}

We use the standard big-O notation: 
as $N \to \infty$, $A_N = O(B_N)$ means $|A_N| \le C B_N$ for  some $C > 0$;
When $A_N$ is positive, we also write $A_N \lesssim B_N$.
$A_N \asymp B_N$ means $C_1 B_N \le A_N \le C_2 B_N$ for some  $C_1, C_2 > 0$, all holding for $N$ large enough. 
Our results are non-asymptotic: every bound holds for finite $N$ above an explicit threshold.
We track the dependence of the large-$N$ thresholds and the constants in big-O in proof. 
Typically, the dependence is on the manifold $\calM$ (and on $p$, when $p$ is non-uniform) and independent of the test function that the operator applies to.

\subsection{Attempt by using graph Laplacian pointwise convergence}\label{subsec:attempt-GL}

 Denote by $P = D^{-1} W$ the row-stochastic transition matrix, a typical pointwise rate
 shows that, for a function $f$ on $\calM$, with high probability,
\begin{equation}\label{eq:GL-rate-C4}
\frac{1}{\epsilon}(P - I) \rho_X f=  \rho_X \Delta f + O( \mathrm{Err}),
\end{equation}
and $\mathrm{Err}$ is an $o(1)$ term at large $N$, under the joint limit of $N \to \infty$ and $\epsilon \to 0$.
For example, when the test function $f \in C^4(\calM)$, one can show that \cite{singer2006graph}
\[
\mathrm{Err} =  \epsilon + \sqrt{\frac{\log N}{N \epsilon^{d/2+1}}},
\]
which is $O(N^{- 2/(d+6)})$ up to a log term when $\epsilon \sim N^{-2/(d+6)}$.
The bias error $O(\epsilon)$ degenerates if $f$ has lower regularity:  if $f \in C^{3}(\calM)$, then the bias error is $O(\epsilon^{1/2})$.

Assume $f \in C^4(\calM)$, then  \eqref{eq:GL-rate-C4} formally gives that 
\[
P = I + \epsilon \Delta + \epsilon O( \mathrm{Err} ).
\]
Combined with the formal expansion
\[
e^{\epsilon \Delta} = I + \epsilon \Delta + O(\epsilon^2 \Delta^2),
\]
and that $\epsilon \le  {\rm Err} $, we then have 
\[
e^{\epsilon \Delta} =  P + O(\epsilon {\rm Err} ).
\]
Taking $n$-th power, this gives that 
\[
e^{n \epsilon \Delta}  
= P^n + O( n \epsilon  {\rm Err}  ) 
= P^n + t O(   {\rm Err}  ).
\]
When $t$ is up to a $\log N$ term,
the final error $t O(   {\rm Err}  ) =  O(   {\rm Err}  )$, 
which recovers the pointwise rate of graph Laplacian convergence (up to a log term).

However, as can be seen from the Taylor expansion $e^{\epsilon\Delta}f = f + \epsilon\Delta f 
+ \frac{\epsilon^2}{2}\Delta^2 f + \cdots$, 
to bound the truncation error at the $O(\epsilon^2)$ term would involve $\Delta^2 f$, and requiring $f \in C^4$. 
A key feature of our approach is that we target the semigroup operator $Q_t$ directly, rather than the differential operator $\Delta$. 
 By instead comparing $P$ directly to $Q_\epsilon = e^{\epsilon\Delta}$, the $\Delta^2 f$ term never appears.
 In our analysis, the $O(\epsilon^2)$ one-step error follows from the parametrix expansion of the heat kernel $\mathcal{H}_\epsilon$, and it reflects the leading-order discrepancy between $\calH_\epsilon$ and the ambient Gaussian kernel $K_\epsilon$, see Lemma \ref{lemma:Keps-g}.
 While the $O(\epsilon^2)$ error is only obtained in Lemma \ref{lemma:Keps-g}(i) and requires $f$ to be in $C^2$, in our multi-step analysis we will only apply it to test function $Q_{ m \epsilon} f$ for $m >1$, which is already smooth. 
 The $C^2$ norm of the function $Q_{ m \epsilon} f$ is bounded by heat smoothing estimates. 
 Together, this allows to cover less regular initial function $f$ using the estimates in Lemma \ref{lemma:Keps-g}(ii)(iii) at the first step.

\paragraph{Regularity subtlety and our approach}
One might attempt to reduce to the smooth case by replacing $f$ with $Q_\delta f$ for small 
$\delta > 0$, since $Q_\delta f \in C^\infty$ instantly. However, the heat smoothing estimates 
give $\|\nabla^k Q_\delta f\|_\infty \lesssim \delta^{-k/2}\|f\|_\infty$, which blows up as 
$\delta \to 0$, and balancing this against the approximation error does not cleanly recover the 
$O(\mathrm{Err})$ rate. 
Our approach instead works directly at the level of the integral operator, telescoping the one-step bound $\|Pg - Q_\epsilon g\|_\infty$ 
over $n$ steps.
At step $m \ge 1$, the effective test function $g$ is $Q_{m\epsilon}f$, which has already been smoothed by the semigroup, and our one-step bound only depends on the $C^2$ norm of $g$.
The possible singularity near $t=0$ then appears through summations of the form $\sum_{m=1}^{n} (m \epsilon)^{-k/2}$ 
arising from heat smoothing estimates for $\|\nabla^kQ_t f\|_\infty$  at $t = m \epsilon$,
and these summations remain controllable since only $k=1,2$ are involved. 
Hence, the singular behavior near $t=0$ remains controllable
and no differentiability assumption on the initial $f$ is needed.

\subsection{Related works}

We already explain the relationship of our work to the manifold diffusion and graph Laplacian convergence literature.
This subsection gives some more discussion on  related studies.

\paragraph{Heat kernel estimation}

The manifold heat kernel $\mathcal{H}_t(x,y)$ has well-studied analytical properties, including parametrix expansions and Gaussian upper bounds \cite{rosenberg1997laplacian,grigor1997gaussian,hsu1999estimates}.
Note that the spectral representation $\calH_t(x,y) = \sum_{l=0}^\infty e^{-\lambda_l t}  \phi_l(x) \phi_l (y)$
suggests that one can approximate the population eigenpairs $\{\lambda_l, \phi_l\}_l$  from data 
via spectral convergence of graph Laplacians,
and control the approximation error using a suitable spectral truncation. 
This program is carried out in \cite{dunson2021spectral}, which establishes $L^\infty$ convergence rates for heat kernel estimation.
A complementary perspective is provided by the Diffusion Map \cite{coifman2006diffusion}, where columns of $P^n$
are interpreted as approximations of the heat kernel $\calH_{n \epsilon} (\cdot, x_j)$.
In line with that observation, our semigroup approximation framework naturally yields estimators for $\calH_t(x,y)$
that consist of powers of the transition matrix, 
and thus require no eigendecomposition
(Section \ref{subsec:heat-kernel-estimator}).
A complete convergence analysis of these heat kernel estimators goes beyond the scope of the present work.

\paragraph{Operator learning perspective}

From a broader perspective, the present work can also be viewed as a form of data-driven operator learning, which seeks to approximate infinite-dimensional operators, typically solution maps of PDEs, from finite data \cite{li2021fourier,kovachki2023neural}.
In contrast to black-box neural operator approaches, the graph diffusion operator studied here is directly constructed from sampled geometry and admits explicit statistical convergence guarantees in uniform norm.

\paragraph{Graph construction variants}
Many variants of graph diffusion operators have been developed, 
including 
$k$-nearest-neighbor graphs/adaptive-bandwidth kernels \cite{ting2010analysis, calder2022improved, cheng2022convergence, cheng2024improved}, diffusion with anisotropic kernels \cite{singer2009detecting, talmon2013empirical, berry2016local},
 bi-stochastic normalization \cite{marshall2019manifold,wormell2021spectral,cheng2024bi},
 and diffusion via landmark sets \cite{bermanis2016measure,long2017landmark, shen2022scalability}.
These constructions are often motivated by improved sampling adaptivity, 
noise robustness, computational scalability, or spectral approximation properties. 
The present work focuses on the fixed-bandwidth ambient Gaussian kernel setting, 
and we expect some of our techniques are transferrable to analyzing a broader class of graph constructions.

\section{Heat semigroup approximation}

\subsection{Set-up and assumptions}\label{subsec:set-up}

\begin{assumption}\label{assump:M}
${\calM}$ is a $d$-dimensional compact connected  
 $C^{\infty}$ manifold (without boundary) 
isometrically embedded in $\mathbb{R}^{D}$.
\end{assumption}
For a point $x \in \calM$, we also use $x$ to denote its embedded point in $\R^D$ when there is no confusion.   
$\| x - y\|$ stands for Euclidean distance in $\R^D$ between two points,
and $d_\calM(x,y)$ stands for the manifold geodesic distance.
Recall that $dV$ is the Riemannian volume form. We assume i.i.d. data samples.

\begin{assumption}\label{assump:iid-data}
Data samples $\{ x_i \}_{i=1}^N$ are i.i.d. drawn from a distribution on $\calM$ which has density $p$,  namely the data distribution is $pdV$. 
\end{assumption}
For simplicity, we start from the case where the data density  $p$ is uniform on $\calM$.
In this case, 
since $\Vol(\calM) = \int_{\calM} dV$ is finite (by the compactness of $\calM$),
$p$ equals the positive constant $\Vol(\calM)^{-1}$.
Our construction and theory naturally extend to the case when $p$ is non-uniform, see Section \ref{subsec:non-uniform-p}. 

We also need an assumption on  $\sigma$ as $N$ increases:
\begin{assumption}\label{assump:sigma-large-N}
As $N\to \infty$, $\sigma \to 0+$ and $ \sigma^{d} \gg \log N/ N$.
\end{assumption}
While Assumption \ref{assump:sigma-large-N} is about the asymptotic relation between $N$ and $\sigma$, our theory is non-asymptotic in the sense that our error bounds hold for finite $N$ large enough (and finite $\sigma$ small enough).
Recall that $\sigma^d \gtrsim \log N /N$ is the ``connectivity threshold'', namely when the graph become connected if we connect two points when their distance is less than $\sigma$. 
As to be detailed below, our construction will utilize Gaussian kernel to compute graph affinity, 
while the magnitude of $\sigma$ at large $n$ still indicates the density of the graph connection in terms of how many neighbors of a node have significantly non-zero affinities.

Given $N$ data samples $\{ x_i \}_{i=1}^N$ and a scale parameter $\sigma > 0$, 
we construct the $N$-by-$N$ graph {\it affinity matrix} $W$ as 
\begin{equation}\label{eq:gaussian-affinity-matrix}
    W_{ij} = \frac{1}{\sigma^d} h \Big( \frac{\| x_i - x_j \|^2}{\sigma^2}\Big),
    \quad h(r) = \frac{1}{(4\pi)^{d/2}} e^{-r /4},
\end{equation}
which is an ambient space Gaussian kernel matrix
(the Euclidean distance $\| x_i - x_j \|$ in $\R^D$ is used in the construction). 
The {\it degree matrix} $D$ is a diagonal matrix  defined as 
\[
D_{ii} = \sum_{j=1}^N W_{ij},  
\]
where $D_{ii}$ is strictly positive. 
The graph {\it transition matrix} $P$ is computed from $W$ by row-normalization, that is, 
\[
P = D^{-1} W.
\]
$P$ provides the transition matrix of a discrete-time diffusion process on the $N$ data points, which we call the ``graph diffusion process''. 
Intuitively, with large $N$, the graph diffusion process will approximate the underlying diffusion process (Brownian motion) on the manifold $\calM$.

To be specific, when the kernel scale $\sigma$ is small, the Gaussian kernel matrix in \eqref{eq:gaussian-affinity-matrix} locally approximates the manifold heat kernel $\calH_\epsilon$ of a short time $\epsilon = \sigma^2$:
Observe that  $W_{ij} = K_{\epsilon}(x_i, x_j)$, where 
\begin{equation}
K_\epsilon( x,y):= \frac{1}{(4 \pi \epsilon)^{d/2}} e^{- { \|x- y \|^2}/{4 \epsilon}}, \quad x, y \in \R^D.
\end{equation}
At short $ \epsilon = \sigma^2$, only nearby $x,y$ with $\| x-y\| \lesssim \sigma$ renders the kernel significantly nonzero.
Euclidean distance $\| x- y \|$ approximates the manifold geodesic distance $d_\calM(x,y)$ when  the two points are close, and then
$K_\epsilon(x,y) \approx G_\epsilon(x,y)$ defined as
\begin{equation}\label{eq:def-Geps}
G_\epsilon(x,y) := \frac{1}{(4 \pi \epsilon)^{d/2}} e^{- { d_\calM (x, y)^2}/{4 \epsilon}}, \quad x,y \in \calM.
\end{equation}
The kernel $G_\epsilon$ locally approximates the manifold heat kernel $\calH_\epsilon$ in the leading order, as a result of the manifold heat parametrix \cite{rosenberg1997laplacian,grigor1997gaussian}.
Putting together, we have
\[ 
K_\epsilon( x,y)  
\approx G_\epsilon(x,y) 
\approx \calH_\epsilon (x,y), \quad x, y \text{ nearby on $\calM$.}
\]
When the data density on $\calM$ is uniform, we then have
\[
\int_\calM K_\epsilon(x,y) f(y) p(y) dV(y)
\propto 
\int_\calM K_\epsilon(x,y) f(y) dV(y)
\approx \int_\calM \calH_\epsilon(x,y) f(y) dV(y)
= Q_\epsilon f(x).
\]
As a result, we expect that multiplying by the transition matrix $P$ approximates the heat semigroup operator $Q_\epsilon$ at $\epsilon = \sigma^2$,
and with $n$-th power the desired approximation \eqref{eq:operator-approximation-goal}.

\subsection{Manifold kernel integral estimates}\label{subsec:integral-operator-bias-analysis}

We first derive a few estimates about kernel integrals and the heat semigroup operator on $\calM$
(under Assumption \ref{assump:M}).
These estimates do not involve data samples,
and will be used to analyze the ``bias error'' at finite small $\sigma$.

We introduce the notions of function norms for functions on manifold. 
Let $\nabla$ denote the Riemannian connection, and $\nabla^k$ the k-th covariant derivative.
We also write $\nabla^k_v f(x) := \nabla^k f (x) (v, \cdots, v)$, for $v \in T_x \calM$.
We define
\[
\| \nabla^k f \|_\infty := \sup_{x \in \calM} \|\nabla^k f(x)\|_{op},
\quad \|\nabla^k f(x)\|_{op}:= \sup_{v \in S^{d-1} \subset T_x\calM} | \nabla^k_v f(x)|,
\]
and for $f \in C^k$, define
$
\| f\|_{C^k} := \sum_{l=0}^k \|\nabla^l f \|_\infty.
$
In particular,  for $ f \in C^1(\calM)$,
\[
\| f\|_{C^1}  = \| f\|_\infty + \| \nabla f \|_\infty,
\]
and for $ f \in C^2(\calM)$,
\[
\| f\|_{C^2}  = \| f\|_\infty + \| \nabla f \|_\infty + \| \nabla^2 f \|_\infty.
\]

We also consider the intrinsic H\"older class $C^{0,\beta}(\calM)$ for $0 < \beta \le 1$. Following the notion in \cite{tang2026adaptive},
let $\xi >0$ be the  injectivity radius  of $\calM$, we define
\[
\| f\|_{0,\beta} := \| f\|_\infty + L_{0,\beta}(f),
\quad 
L_{0, \beta}(f):=
\sup_{x \in \calM}  
\sup_{y \in B_{\xi}(x)}  {| f(x) - f(y)|}/{d^\beta_{\mathcal{M}}(x,y)},
\]
and we say $f$ is in $C^{0,\beta}(\calM)$ if $\| f\|_{0,\beta} < \infty$.
In particular, when $\beta = 1$, $f \in C^{0,1}(\calM)$ if and only if $f$ is globally Lipschitz continuous on $\calM$, and $L_{0,1}(f)$ is the (global) Lipschitz constant.  In addition, when $f \in C^1(\calM)$, then $f \in C^{0,1}(\calM)$ and $\| f\|_{C^1} = \| f\|_{0,1}$.
We also note that when $\beta =0$, formally, the $C^{0,0}(\calM)$ class recovers $L^\infty(\calM)$. In this work we write $L^\infty(\calM)$ separately for clarity.

The following lemma establishes the approximation of the manifold heat semigroup operator, applied to a test function $g$,
by the integral operator with the ambient Gaussian kernel $K_\epsilon$ on the manifold.
\begin{lemma}\label{lemma:Keps-g}
There exist positive constants 
$t_1 $, 
$C_4$, $C_5$, all only depending on $\calM$, 
s.t. when $\epsilon < t_1$,
\begin{itemize}
\item[(i)] For any $g \in C^2(\calM)$, 
\begin{equation*}
\int_\calM K_\epsilon( x,y) (g(y) - g(x)) dV(y)
= (Q_\epsilon g - g)(x) + r(x),
\quad \| r \|_{\infty} \le C_4 \| g\|_{C^2} \epsilon^2;
\end{equation*}

\item[(ii)] For any $g \in C^{0, \beta} (\calM)$, $0 < \beta \le 1$,
the residual $r$ satisfies that  $\| r \|_{\infty} \le C_5 \| g\|_{0,\beta} \epsilon^{1 + \beta/2}$;

\item[(iii)]  For any $g \in L^{\infty}(\calM)$, 
the residual $r$ satisfies that 
$\| r \|_{\infty} \le C_5 \| g\|_{\infty} \epsilon$.
\end{itemize}
\end{lemma}

The proof of the lemma adopts certain techniques in analyzing kernel integrals on manifold,
such as \cite[Lemma 4.1]{tang2026adaptive}.
Here, we need to be careful with the function norm dependence on $g$.

\begin{proof}[Proof of Lemma \ref{lemma:Keps-g}]
In all three cases, $g$ is always in $L^\infty(\calM)$, and then
\[
(Q_\epsilon g - g)(x) = \int_\calM \calH_\epsilon( x,y) (g(y)-g(x)) dV(y).
\]
Thus,
\[
r(x)=  \int_\calM (K_\epsilon(x,y) - \calH_\epsilon( x,y)) (g(y)-g(x)) dV(y),
\]
and we aim to bound $|r(x)|$ uniformly in $x$.
The proof consists of truncating the integral on a local (geodesic) ball around $x$, and showing that the tail contributes negligibly to the integral,
and then comparing $K_\epsilon(x,y)$ to $\calH_\epsilon(x,y)$ locally. 

We proceed in a few steps, and we first introduce the local truncation ball. 
Denote by $B_r(x)$ the local geodesic ball around $x$ of radius $r < \xi$. 
Let $t_0 <1$ and $\delta_0 < \xi$ be as in Lemma \ref{lemma:Heat-short-time}.
Define
$$
\delta(\epsilon) := \sqrt{2(d + 4)\epsilon \log\frac{1}{\epsilon}}.
$$ 
Recall that $\tau$ is the reach and $\xi$ is the injective radius of $\calM$,
and let $t_1'$ be the constant depending on $\calM$ s.t. 
\begin{equation}\label{eq:def-t1M}
\epsilon < t_1' \Rightarrow
2 \delta(\epsilon) < \min \{ \tau/2, \delta_0, 1\}; \quad \text{let } t_1:= \min\{ t_0, t_1'\},
\end{equation}
then $\epsilon < t_1$ also implies $\epsilon  < t_0 <1$.
This allows us to apply Lemma \ref{lemma:Heat-short-time} with 
\[
m = \lceil d/2+3 \rceil
\] 
and $t= \epsilon$, 
and Lemma \ref{lemma:Heat-short-time}(i) gives that $\forall x \in \calM$,
$\forall y \in B_{2\delta (\epsilon)}(x) \subset B_{\delta_0}(x)$,
\begin{equation}\label{eq:heat-kernel-lemma-1}
\begin{split}
& \calH_\epsilon (x,y) = G_\epsilon(x,y) \Big( 
	u_0(x,y) + \epsilon u_1(x,y) + r_2^H(\epsilon, x,y)
	\Big) + r^H(\epsilon, x,y),  \\
& |r_2^H(\epsilon, x,y)| \le C_2' \epsilon^2, 
	\quad |r^H(\epsilon,x,y)| \le C_2 \epsilon^{4}, 
\end{split}
\end{equation}
where $C_2'= \sum_{l=2}^m \| u_l\|_\infty$ is another constant depending on $\calM$,
and the uniform upper bound of $|r^H|$ uses that $m \ge d/2+3$.
In addition, Lemma \ref{lemma:Heat-short-time}(ii) gives that 
\begin{equation}\label{eq:heat-kernel-lemma-2}
\calH_\epsilon( x,y) \le C_3 \epsilon^{-d/2} e^{- d_\calM(x,y)^2/(5\epsilon)}, \quad \forall x, y \in \calM.
\end{equation}

\vspace{5pt}
\noindent
Step 1. truncation to local ball.
\vspace{5pt}

We consider $r(x)= I_1(x) + I_2(x)$, where 
\[
I_1(x) =  \int_{\calM  \backslash  B_{2\delta(\epsilon)}(x) }(K_\epsilon(x,y) - \calH_\epsilon( x,y)) (g(y)-g(x)) dV(y),
\]
\[
I_2(x) =  \int_{   B_{2\delta(\epsilon)}(x) }(K_\epsilon(x,y) - \calH_\epsilon( x,y)) (g(y)-g(x)) dV(y).
\]

We first bound $|I_1(x)|$.
For any $y \notin B_{2\delta(\epsilon)}(x)$, 
one can verify that $\epsilon^{-d/2} e^{- { d_\calM( x,y)^2}/{( 5 \epsilon)}} \le \epsilon^6$.
Then, by the upper bound \eqref{eq:heat-kernel-lemma-2},
we have $0 \le \calH_\epsilon(x,y) \le C_3 \epsilon^6$.

To upper bound $K_\epsilon(x,y)$, 
we observe that $y \notin B_{2\delta(\epsilon)}(x)$ implies that $\|  x - y\| \ge \delta(\epsilon)$:
Suppose otherwise, $\| x - y\| < \delta(\epsilon) < \tau/2$,
and by Lemma \ref{lemma:manifold-reach}, $d_\calM(x,y) \le 2 \| x-y\| < 2 \delta(\epsilon)$, contradicting with $y \notin B_{2\delta(\epsilon)}(x)$.
As a result,
\[
0 \le K_\epsilon(x,y) = \frac{1}{(4\pi \epsilon)^{d/2}} e^{- \| x-y\|^2/(4\epsilon)} \le (4\pi)^{-d/2} \epsilon^2.
\]

Putting together, we then have
\begin{align*}
|I_1(x)| 
& \le 
 \int_{\calM \backslash  B_{2\delta(\epsilon)}(x)} (K_\epsilon(x,y) + \calH_\epsilon(x,y) ) |g(y) - g(x)| dV(y) \\
& \le 2 \| g\|_\infty \Vol(\calM) ( C_3 \epsilon^6 + (4\pi)^{-d/2} \epsilon^2 ) \\
& \le  \| g\|_\infty C_{4,1} \epsilon^2, \quad C_{4,1} = 2 \Vol(\calM) ( C_3 + (4\pi)^{-d/2}  ),
\end{align*}
where in the last inequality we used that $\epsilon < 1$.

\vspace{5pt}
\noindent
Step 2. expansion of $\calH_\epsilon$ on local ball.
\vspace{5pt}

Restricting to $y \in B_{2\delta(\epsilon)}(x) $,
we focus for a moment on $\calH_\epsilon(x,y)$ and recall the expansion \eqref{eq:heat-kernel-lemma-1}.
We write $\calH_\epsilon(x,y) = \calH_{\epsilon}^H(x,y) + r^H(\epsilon, x,y)$, where
\[
 \calH_{\epsilon}^H(x,y) :=  G_\epsilon(x,y) \Big( 
	u_0(x,y) + \epsilon u_1(x,y) + r_2^H(\epsilon, x,y)
	\Big).
\]
Then, $I_2(x) = I_{2}^H(x)+I_3(x)$, where
\begin{align*}
I_{2}^H(x) 
	& = \int_{   B_{2\delta(\epsilon)}(x) }(K_\epsilon(x,y) - \calH_{\epsilon}^H( x,y)) (g(y)-g(x)) dV(y), \\
I_3(x) 
	& = - \int_{   B_{2\delta(\epsilon)}(x) }  r^H( \epsilon, x,y) (g(y)-g(x)) dV(y).
\end{align*}
We will handle $I_{2}^H(x)$ in Step 3. 
Here, using that $|r^H(\epsilon, x, y)| \le C_2 \epsilon^4$ uniformly, we have that 
\begin{align*}
|I_3(x)| 
& \le \int_{   B_{2\delta(\epsilon)}(x) }  |r^H( \epsilon, x,y)| |g(y)-g(x)| dV(y) \\
& \le  \| g\|_\infty  C_{4,2} \epsilon^4, \quad C_{4,2} := 2   \Vol(\calM) C_2  .
\end{align*}

\vspace{5pt}
\noindent
Step 3. comparing $K_\epsilon$ to $\calH_{\epsilon}$ on local ball.
\vspace{5pt}

To analyze $I_2^H(x)$, we revisit the expressions of $\calH_\epsilon^H$ and $K_\epsilon$.
Since $y \in  B_{2\delta(\epsilon)}(x)$, we use normal coordinates at $x$ and write $y = \exp_x( s \theta)$, where $ s = d_\calM(x,y)$ and $\theta \in S^{d-1} \subset T_x \calM \cong  \R^d$.
We have $ s < 2\delta(\epsilon)  < \delta_0 < \xi$.
Recall that 
\[ 
G_\epsilon(x,y) = \epsilon^{-d/2}h \Big( \frac{s^2}{\epsilon} \Big), \quad
K_\epsilon(x,y) = \epsilon^{-d/2}h \Big( \frac{\| x-y\|^2}{\epsilon} \Big),
\quad
h(r) = \frac{1}{(4\pi)^{d/2}} e^{-r /4}.
\]
The comparison of $K_\epsilon(x,y)$ to $\calH_\epsilon^H(x,y)$ thus consists of 
handling 

a) the difference $\calH_\epsilon^H - G_\epsilon$ by controlling the remainder terms in the expansion of $\calH_\epsilon^H$, and

b) the difference $G_\epsilon - K_\epsilon$ by tracking the influence of using $\|x-y\|$ instead of $s=d_\calM(x,y)$ inside $h$.

In this proof, we use big-O notation to indicate bound (instead of any asymptotic meaning):
 $Y = O(Z)$ means $|Y| \le C Z$ for some constant $C$,
 and we will track the constant dependence.
 In all the cases, the big-O notation is used to bound the remainder terms in absolute value,
 and the constants will be depending on $\calM$ and uniformly in $x$ and $\theta$.

We start with a). Because $u_0$, $u_1 \in C^{\infty}(\calM \times \calM)$,  
 and that  $u_0(x,y) = 1+O(s^2)$ by Lemma \ref{lemma:Heat-short-time}(i),  we have 
\[
 \calH_\epsilon^H(x,y)
 = \epsilon^{-d/2} h( \frac{s^2}{\epsilon})( 1 + O(s^2) + O(\epsilon) ).
\]
 For b), by Lemma \ref{lemma:local-comparisons}(ii), we have
$
\|x- y\|^2= s^2 + O(s^4)$, and then by Taylor expansion
\begin{equation*}
h(  \frac{ \| x-y\|^2}{\epsilon}) 
= h( \frac{s^2}{\epsilon}) 
    +h'( \frac{ \tilde r}{\epsilon}) \frac{O(s^4)}{\epsilon},
\end{equation*}
where $ \tilde r $ is between $s^2$ and $\|x-y\|^2$.
Recall that  we have $\| x - y\|\le s < 2\delta(\epsilon) < \tau/2$, then by Lemma \ref{lemma:manifold-reach}, 
$s/2\le  \| x-y \| $. Thus, $\tilde r \ge s^2/4$.
Recall that $h'(r)=-h(r)/4)$,
we have $|h'( \frac{ \tilde r}{\epsilon})| \le \frac{1}{4} h( \frac{s^2}{ 4 \epsilon} )$. This gives that 
\[
K_\epsilon(x,y)
= \epsilon^{-d/2}  h( \frac{s^2}{\epsilon}) 
    + \epsilon^{-d/2}  h( \frac{s^2}{ 4 \epsilon} ) \frac{O(s^4)}{\epsilon}.
\]
We then have
\[
K_\epsilon(x,y) -  \calH_\epsilon^H(x,y)
= \epsilon^{-d/2}  
  h( \frac{s^2}{4 \epsilon}) O(s^2 + \epsilon  + \frac{s^4}{\epsilon}),
\]
where we used that  $h(r) \le h(r/4)$ to simplify the bound.
This estimate would allow us to prove the cases (ii) and (iii) of $g$, which has lower regularity. 

Specifically, we always have
\begin{align*}
| I_{2}^H(x)  |  \le  \int_{   B_{2\delta(\epsilon)}(x) } | K_\epsilon(x,y) - \calH_{\epsilon}^H( x,y))| | g(y)-g(x)| dV(y).
\end{align*}
The volume form  $dV(\exp_x(s\theta)) = V(x,\theta,s) s^{d-1}ds d\theta$,
 and by Lemma \ref{lemma:local-comparisons}(i), 
 we know that $0 \le V(x,\theta,s) \le b_V$
 for some constant $b_V$ depending on $\calM$. As a result, when $g \in L^\infty(\calM)$, we have
\begin{align*}
| I_{2}^H(x)  |  
& \le  2 \| g\|_\infty b_V
	\int_0^{2\delta(\epsilon)}
	\int_{S^{d-1}} | K_\epsilon(x,y) - \calH_{\epsilon}^H( x,y))| s^{d-1}ds d\theta \\
& = 2 \| g\|_\infty b_V |S^{d-1}|
	 \epsilon^{-d/2}
	\int_0^{2\delta(\epsilon)}
	  h( \frac{s^2}{4 \epsilon}) O(s^{2} + \epsilon + \frac{s^4}{\epsilon})  s^{d-1} ds. 
\end{align*}
To bound the integral, we change the variable to \[
u := s/\sqrt{\epsilon},
\]
and then
\[
 \epsilon^{-d/2}
	\int_0^{2\delta(\epsilon)}
	  h( \frac{s^2}{4 \epsilon}) O(s^{2} + \epsilon + \frac{s^4}{\epsilon})  s^{d-1} ds
=O( \epsilon) \int_0^{ 2 \sqrt{ 2 (d+4)  \log \frac{1}{\epsilon} } } h(\frac{u^2}{4})  (u^4+ u^2 + 1)    u^{d-1} du.
\]
Suppose $O(\epsilon)$ term is upper bounded in absolute value by $c_{A,1}\epsilon$, 
the constant  $c_{A,1}$ only depends on $\calM$ as it can be tracked to the expansions of $\calH_\epsilon^H(x,y)$ and $K_\epsilon(x,y)$.
Now,
\begin{align*}
& \int_0^{ 2 \sqrt{ 2 (d+4)  \log \frac{1}{\epsilon} } } h(\frac{u^2}{4})  (u^4+ u^2 + 1)    u^{d-1} du
 \le \int_0^{ \infty} h(\frac{u^2}{4})  (u^4+ u^2 + 1)    u^{d-1} du \\
&~~~ 
= (4\pi)^{-d/2}\int_0^{ \infty} e^{-u^2/16}   (u^4+ u^2 + 1)    u^{d-1} du  \\
&~~~
 \le(4\pi)^{-d/2}\int_0^{ \infty} e^{-u^2/16}   \sum_{l=0}^5 u^{l+d-1} du   =: m_{d,1}.
\end{align*}
Putting together, this shows that 
\[
| I_{2}^H(x)  |  
\le 2 \| g\|_\infty b_V |S^{d-1}| c_{A,1} m_{d,1}  \epsilon. 
\]
Let $C_{5,3} = 2  b_V |S^{d-1}| c_{A,1} m_{d,1} $,
then $| I_{2}^H(x)  |   \le \| g\|_\infty C_{5,3}  \epsilon$.

Putting together the bounds of $|I_1|$, $|I_2^H|$, $|I_3|$, and note that $\epsilon<1$, we have that 
\begin{align*}
|r(x)|
& \le |I_1(x)|  + |I_3(x)| + |I_2^H(x)| \\
& \le \| g\|_\infty C_{4,1} \epsilon^2
	+ \|g\|_\infty C_{4,2} \epsilon^4 +  \| g\|_\infty C_{5,3} \epsilon  \\
& \le \| g\|_\infty C_{5} \epsilon, \quad C_5 := 	C_{4,1} + C_{4,2} +C_{5,3}.
\end{align*}
The constant $C_5$ only depends on $\calM $ (and $d$) and is uniform in $x$. This proves case (iii) of the lemma.

\

For case (ii), $g \in C^{0,\beta}(\calM)$. Since $s = d_\calM(x,y) < 2\delta(\epsilon) < \delta_0 < \xi$, we have
\[ 
|g(y) - g(x)| \le L_{0,\beta}(g) s^\beta.
\]
Then, together with that $dV(y) \le b_V s^{d-1} ds d\theta$, we have
\begin{align*}
| I_{2}^H(x)  |
& \le L_{0,\beta}(g)  b_V
	\int_0^{2\delta(\epsilon)}
	\int_{S^{d-1}} | K_\epsilon(x,y) - \calH_{\epsilon}^H( x,y))| s^{\beta+d-1}ds d\theta  \\
& = L_{0,\beta}(g)  b_V |S^{d-1}| 
	\epsilon^{-d/2}  
	\int_0^{2\delta(\epsilon)}
	  h( \frac{s^2}{4 \epsilon}) O(s^2 + \epsilon  + \frac{s^4}{\epsilon})
  	 s^{\beta+d-1}ds \\
& \le L_{0,\beta}(g)  b_V |S^{d-1}|  c_{A,1} \epsilon^{\beta/2+1}
	\int_0^{ 2 \sqrt{ 2 (d+4)  \log \frac{1}{\epsilon} } } 
	h(\frac{u^2}{4})  (u^4+ u^2 + 1)    u^{\beta+d-1} du,
\end{align*}
where in the last inequality we used the constant $c_{A,1}$ to bound the big-O term again.
To upper bound the integral in $du$, we again relax to $\int_0^\infty$, and note that since $0<\beta \le 1$,
we have $u^{\beta} \le 1+ u$, $\forall u \ge 0$.
Then, 
\[
\int_0^{ 2 \sqrt{ 2 (d+4)  \log \frac{1}{\epsilon} } } 
	h(\frac{u^2}{4})  (u^4+ u^2 + 1)    u^{\beta+d-1} du
\le
\int_0^{ \infty} 
	h(\frac{u^2}{4})  (u^4+ u^2 + 1)   (1+u) u^{d-1} du \le m_{d,1}.
\]
This shows that
\[
| I_{2}^H(x)  |  
\le  L_{0,\beta}(g)  b_V |S^{d-1}|  c_{A,1} m_{d,1} \epsilon^{\beta/2+1}
\le L_{0,\beta}(g)  C_{5,3}  \epsilon^{\beta/2+1}.
\]
Combined with the bounds of $|I_1|$, $|I_3|$ and that $\epsilon<1$, $ 0< \beta \le 1$, this gives that 
\begin{align*}
|r(x)|& \le \| g\|_\infty C_{4,1} \epsilon^2
	+ \|g\|_\infty C_{4,2} \epsilon^4 + L_{0,\beta}(g)  C_{5,3}  \epsilon^{\beta/2+1}  \\
& \le \| g\|_{0,\beta} C_{5} \epsilon^{\beta/2+1},
\end{align*}
which proves case (ii) of the lemma.

 \
 
 To prove the case for $g \in C^2(\calM)$, we need to refine the above estimates ``to the next order''. 
 Specifically, for the expansion of $\calH_\epsilon^H$ in a),  
 we have
 \begin{align*}
 u_0(x,y) &= 1 + q_2^U(x,\theta)s^2 + O(s^3), \quad q_2^U(x,\theta) = \frac{1}{12} \texttt{Ric}_{x}(\theta,\theta),\\
 u_1(x,y) &= u_1(x,x) + O(s),
 \end{align*}
 Thus,
 \[
 \calH_\epsilon^H(x,y) 
 	= \epsilon^{-d/2} h( \frac{s^2}{\epsilon})
	( 1 +  q_2^U(x,\theta) s^2  
  	+ \epsilon u_1(x,x) + O(s^3  + \epsilon  s +\epsilon^2)).
 \]
 In b), we also use the refined expansion of $\| x-y\|^2$ in Lemma \ref{lemma:local-comparisons}(ii):
 with $q_4^M( x,\theta) = -\frac{1}{12} \|\Second_x(\theta,\theta)\|^2 $,
 \[
\|x- y\|^2= s^2 +q_4^M(x,\theta) s^4 + r_5 =: s^2 +  r_4,
\]
where $r_5 = O( s^5)$, $r_4 =O(s^4)$,
and $q_4^M(x,-\theta) = q_4^M(x,\theta)$.

Recall that $h'(r) = -h(r)/4$, and $h^{(2)}(r) = h(r)/16$.
This gives the refined Taylor expansion
\begin{equation*}
h(  \frac{s^2+ r_4}{\epsilon}) 
= h( \frac{s^2}{\epsilon}) 
        +
	h'( \frac{s^2}{\epsilon}) \frac{r_4}{\epsilon}
	+ \frac{1}{2} h^{(2)}( \frac{ \tilde r'}{\epsilon} ) (\frac{r_4}{\epsilon})^{2},
\end{equation*}
where $ \tilde r' $ is  again between $s^2$ and $\|x-y\|^2$.
For the same reason as before, $\tilde r' \ge s^2/4$, and then $
| h^{(2)}(  \frac{ \tilde r'}{\epsilon} )  |
= \frac{1}{16}h (  \frac{ \tilde r'}{\epsilon} ) 
\le \frac{1}{16}h (  \frac{ s^2 }{4 \epsilon} )  $.
Combined with the bounds of $|r_4|$, $|r_5|$, we have
\[
K_\epsilon(x,y) 
 = \epsilon^{-d/2} h(\frac{s^2}{\epsilon}) 
 	\Big( 1 -  \frac{1}{4} q_4^M  (x,\theta)  \frac{s^4}{\epsilon}  \Big)  
	+  \epsilon^{-d/2}  h( \frac{s^2}{4\epsilon} ) O( \frac{s^5}{\epsilon}  + \frac{s^8}{\epsilon^2} ) ,
\]
where we used $h(r) \le h(r/4)$ to simplify the big-O term.
In both the expansions of $\calH_\epsilon^H$ and $K_\epsilon$,
 all the constants in big-O depend on $\calM$ only and is uniform for $x$, $\theta$.
 Subtracting the two, 
 and letting  $ Q_0(x) = - u_1(x,x) $,
 $Q_2 (x,\theta) = -q_2^U(x,\theta)$,
 $Q_4(x,\theta) = -\frac{1}{4} q_4^M(x,\theta)$,
 we have
 \begin{align*}
 K_\epsilon(x,y) -  \calH_\epsilon^H(x,y)
& =  J(\epsilon, x,y)   + R_{A,2}(s  , \theta )\\ 
J(\epsilon, x,y) 
& := \epsilon^{-d/2}  
	  		h( \frac{s^2}{ \epsilon}) 
			  \Big(Q_4(x,\theta) \frac{s^4}{\epsilon} 
			   + Q_2 (x,\theta) s^2  +  \epsilon Q_0(x)  \Big)\\
|R_{A,2}(s   , \theta)| 
& \le c_{A,2}
   \epsilon^{-d/2}  h( \frac{s^2}{4\epsilon} ) 
  (  \frac{s^8}{\epsilon^2} + \frac{s^5}{\epsilon}  + s^3  + \epsilon  s +\epsilon^2), \\
|J(\epsilon, x,y)|
& \le    c_{A,2}  \epsilon^{-d/2}  h( \frac{s^2}{4\epsilon} )
 		(\frac{s^4}{\epsilon}  + s^2  +  \epsilon),
 \end{align*}
where we used $h(r) \le h(r/4)$  again to simplify the bound,
and the constant  $c_{A,2}$ only depends on $\calM$.
In addition, $Q_j$ for $j=0,2,4$, are all even in $\theta$, in the sense that $Q_j(x, -\theta) = Q_j(x,\theta)$.

We now write
\begin{align*}
I_2^H(x) 
 & = I_4(x) + I_5(x), \\
I_4(x)& = \int_{   B_{2\delta(\epsilon)}(x) }
			J(\epsilon, x,y)
			(g(y)-g(x)) dV(y),\\
I_5(x)& = \int_{   B_{2\delta(\epsilon)}(x) }
		R_{A,2}(s   , \theta)
		(g(y)-g(x)) dV(y).
\end{align*}
By the same argument of bounding $|I_2^H(x)|$ for $g \in C^{0,\beta}(\calM)$ above, we can show that 
\begin{align*}
|I_5(x)| 
& \le \int_{   B_{2\delta(\epsilon)}(x) }
		 | R_{A,2}(s  , \theta) |
		| g(y)-g(x)| dV(y) \\
& \le \| \nabla g\|_\infty b_V 
	\int_0^{2\delta(\epsilon)}
	\int_{S^{d-1}} 
		 | R_{A,2}(s  , \theta) |
		s^{d} ds d\theta \\
& \le \| \nabla g\|_\infty b_V  |S^{d-1}| c_{A,2}
	\int_0^{2\delta(\epsilon)}
	   \epsilon^{-d/2}  h( \frac{s^2}{4\epsilon} ) 
	  (  \frac{s^8}{\epsilon^2} + \frac{s^5}{\epsilon}  + s^3  + \epsilon  s +\epsilon^2)
		s^{d} ds,  
\end{align*}
and
\begin{align*}
& \int_0^{2\delta(\epsilon)}
	   \epsilon^{-d/2}  h( \frac{s^2}{4\epsilon} ) 
	  (  \frac{s^8}{\epsilon^2} + \frac{s^5}{\epsilon}  + s^3  + \epsilon  s +\epsilon^2)
		s^{d} ds
 \le \int_0^{\infty}
	   h( \frac{u^2}{4} ) 
	  (  \epsilon^2 (u^8+1) + \epsilon^{3/2} (u^5 + u^3 + u)  ) \sqrt{\epsilon} 
		u^{d} du \\
& \le \epsilon^2 \int_0^{\infty}
	   h( \frac{u^2}{4} ) 
	   \sum_{l=1}^9 u^{l+d-1} du =: \epsilon^2 m_{d,2},
\end{align*}
where we used $\epsilon <1$ in the last inequality.
This shows that 
\[
|I_5(x)|  \le  \| \nabla g\|_\infty b_V  |S^{d-1}| c_{A,2}  m_{d,2}  \epsilon^2.
\]
Let $C_{4,3} := (b_V + c_V)  |S^{d-1}| c_{A,2}  m_{d,2} $,
which will also be used below to bound $|I_{4}(x)|$, 
we have $|I_5(x)|  \le  \| \nabla g\|_\infty C_{4,3} \epsilon^2$.

To bound $|I_4(x)|$, we will also need to expand $dV(y)$ and $g(y)-g(x)$.
By Lemma \ref{lemma:local-comparisons}(i), we have
\[
dV(y)= (1+ r_V(s  , \theta)  )s^{d-1} ds d\theta, \quad |r_V(s  , \theta)| \le c_V s^2.
\]
where $c_V$ depends on $\calM$ only and is uniform for $x$, $\theta$.
Since $g \in C^2(\calM)$, we have
\begin{equation}\label{eq:expansion-g-C2}
g(y)= g(x) +  \nabla_\theta g(x) s + r_2^F(s , \theta) , \quad |r_2^F(s , \theta)| \le \| \nabla^2 g\|_\infty s^2.
\end{equation}
Then, we have 
\begin{align*}
I_4(x) 
& = I_{4,1}(x) +  I_{4,2}(x)+ I_{4,3}(x), \\
I_{4,1}(x) 
& = 
	\int_0^{2\delta(\epsilon)}
			\int_{S^{d-1}} 
			J(\epsilon, x,y)
			\nabla_\theta g(x)   
			s^{d} ds d\theta, \\
I_{4,2}(x) 
& = \int_0^{2\delta(\epsilon)}
			\int_{S^{d-1}} 
			J(\epsilon, x,y)
			 \nabla_\theta g(x) 
			r_V(s  , \theta) s^{d} ds d\theta,	\\	
I_{4,3}(x) 
& = \int_{   B_{2\delta(\epsilon)}(x) } 
			J(\epsilon, x,y)
			 r_2^F(s , \theta)
			dV(y). 
\end{align*}
Observe that 
\[
I_{4,1}(x) 
 = 
	\int_0^{2\delta(\epsilon)}
			\int_{S^{d-1}} 
			\epsilon^{-d/2}  
	  		h( \frac{s^2}{ \epsilon}) 
			  \Big(Q_4(x,\theta) \frac{s^4}{\epsilon} 
			   + Q_2 (x,\theta) s^2  +  \epsilon Q_0(x)  \Big)
			\nabla_\theta g(x)   
			s^{d} ds d\theta, 
\]
and recall that $Q_0$, $Q_2$, $Q_4$ are even in $\theta$,
while $\nabla_\theta g(x)$ is odd in $\theta$. As a result, the integration of $d\theta$ over $S^{d-1}$ renders $I_{4,1}(x)=0$.

The terms $|I_{4,2}|$ and $|I_{4,3}|$ will be bounded by the same argument as above:
One can verify that
\begin{align*}
|I_{4,2}(x)|
& \le \int_0^{2\delta(\epsilon)}
			\int_{S^{d-1}} 
			|J(\epsilon, x,y)|
			  |\nabla_\theta g(x)|  
			|r_V(s  , \theta)| s^{d} ds d\theta \\
&\le \| \nabla g\|_\infty c_V  |S^{d-1}|  c_{A,2}
			 \int_0^{2\delta(\epsilon)} 
			  \epsilon^{-d/2}  h( \frac{s^2}{4\epsilon} )
 		(\frac{s^4}{\epsilon}  + s^2  +  \epsilon)
			 s^{d+2} ds  \\
& \le \| \nabla g\|_\infty c_V  |S^{d-1}|  c_{A,2} \epsilon^{5/2}
			 \int_0^{\infty} 
			   h( \frac{u^2}{4} )
	 		 (u^4   + u^2  +  1) 
			 u^{d+2}  du  \\		
& \le \| \nabla g\|_\infty c_V  |S^{d-1}|  c_{A,2} m_{d,2} \epsilon^{5/2},		 	
\end{align*}
\begin{align*}
|I_{4,3}(x)|
& \le \int_{   B_{2\delta(\epsilon)}(x) } 
			|J(\epsilon, x,y)|
			 |r_2^F(s  , \theta)|
			dV(y) \\
& \le \| \nabla^2 g\|_\infty b_V |S^{d-1}| c_{A,2} 
	 \int_0^{2\delta(\epsilon)} 
			 \epsilon^{-d/2}  h( \frac{s^2}{4\epsilon} )
 			(\frac{s^4}{\epsilon}  + s^2  +  \epsilon)
						s^{d+1} ds \\
& \le \| \nabla^2 g\|_\infty b_V |S^{d-1}| c_{A,2} \epsilon^2
		 \int_0^{\infty} 
			 h( \frac{u^2}{4} )
 			(u^4  + u^2  +  1)u^{d+1}du	\\
& \le \| \nabla^2 g\|_\infty b_V |S^{d-1}| c_{A,2}  m_{d,2} \epsilon^2.				
\end{align*}
Putting together, and by  $\epsilon < 1$, we have
\[
|I_4(x)| 
\le 
|I_{4,2}(x)| + |I_{4,3}(x)|
\le 	(\| \nabla g\|_\infty + \| \nabla^2 g\|_\infty)
	C_{4,3} \epsilon^2.
\]
This then gives that 
\[
|I_2^H(x)| \le |I_4(x)| + |I_5(x)| \le 
(2 \| \nabla g\|_\infty + \| \nabla^2 g\|_\infty)
	C_{4,3} \epsilon^2.
\]
Finally, we have
\begin{align*}
|r(x)|
& \le |I_1(x)|  + |I_3(x)| + |I_2^H(x)| \\
& \le \| g\|_\infty C_{4,1} \epsilon^2
	+ \|g\|_\infty C_{4,2} \epsilon^4 +
		  (2 \| \nabla g\|_\infty + \| \nabla^2 g\|_\infty) C_{4,3} \epsilon^2 \\
& \le \| g\|_{C^2} C_{4} \epsilon^2, \quad C_4 := 	C_{4,1} + C_{4,2} +2C_{4,3}.
\end{align*}
This proves case (i) of the lemma.
\end{proof}

We recall that $Q_t$ is contractive in $L^\infty$ norm (Lemma \ref{lemma:Qt-contractive-Linf}).
The next lemma provides some estimates of the 
heat semigroup operator $Q_t$ applied to a test function $g$.

\begin{lemma}\label{lemma:Qeps-g-estimate}
\begin{itemize}
\item[(i)] 
$\forall t > 0$,   $\forall g \in C^2(\calM)$, $ \| Q_t g - g \|_\infty \le d \| \nabla^2 g\|_\infty t $.
\item[(ii)] 
Let $t_1$ be as in Lemma \ref{lemma:Keps-g},
there exists a constant $C_6$ depending on $\calM$ s.t.,
 $\forall \epsilon < t_1$, $\forall g \in C^{0,\beta}(\calM)$, $0 < \beta \le 1 $,
$ \| Q_\epsilon g - g \|_\infty \le C_6 \| g \|_{0,\beta} \epsilon^{\beta/ 2} $.
In particular, if $g \in C^1(\calM)$, then 
$ \| Q_\epsilon g - g \|_\infty \le C_6  \| g \|_{C^1} \epsilon^{1/ 2} $.
\item[(iii)] 
$\forall t > 0$,   $\forall g \in L^\infty(\calM)$, $ \| Q_t g - g \|_\infty \le 2 \|  g\|_\infty $.
\end{itemize}
\end{lemma}

\begin{proof}[Proof of Lemma \ref{lemma:Qeps-g-estimate}]
For any $t > 0$, we have that for any $x \in \calM$,
\begin{equation}\label{eq:Qtg-g-proof-1}
(Q_t g-g)(x) = \int_\calM \calH_t(x,y) (g(y) - g(x)) dV(y).
\end{equation}
Then, for $g \in L^\infty(\calM)$, we have
\[
|(Q_t g-g)(x) |\le \int_\calM \calH_t(x,y) |g(y) - g(x)| dV(y)
\le 2 \| g\|_\infty,
\]
and this proves (iii).

\

When $g \in C^2(\calM)$, we have $Q_t g(x) = u(t,x)$ satisfies the manifold heat equation $\partial_t u = \Delta u$ from initial condition $u(0,x)=g(x)$.
Then,
\[
u(t,x) - g(x) 
= \int_0^t \partial_t u(s, x) ds
= \int_0^t \Delta u(s, x) ds.
\]
Meanwhile, $\Delta u = \Delta Q_t g = Q_t \Delta g$,
and  then
\[
\int_0^t \Delta u(s, x) ds
= \int_0^t (Q_s \Delta g)(x) ds.
\]
As a result,
\[
|\int_0^t \Delta u(s, x) ds |
\le \int_0^t | (Q_s \Delta g)(x)| ds
\le \int_0^t \|Q_s \Delta g\|_\infty ds
\le \|\Delta g\|_\infty t,
\]
where the last inequality is by Lemma \ref{lemma:Qt-contractive-Linf} applied to $\Delta g$.
Putting together, this gives that 
\[
|(Q_t g - g)(x)|
= |u(t,x) - g(x)| \le \|\Delta g\|_\infty t.
\]
Combined with that $\Delta g = \Tr \nabla^2 g$, which implies $\| \Delta g \|_\infty \le d \| \nabla^2 g\|_\infty$, this proves (i).

\

To prove (ii), consider $\epsilon < t_1$ as defined in \eqref{eq:def-t1M}. 
Since $\epsilon < t_1 \le t_0$, Lemma \ref{lemma:Heat-short-time}(ii) is applicable, which gives the Gaussian envelope \eqref{eq:H-decay}  of $\calH_\epsilon(x,y)$ globally on $\calM$. Combined with the expression \eqref{eq:Qtg-g-proof-1}, we have
\[
| (Q_\epsilon g-g)(x) | \le 
C_3
\int_{\calM}  \epsilon^{-d/2} e^{- \frac{ d_\calM( x,y)^2}{ 5 \epsilon}} | g(y) - g(x) | dV(y).
\]
We again consider the truncation to the $B_{2 \delta(\epsilon)}(x)$ geodesic ball as in the proof of Lemma \ref{lemma:Keps-g}.

For any $y \notin B_{2 \delta(\epsilon)}(x)$, 
recall that 
$\epsilon^{-d/2} e^{- { d_\calM( x,y)^2}/{( 5 \epsilon)}} \le \epsilon^6$, and then 
\[
\int_{\calM \backslash B_{2 \delta(\epsilon)}(x)}   \epsilon^{-d/2} e^{- \frac{ d_\calM( x,y)^2}{ 5 \epsilon}} | g(y) - g(x) | dV(y)
\le 2 \|g\|_\infty \Vol(\calM) \epsilon^6.
\]
For any $y \in B_{2 \delta(\epsilon)}(x)$, since $2 \delta(\epsilon) < \delta_0 < \xi$, use the normal coordinates at $x$ and write $y = \exp_x(s\theta)$.
Again, because $g \in C^{0,\beta}(\calM)$, we have
 $|g(y) - g(x)| \le L_{0,\beta}(g) s^\beta$.
 Then, by the same argument as in the proof of Lemma \ref{lemma:Keps-g} Step 3 for $C^{0,\beta}$ function $g$, 
using that $dV(y) \le b_V s^{d-1} ds d\theta$,  we have
\begin{align*}
& \int_{  B_{2 \delta(\epsilon)}(x)}   \epsilon^{-d/2} e^{- \frac{ d_\calM( x,y)^2}{ 5 \epsilon}} | g(y) - g(x) | dV(y) \\
& \le  L_{0,\beta}(g) b_V |S^{d-1}|
	 \int_0^{ 2 \delta(\epsilon) } 
 	\epsilon^{-d/2} e^{-s^2/(5 \epsilon)} s^{\beta + d-1} ds  \\
& \le L_{0,\beta}(g) b_V |S^{d-1}|  \epsilon^{\beta/2}
 \int_0^{ 2 \sqrt{ 2(d+4)  \log \frac{1}{\epsilon} } }  e^{-u^2/5 } u^{\beta+ d-1} du.
\end{align*}
Observe that 
\begin{align*}
\int_0^{ 2 \sqrt{ 2 (d+4)  \log \frac{1}{\epsilon} } }  e^{-u^2/5 } u^{\beta+ d-1} du
& \le 
\int_0^{\infty }  e^{-u^2/5 } u^{\beta+ d-1} du  \\
& \le \int_0^{\infty}  e^{-u^2/5 } 
	(1+u)u^{d-1} du
 =: m_{d,3},
\end{align*}
where the second inequality is by that $u^{\beta} \le 1+ u$, $\forall u \ge 0$, due to  $0<\beta \le 1$.
Thus, 
\[
 \int_{  B_{2 \delta(\epsilon)}(x)}   \epsilon^{-d/2} e^{- \frac{ d_\calM( x,y)^2}{ 5 \epsilon}} | g(y) - g(x) | dV(y) 
 \le L_{0,\beta}(g) b_V |S^{d-1}|  m_{d,3} \epsilon^{\beta/2}.
\]
The constant $b_V$ depends on $\calM$ and  $m_{d,3}$ depends on $d$.
Combining the bounds for integrals in and outside $B_{2 \delta(\epsilon)}(x)$, we have
\begin{align*}
| (Q_\epsilon g-g)(x) | 
& \le 
C_3 (  2 \|g\|_\infty \Vol(\calM) \epsilon^6 + L_{0,\beta}(g) |S^{d-1}|  m_{d,3} \epsilon^{\beta/2}) \\
& \le  C_3 (  2  \Vol(\calM)  +b_V  |S^{d-1}|  m_{d,3} )  \| g\|_{0,\beta} \epsilon^{\beta/2},
\end{align*}
where we also use that $\epsilon <1$.
Let $C_6 = C_3 (  2  \Vol(\calM)  + b_V |S^{d-1}|  m_{d,3} )$, this proves the bound for $g \in C^{0,\beta}(\calM)$.

Finally, when $g \in C^1(\calM)$, then $g \in C^{0,1}(\calM)$, with $L_{0,1}(g) \le \| \nabla g\|_\infty$.
As a result, $\| g\|_{0,1} \le \| g\|_{C^1}$, and the bound follows by the $C^{0,\beta}$ case with $\beta =1$.
\end{proof}

\subsection{One-step analysis}

We will derive ``one-step'' estimates on applying one multiplication of matrix $P$ to the vector $\rho_X g$.
The proof follows by combining the estimates in Section \ref{subsec:integral-operator-bias-analysis}
with concentration argument at finite large number of samples.
We invoke Assumptions \ref{assump:iid-data} and \ref{assump:sigma-large-N}
(in addition to Assumption \ref{assump:M}) in this subsection,
and further assumes that $p$ is uniform.

\begin{lemma}\label{lemma:one-step-C2}
For large enough $N$ that only depends on $\calM$, 

\begin{itemize}
\item[(i)] For any $g \in C^2(\calM)$, 
w.p.$> 1-4 N^{-9}$, 
we have
\[
\sum_{j=1}^N P_{ij} g(x_j)=  (Q_{\sigma^2} g) (x_i) + r_i, \, i =1,\cdots, N, 
\quad \| r \|_\infty = 
	O( \| g \|_{C^2} \sigma^4 + 
	 \| g \|_{C^1} \sigma \sqrt{\frac{\log N}{N \sigma^d}}).
\]

\item[(ii)]  For any $g \in C^{0,\beta}(\calM)$, $ 0 < \beta \le 1$, 
w.p.$> 1-4 N^{-9}$, 
we have
\[
 \| r \|_\infty = \| g\|_{0,\beta} O \Big(  \sigma^{2+\beta}
			 +  \sigma^\beta \sqrt{\frac{\log N}{ N \sigma^d}} \Big).
\]

\item[(iii)]  For any $g \in L^{\infty}(\calM)$, 
w.p.$> 1-4 N^{-9}$,
we have
\[
 \| r \|_\infty = \| g\|_{\infty} O \Big(  \sigma^{2}
			 +  \sqrt{\frac{\log N}{ N \sigma^d}} \Big).
\]
\end{itemize}
In all the cases,  the constants in big-O only depend on $\calM$.
\end{lemma}

The large-$N$ threshold required for the above estimates depends only on the manifold $\calM$,
and is uniform over all test functions $g$.
This is important when we use the lemma later to prove our main theorem.
Comparing Lemma \ref{lemma:one-step-C2} to the pointwise rate of graph Laplacian ($\epsilon=\sigma^2$), 
c.f. \eqref{eq:GL-rate-C4}, which requires $g \in C^4(\calM)$,
we see that the rate in case (i) matches the pointwise rate therein (the graph Laplacian has $1/\epsilon$ normalization in front),
while our result here only require $g$ to be $C^2$. 
The cases (ii) and (iii) are to handle $g$ of lower regularity, 
and will be used to handle such initial value $f$ before $Q_t f$ makes it smooth.

\begin{proof}[Proof of Lemma \ref{lemma:one-step-C2}]
Let $\epsilon = \sigma^2$. 
By definition, 
\begin{equation}
\sum_{j=1}^N P_{ij} g(x_j)
= g(x_i) + 
\frac{ \frac{1}{N-1} \sum_{j=1}^N W_{ij} ( g(x_j) - g(x_i))}{\frac{1}{N-1}  \sum_{j=1}^N W_{ij}},
\end{equation}
and we analyze the numerator and denominator respectively. 

\vspace{5pt}

Denominator:  
Define $D_i := \sum_{j=1}^N W_{ij}$. 
We apply Lemma \ref{lemma:Di-concen-eps2} which is based on concentration argument.
Since $p$ is uniform, Lemma \ref{lemma:Di-concen-eps2}(i) gives that,
at large $N$ under a good event $E_1$ which happens w.p. $>1- 2N^{-9}$,
$
\frac{ D_i}{N} = p + O(\epsilon + \sqrt{ \frac{\log N}{N \epsilon^{d/2}} })$,
$ i=1,\cdots, N$,
and the constant in big-O depends on $\calM$ and is uniform for all $i$.
Under Assumption \ref{assump:sigma-large-N}, we know that $O(\epsilon + \sqrt{ \frac{\log N}{N \epsilon^{d/2}} }) = o(1)$.
As a result, $\frac{ D_i}{N} = p + o(1) = O(1)$ uniform for all $i$, 
and then 
\[
\textcircled{2}_i 
:= \frac{1}{N-1}  \sum_j W_{ij}
= ( 1+ \frac{1}{N-1} ) \frac{D_i}{N} 
= \frac{D_i}{N} + O(\frac{1}{N})
= p+ O( \epsilon + \sqrt{ \frac{\log N}{N \epsilon^{d/2}} } ),
\]
where  in the last equality we used that $1/N \ll \sqrt{ \frac{\log N}{N \epsilon^{d/2}} }$. 
Since $p$ is a positive constant, we have
\begin{equation}\label{eq:degree-D-concen-rel}
\textcircled{2}_i 
= p( 1+ \varepsilon^{D}_i),
\quad
\max_{1 \le i \le N} |\varepsilon^{D}_i| = O(\epsilon + \sqrt{ \frac{\log N}{N \epsilon^{d/2}} }),
\end{equation}
and recall that $O(\epsilon + \sqrt{ \frac{\log N}{N \epsilon^{d/2}} }) = o(1)$,
at large $N$ we also have $  |\varepsilon_i^D| < 1/2$ for all $i$.
The needed large-$N$ threshold for $E_1$ 
and $  |\varepsilon_i^D| < 1/2$
only depends on $\calM$.

\vspace{5pt}

Numerator: We want to analyze
\[
\textcircled{1}_i :=\frac{1}{N-1} \sum_{j, \, j \neq i} Y_j, \quad Y_j:= K_\epsilon( x_i, x_j) ( g(x_j) - g(x_i)).
\]
For each $i$, conditioning on $x_i$, this is an independent sum over $j \neq i$. We have
\begin{align*}
\E [ Y_j  | x_i ]
& = \int_\calM K_\epsilon( x_i, y) ( g(y) - g(x_i)) p dV(y). 
\end{align*}
We consider the three cases of $g$ respectively. 

\vspace{5pt}
\noindent
$\bullet$
Case (i), $g \in C^{2}(\calM)$
 \vspace{5pt}

Recall the $O(1)$ constant $t_1$ as defined in \eqref{eq:def-t1M}, which depends on $\calM$.
We assume large enough $N$ s.t. $\epsilon < t_1$ (due to that $\epsilon = \sigma^2 \ll 1$),
and then by Lemma \ref{lemma:Keps-g}(i),
\[
\E [ Y_j  | x_i ]
= p ( (Q_{\epsilon} g - g)(x_i) +  O( \| g\|_{C^2} \epsilon^2)).
\] 

To bound $|\textcircled{1}_i - \E [ Y_j  | x_i ] |$, we apply the classical Bernstein inequality Lemma \ref{lemma:bernstein}.
First,
\[
{\rm Var}( Y_j  | x_i )
\le  \E [Y_j^2  | x_i ]
= \int_\calM \frac{1}{(4 \pi \epsilon)^{d}} e^{- { \|x_i - y \|^2}/{2 \epsilon}} ( g(y) - g(x_i))^2 p dV(y).
\]
Using the same technique as in proving Lemma \ref{lemma:Keps-g}(ii), 
namely by truncating inside and outside the $B_{2\delta(\epsilon)}(x_i)$ ball, 
one can show that (with the $(4 \pi \epsilon)^{-d/2}$ normalization) 
\[
\int_\calM \frac{1}{(4 \pi \epsilon)^{d/2}} e^{- { \|x- y \|^2}/{2 \epsilon}} ( g(y) - g(x))^2 dV(y)
= \|g\|_\infty^2 O( \epsilon^{d/2+4}) + \|\nabla g\|_\infty^2 O( \epsilon)
= \| g \|_{C^1}^2 O(\epsilon).
\]
This is because
for $y \notin B_{2 \delta(\epsilon)}(x)$, we have $\| x - y\| \ge \delta(\epsilon)$, which implies $ e^{- { \|x- y \|^2}/{2 \epsilon}} \le \epsilon^{d+4}$, and then the outside-ball part of the integral can be bounded by $\|g\|_\infty^2 O( \epsilon^{d/2+4})$;
for $y \in B_{2 \delta(\epsilon)}(x)$, writing $y = \exp_x( s\theta)$,
we have $ |g(y) - g(x)| \le \| \nabla g\|_\infty s$, 
and this leads to the inside-ball part of the integral bounded by $\|\nabla g\|_\infty^2 O( \epsilon)$;
the final bound is again by that  $\epsilon <1$.
We then have
\[
{\rm Var}( Y_j  | x_i )
= \frac{p}{(4 \pi \epsilon)^{d/2}}  \| g \|_{C^1}^2 O (\epsilon)
\le c_\nu \| g \|_{C^1}^2  \epsilon^{-d/2+1} =:  \nu_Y,
\]
where the constant $c_\nu$ depends on $\calM$.

We also claim that, with large enough $N$ s.t. $\epsilon < t_1$,
\begin{equation}\label{eq:bounded-LY-goal}
|Y_j | \le L_Y = c_L \| g\|_{C^1} \epsilon^{-d/2+1/2},
\end{equation}
where $c_L$ is a constant depending on $\calM$.
The proof of \eqref{eq:bounded-LY-goal} is postponed below.

With the variance upper bound $\nu_Y$ and boundedness $L_Y$ of $Y_j | x_i$, 
we aim at deviation $\tau =\sqrt{ 40 \frac{\log N}{N-1} \nu_Y}$.
The needed condition $\tau L_Y \le 3 \nu_Y$ reduces to 
\[
\frac{\log N}{ N-1} \le \frac{9}{40} \frac{c_\nu}{c_L^2} \epsilon^{d/2},
\]
where the $\| g\|_{C^1}$ factor cancels.
Thus, the condition is  satisfied with large enough $N$  due to Assumption \ref{assump:sigma-large-N},
and the large-$N$ threshold is determined by $c_\nu/c_L^2$, which only depends on $\calM$ and is uniform for all $g$.
Then, w.p. $> 1-2 N^{-10}$, we have that 
\[
\textcircled{1}_i = \E [ Y_j  | x_i ]  + O( \| g\|_{C^1} \sqrt{\epsilon \frac{\log N}{N \epsilon^{d/2}}  }).
\]
Taking union bound over $i= 1, \cdots, N$, we have that,  w.p. $> 1-2 N^{-9}$, 
which we call a good event $E_4$,
\begin{equation}\label{eq:numerator-circle1-prime}
\textcircled{1}_i  = 
p (Q_{\sigma^2} g - g)(x_i) +  O \Big( \| g\|_{C^2} \sigma^4
			 + \| g \|_{C^1} \sigma \sqrt{\frac{\log N}{ N \sigma^d}} \Big).
\end{equation}

\vspace{5pt}

Putting together \eqref{eq:degree-D-concen-rel} and \eqref{eq:numerator-circle1-prime},  and by that $  |\varepsilon_i^D| < 1/2$ for all $i$, 
we have that, under $E_1 \cap E_4$, 
\begin{align*}
\frac{\textcircled{1}_i }{\textcircled{2}_i }
&=  \frac{ p (Q_{\sigma^2} g - g)(x_i) +  O \Big( \| g\|_{C^2} \sigma^4
			 + \| g \|_{C^1} \sigma \sqrt{\frac{\log N}{ N \sigma^d}} \Big) }
	{p( 1+ \varepsilon^{D}_i)} \\
&=
	 (Q_{\sigma^2} g - g)(x_i) 
	 + O(|\varepsilon^{D}_i|) \| Q_{\sigma^2} g - g \|_\infty
	 +  O \Big( \| g\|_{C^2} \sigma^4
			 + \| g \|_{C^1} \sigma \sqrt{\frac{\log N}{ N \sigma^d}} \Big).
\end{align*}
Note that 
\[
O(|\varepsilon^{D}_i|) \| Q_{\sigma^2} g - g \|_\infty
= O( \sigma^2) \| Q_{\sigma^2} g - g \|_\infty + O( \sqrt{ \frac{\log N}{N \sigma^{d}} } ) \| Q_{\sigma^2} g - g \|_\infty,
\]
and we bound the two terms respectively:
Because $g \in C^2(\calM)$, 
 Lemma \ref{lemma:Qeps-g-estimate}(i) gives that
$ \| Q_{\sigma^2} g - g \|_\infty = O(  \| \nabla^2 g\|_\infty \sigma^2)$;
since $g$ is also in $C^1(\calM)$,
 Lemma \ref{lemma:Qeps-g-estimate}(ii)  gives that 
$ \| Q_{\sigma^2} g - g \|_\infty = O(  \| g\|_{C^1} \sigma)$.
Use the former to bound the first term, and the latter to the second term, we have that 
\[
O(|\varepsilon^{D}_i|) \| Q_{\sigma^2} g - g \|_\infty
= O \Big(  \| \nabla^2 g\|_\infty \sigma^4
+   \| g\|_{C^1} \sigma \sqrt{ \frac{\log N}{N \sigma^{d}} } \Big).
\]
The lemma follows by putting this back to the bound of ${\textcircled{1}_i }/{\textcircled{2}_i }$. 

\vspace{5pt}
\noindent
Proof of \eqref{eq:bounded-LY-goal}: Consider two cases,

a) $\| x_i - x_j\| \ge  \delta(\epsilon)$, then by definition of $\delta(\epsilon)$, 
we have $\epsilon^{-d/2} e^{- { \|x- y \|^2}/{4 \epsilon}}  \le \epsilon^2$, and then
\[
\epsilon^{-d/2} e^{- { \|x_i - x_j \|^2}/{4 \epsilon}} |g(x_j) - g(x_i)|
\le 2\| g\|_\infty \epsilon^2;
\]

b) $\| x_i - x_j\| <  \delta(\epsilon)$, then $\| x_i - x_j\| < \tau /2$, and then 
by Lemma \ref{lemma:manifold-reach}, $d_\calM(x_i,x_j) \le 2 \| x_i - x_j\|  < 2 \delta(\epsilon) < \xi$.
In this case,
\[
|g( x_j )-g(x_i)|\le \| \nabla g \|_\infty d_\calM( x_i, x_j) 
	\le 2 \| \nabla g \|_\infty  \| x_i - x_j\|,
\]
and then 
\[
e^{- { \|x_i -  x_j \|^2}/{4 \epsilon}} |g(x_j) - g(x_i)| 
\le 2 \| \nabla g \|_\infty  e^{- { \|x_i - x_j \|^2}/{4 \epsilon}}  \| x_i - x_j\|
\le 2 \| \nabla g \|_\infty c_0 \sqrt{\epsilon },
\]
where $c_0$ is an absolute positive constant.

Combining both cases, since $\epsilon < 1$, we always have that 
\[
|Y_j | \le c_d \epsilon^{-d/2 + 1/2} (\| g\|_\infty +  \| \nabla g \|_\infty), 
\]
where $c_d$ is a constant determined by $d$. This proves the claim \eqref{eq:bounded-LY-goal}.

This finishes case (i) of the lemma.

\vspace{5pt}
\noindent
$\bullet$
Case (ii), $g \in C^{0,\beta}(\calM)$
 \vspace{5pt}

Following the same strategy, 
by Lemma \ref{lemma:Keps-g}(ii),
\[
\E [ Y_j  | x_i ]
= p  (Q_{\epsilon} g - g)(x_i) +  O( \| g\|_{0,\beta} \epsilon^{1 + \beta/2});
\] 
 By that $ |g(x_j) - g(x_i)| \le L_{0,\beta}(g) d_\calM( x_i, x_j)^\beta$ 
 when $x_j \in B_{2\delta(\epsilon)}(x_i)$, one can verify that 
 \[
{\rm Var}( Y_j  | x_i )
\le \nu_Y  
= c_\nu' \| g \|_{0, \beta}^2  \epsilon^{-d/2+\beta},
 \]
and 
\[
|Y_j | \le L_Y = c_L' \| g\|_{0, \beta} \epsilon^{-d/2+\beta/2}.
\]
The good event $E_1$ is about $D_i$  for all $i$ and same as before;
the good event $E_4$ is for the concentration of the summation $\textcircled{1}_i $,
and again holds w.p. $>1-2N^{-9}$ for all $i$.
Putting together, under $E_1 \cap E_4$ we have
\begin{align*}
\frac{\textcircled{1}_i }{\textcircled{2}_i }
&=  \frac{ p (Q_{\sigma^2} g - g)(x_i) +  \| g\|_{0,\beta} O \Big(  \sigma^{2+\beta}
			 +  \sigma^\beta \sqrt{\frac{\log N}{ N \sigma^d}} \Big) }
	{p( 1+ \varepsilon^{D}_i)} \\
&=
	 (Q_{\sigma^2} g - g)(x_i) 
	 + O(|\varepsilon^{D}_i|) \| Q_{\sigma^2} g - g \|_\infty
	 +  \| g\|_{0,\beta} O \Big(  \sigma^{2+\beta}
			 +  \sigma^\beta \sqrt{\frac{\log N}{ N \sigma^d}} \Big) .
\end{align*}
Lemma \ref{lemma:Qeps-g-estimate}(ii)  gives that 
$ \| Q_{\sigma^2} g - g \|_\infty = O(  \| g\|_{0,\beta} \sigma^{\beta})$, and then
\[
O(|\varepsilon^{D}_i|) \| Q_{\sigma^2} g - g \|_\infty 
=  \| g\|_{0,\beta} O(\sigma^{2+\beta} + \sigma^{\beta}\sqrt{ \frac{\log N}{N \sigma^{d}} }).
\]

\vspace{5pt}
\noindent
$\bullet$
Case (iii), $g \in L^{\infty}(\calM)$
 \vspace{5pt}
 
 We use Lemma \ref{lemma:Keps-g}(iii) and Lemma \ref{lemma:Qeps-g-estimate}(iii).
 The proof follows the same argument as that in case (ii)
 with $|g(x_i) - g(x_j)| \le 2 \| g\|_\infty$,
 where setting $\beta = 0$ and replacing $\| g \|_{0,\beta}$ to be $\| g\|_\infty$.
 In particular, one can verify that
 $
\nu_Y  = c_\nu'' \| g \|_{\infty}^2  \epsilon^{-d/2}$,
and $ L_Y = c_L'' \| g\|_{\infty} \epsilon^{-d/2}$.

\vspace{5pt}

Finally, the constants $c_\nu$, $c_\nu'$, $c_\nu''$, and $c_L$, $c_L'$, $c_L''$ may differ, while all only depend on $\calM$. 
As a result, the large-$N$ needed for (event $E_4$ in) the three cases (i)(ii)(iii) may differ. The final threshold takes the maximum of those of the three cases. 
\end{proof}

\subsection{Multi-step analysis}

A key estimate for our analysis is that after a positive diffusion time $t$, 
the higher order derivative norm of $Q_t f$ can be bounded by low order ones of $f$ times an inverse power of $t$. This is known as the heat smoothing estimates,  see \cite{lunardi2012analytic} for a general reference.
Because we will apply it on the manifold, we derive the formal statement in the following lemma,
and include a proof in Appendix \ref{app:lemmas} for completeness. 

\begin{lemma}\label{lemma:smoothing-estimates}
There exists positive constant $C_7$ that only depends on $\calM$, s.t.  $\forall t > 0$,

(i) For any $f \in C^{0,\beta}(\calM)$,   $ 0< \beta \le 1$,
\begin{equation}
\| \nabla Q_t f \|_\infty \le \frac{ C_7}{t^{1/2 - \beta/2}} \| f \|_{0,\beta},
\quad 
\| \nabla^2 Q_t f \|_\infty \le \frac{C_7}{t^{1-\beta/2} } \| f \|_{0,\beta}.
\end{equation}

(ii) For any $f \in L^{\infty}(\calM)$,   
\begin{equation}
\| \nabla Q_t f \|_\infty \le \frac{ C_7}{t^{1/2}} \| f \|_\infty,
\quad 
\| \nabla^2 Q_t f \|_\infty \le \frac{C_7}{t } \| f \|_\infty.
\end{equation}
\end{lemma}

We are ready to prove the main theorem of this section.

\begin{theorem}\label{thm:p-uniform}
Under Assumptions \ref{assump:M}, \ref{assump:iid-data}, \ref{assump:sigma-large-N},
suppose $p$ is uniform.
Suppose  $n \ll N^2$, let $ t = n\sigma^2$.
Then, when $N$ is large enough, 

(i)  For any $f \in C^{0,\beta}(\calM)$, $0 < \beta  \le 1$,
 w.p.$> 1-  N^{-7}$,
\[
\|  P^n \rho_X f -  \rho_X Q_t f   \|_\infty 
= \| f\|_{0,\beta} O\Big( 	
		( \frac{ t^{{\beta}/{2}} }{\beta} + t  ) \sigma^2   
		+   ( t^{{(\beta+1)}/{2}}  + t)
			\sqrt{ \frac{\log N}{N \sigma^{d+2}}} \Big).
\]

(ii) For any $f \in L^\infty(\calM)$, 
 w.p.$> 1-  N^{-7}$,
\[
\|  P^n \rho_X f - \rho_X Q_t f  \|_\infty 
= \| f\|_\infty O\Big(	   (t+ 1+ \log n) \sigma^2 
			+( t^{1/2} + t ) \sqrt{ \frac{\log N}{N \sigma^{d+2}}} \Big).
\]

In (i)(ii), the constants in big-O only depend on $\calM$.
\end{theorem}

The condition $n \ll N^2$ is only used to ensure that the high probability has $N^{-7}$ tail.  
Note that by Assumption \ref{assump:sigma-large-N},
\[
n = t \sigma^{-2} 
	\ll t (N / \log N)^{2/d} 
	\le t (N / \log N)^{2},
\]
so whenever $t =O( \log^2 N)$ we would have $n \ll N^2$.
When $t$ is large, $Q_t f $ approaches the average of $f$ on $\calM$ and is not informative,
and thus we are only interested in when $t \lesssim 1$ in most scenarios.

This not only guarantees $ n \ll N^2$, but also allows to simplify the error bound in Theorem \ref{thm:p-uniform}.
For example, when $ t \lesssim 1$, the error bound in (i) becomes (omitting the constants)
\[
O \Big( \sigma^2 + \sqrt{ \frac{\log N}{N \sigma^{d+2}}} \Big)
= O(N^{- 2/(d+6)}) \text{ up to log factor,  when $\sigma\sim N^{-1/(d+6)}$},
\]
which recovers the point-wise rate \eqref{eq:GL-rate-C4} of graph Laplacian.
Note that this holds for any $f \in C^{0,\beta}(\calM)$, which may not have 2nd derivative.
If $f$ is merely in $L^\infty(\calM)$, the error bound bound only degenerates  another $\log n$ factor in the $O(\sigma^2)$ term.

When $f$ has higher regularity, we expect the rate to remain the same as long as $t \sim 1$.
Specifically, if $f \in C^2(\calM)$, then one can expect 
$ \| \nabla Q_t f \|_\infty \le C \|f \|_{C^1}$, 
$ \| \nabla^2 Q_t f \|_\infty \le C \|f \|_{C^2},$
as the counterpart of Lemma \ref{lemma:smoothing-estimates}.
By applying Lemma \ref{lemma:one-step-C2}(i) starting from the 1st step, one can show the final bound to be 
$t O( \sigma^2 + \sqrt{ \frac{\log N}{N \sigma^{d+2}}} )$, namely $n$ times the one-step error bound.
Compared to Theorem \ref{thm:p-uniform}(i), this only improves the factor in powers of $t$ when $t \ll 1$,
e.g. from  $t^{\beta/2} $ to $t$ in the $O(\sigma^2)$ term.
In particular, if $ t \sim 1$, then the rate is the same and matches the pointwise rate of $O( \sigma^2 + \sqrt{ \frac{\log N}{N \sigma^{d+2}}} )$.

\begin{proof}[Proof of Theorem \ref{thm:p-uniform}]

We define $f_0 : = f$ and $f_m:= Q_{m \sigma^2} f$, $m = 1, 2, \cdots, n$.
We also denote by $\vec g = \rho_X g$ for a function $g$, \, $\vec{}$\, standing for the vector.
Let $N$ be large enough so that Lemma \ref{lemma:one-step-C2} applies, and the big-O bounds therein for all three cases involve a constant $C_8$ that depends on $\calM$ only.

\vspace{5pt}

Consider case (i).
In the first multiplication of $P$, we use Lemma \ref{lemma:one-step-C2}(ii),
which gives that w.p. $>1-4N^{-9}$, 
\[
P \vec f_0 = \vec f_1 + r_1, \quad 
	\| r_1\|_\infty \le C_8  \| f\|_{0,\beta} \Big(  \sigma^{2+\beta}
								 +  \sigma^\beta \sqrt{\frac{\log N}{ N \sigma^d}} \Big).
\]
If $n=1$, this already proves the bound. If $n \ge 2$,
in the next $n-1$ multiplications, $f_m  \in C^\infty(\calM)$, which allows us to  use Lemma \ref{lemma:one-step-C2}(i).
Then, w.p. $>1-4N^{-9}$ for each $m=1, \cdots, n-1$, 
\[
P \vec f_{m} = \vec f_{m+1} + r_{m+1}, \quad 
\| r_{m+1}\|_\infty 
	\le C_8 ( \| f_m \|_{C^2} \sigma^4 + 
	 \| f_m \|_{C^1} \sigma \sqrt{\frac{\log N}{N \sigma^d}}).
\]
To bound $\| f_m \|_{C^2}$ and $\| f_m \|_{C^1}$, we use Lemma \ref{lemma:smoothing-estimates}(i).
For each $m$, $t = m\sigma^2 > 0$, then 
\begin{align*}
\| \nabla f_m \|_\infty  \le C_7 (m\sigma^2)^{-1/2 + \beta/2} \| f\|_{0,\beta}, \quad
\| \nabla^2 f_m \|_\infty  \le C_7 (m\sigma^2)^{-1 + \beta/2} \| f\|_{0,\beta}. 
\end{align*}
Combined with that $\| f_m\|_\infty \le \|f_0\|_\infty$ due to Lemma \ref{lemma:Qt-contractive-Linf}, we have
\begin{align*}
\| f_m \|_{C^2} 
	& = \|f_m\|_\infty + \| \nabla f_m\|_\infty + \| \nabla^2 f_m\|_\infty \\
	& \le \|f_0\|_\infty + C_7 (m\sigma^2)^{-1/2 + \beta/2} \| f\|_{0,\beta}
			+ C_7 (m\sigma^2)^{-1 + \beta/2} \| f\|_{0,\beta} \\
 	&\le 2(1+C_7) \| f\|_{0,\beta} ( 1  + (m\sigma^2)^{-1 + \beta/2} ),
\end{align*}
where in the last inequality we used that $t ^{-1/2 + \beta/2} \le 1 + t^{-1+\beta/2}$ for any $t >0$;
Meanwhile,
\begin{align*}
\| f_m \|_{C^1} 
	& = \|f_m\|_\infty + \| \nabla f_m\|_\infty  \\
	& \le \|f_0\|_\infty + C_7 (m\sigma^2)^{-1/2 + \beta/2} \| f\|_{0,\beta} \\
 	&\le (1+C_7) \| f\|_{0,\beta} ( 1  + (m\sigma^2)^{-1/2 + \beta/2} ).
\end{align*}

By construction,
\[
 P^n \vec f_0  - \vec f_n  = \sum_{m=1}^{n} P^{n-m} r_{m},
\]
and then, by triangle inequality,
\[
\|  P^n \vec f_0  - \vec f_n   \|_\infty
\le  \sum_{m=1}^{n} \| P^{n-m} r_{m}\|_\infty 
\le \| r_1\|_\infty + \sum_{m=1}^{n-1} \| r_{m+1}\|_\infty,
\]
where the 2nd inequality is by that $P$ is contractive for vector $\infty$-norm (by that $P$ is row-stochastic).
Collecting the estimates above, we have
\begin{align*}
\sum_{m=1}^{n-1} \| r_{m+1}\|_\infty
\le C_8(1+C_7) \| f\|_{0,\beta} 
 	\Big(
	2 \sigma^4  \sum_{m=1}^{n-1} 
	  ( 1  + (m\sigma^2)^{-1 + \beta/2} )  
	+ \sigma \sqrt{\frac{\log N}{N \sigma^d}} \sum_{m=1}^{n-1}   ( 1  + (m\sigma^2)^{-1/2 + \beta/2} )  \Big).
\end{align*}
Note that, since $0< \beta \le 1$,
\[
 \sum_{m=1}^{n-1} 
	   m^{-1 + \beta/2} 
	    \le 1+\frac{2}{\beta} n^{\beta/2}
	   \le \frac{4}{\beta} n^{\beta/2}, \quad
 \sum_{m=1}^{n-1} 
	   m^{-1/2 + \beta/2} \le \frac{2}{1+\beta} n^{1/2+\beta/2}
	   \le 2 n^{1/2+\beta/2}.	   
\]
Writing $n \sigma^2$ as $t$, we then have
\begin{align*}
\sum_{m=1}^{n-1} \| r_{m+1}\|_\infty
\le C_8(1+C_7) \| f\|_{0,\beta} 
 	\Big(
	2 \sigma^2  
	( t + \frac{4}{\beta}  t^{\beta/2} )
	+  \sqrt{\frac{\log N}{N \sigma^{d+2}}} 
	   ( t  + 2 t^{(\beta+1)/2} )  \Big).
\end{align*}
Combined with the bound of $\| r_1\|_\infty$, we have
\[
\|  P^n \vec f_0  - \vec f_n   \|_\infty
= \| f\|_{0,\beta}
		    O	\Big(  
			\sigma^2  
			( \sigma^\beta+ t + \frac{t^{\beta/2}}{\beta}  )
			 +  \sqrt{\frac{\log N}{N \sigma^{d+2}}}  (\sigma^{\beta+1} + t  +  t^{(\beta+1)/2} )
			 \Big),
\]
and by that $ \sigma^2 \le t$, we have $\sigma^{\beta} \le t^{\beta/2}$,
and $\sigma^{\beta+1} \le t^{(\beta+1)/2}$.
This proves the bound when $n\ge 2$, where the good event is the intersection of the $n$ good events due to applying Lemma \ref{lemma:one-step-C2} for $n$ times, 
and then the tail probability is $4 n N^{-9} \ll N^{-7} $.

\vspace{5pt}
The proof of case (ii) follows the same strategy:
in the first multiplication of $P$ we use Lemma \ref{lemma:one-step-C2}(iii),
which proves the bound when $n=1$;
In the rest $n-1$ multiplications we still use Lemma \ref{lemma:one-step-C2}(i) because $f_m \in C^{\infty}(\calM)$.
We use Lemma \ref{lemma:smoothing-estimates}(ii) to bound $\| \nabla f_m \|_\infty$ and $\| \nabla^2 f_m \|_\infty$.
This leads to 
\begin{align*}
\sum_{m=1}^{n-1} \| r_{m+1}\|_\infty
&\le C_8(1+C_7) \| f\|_{\infty} 
 	\Big(
	2 \sigma^4   
	  ( n-1  + \sum_{m=1}^{n-1} (m\sigma^2)^{-1} )  
	+ \sigma \sqrt{\frac{\log N}{N \sigma^d}} 
	  ( n-1  + \sum_{m=1}^{n-1}  (m\sigma^2)^{-1/2 } )  \Big) \\
&\le 	  C_8(1+C_7) \| f\|_{\infty} 
 	\Big(
	2 \sigma^2   
	  ( t  +  1+ \log n )  
	+ \sqrt{\frac{\log N}{N \sigma^{d+2}}} 
	  ( t  +  2 \sqrt{t} )  \Big),
\end{align*}
where we use that 
$ \sum_{m=1}^{n-1}  m^{-1} \le 1+ \log n$,
$\sum_{m=1}^{n-1} m^{-1/2} \le 2\sqrt{n}$.
As a result,
\[
\|  P^n \vec f_0  - \vec f_n   \|_\infty
= 	  \| f\|_{\infty} 
 	O\Big(
	 \sigma^2   
	  ( t  +  1+ \log n )  
	+ \sqrt{\frac{\log N}{N \sigma^{d+2}}} 
	  ( \sigma + t  +   \sqrt{t} )  \Big),
\]
and the claimed bound follows by that $\sigma \le \sqrt{t}$ since $ n \ge 1$.
\end{proof}

\section{Theoretical extensions}\label{sec:theory-extend}

\subsection{Non-uniform data density on manifold}\label{subsec:non-uniform-p}

We still assume Assumptions \ref{assump:M}, \ref{assump:iid-data}, \ref{assump:sigma-large-N} and, in addition,

\begin{assumption}\label{assump:p-C3}
$p\in C^{3}(\calM)$ and is uniformly bounded both from below and above, that is, $\exists p_{\min}, \, p_{\max} > 0$ s.t.
\[ 0< p_{\min}  \le p(x) \le p_{\max} < \infty, \quad\forall x\in{\calM}.
\]
\end{assumption}

Because $p$ is no longer uniform, 
$\frac{1}{N}\sum_j W_{ij} f(x_j)$ no longer approximates $\int_\calM K_\epsilon(x,y) f(y) dV(y)$, but $\int_\calM K_\epsilon(x,y) f(y) p(y) dV(y)$. 
To recover the convolution with the manifold heat kernel, one would need to cancel the $p(y)$ factor.
Observe that $D_i/N \approx p(x_i)$ by Lemma \ref{lemma:Di-concen-eps2}(i),
a natural choice is thus to right normalize the matrix $W$ by the degrees $D_j$.
Such ``density correction'' normalization 
has been previously introduced \cite{coifman2006diffusion}
and the convergence of the associated graph Laplacian was analyzed in \cite{cheng2022eigen}.
One then expects that
\[
\sum_{ j=1}^N W_{ij} \frac{f(x_j)}{ D_j}
\approx 
\int_\calM K_\epsilon( x_i, y) \frac{f(y)}{\cancel{p(y)}} \cancel{p(y)} dV(y)
= \int_\calM K_\epsilon( x_i, y) f(y) dV(y).
\]
We thus consider 
\[
\tilde P := D_s^{-1} W D^{-1}, \quad D_s = {\rm diag}(s), \quad s:= WD^{-1} {\bf 1}_N,
\]
and aim at showing that 
\[
\rho_X Q_t f \approx  \tilde P^n \rho_X f, \quad  t = n \sigma^2.
\]

\begin{lemma}\label{lemma:one-step-C2-pC3}
For large enough $N$ that depends on $(\calM, p)$, 
the same bounds for $ \| \tilde P \rho_X g - \rho_X Q_{\sigma^2} g \|_\infty$ hold 
under the three cases as in Lemma \ref{lemma:one-step-C2} respectively, 
w.p. $> 1- 8N^{-9}$ for each case.
The constants in big-O depend on $(\calM, p)$.
\end{lemma}

\begin{proof}[Proof of Lemma \ref{lemma:one-step-C2-pC3}]
The proof follows the strategy of Lemma \ref{lemma:one-step-C2}, where we handle the $1/D_j$ normalization with extra care.
For the latter, we borrow the techniques used in \cite[Section 6]{cheng2022eigen}.

Let $\epsilon = \sigma^2$. 
By definition, 
\begin{equation}\label{eq:Pij-ratio}
\sum_{j=1}^N \tilde P_{ij} g(x_j)
= g(x_i) + 
\frac{  \sum_{j=1}^N W_{ij} \frac{ g(x_j) - g(x_i) }{D_j} }{  \sum_{j=1}^N \frac{W_{ij}}{D_j}} 
= g(x_i) + \frac{\textcircled{1}_i}{\textcircled{2}_i}. 
\end{equation}
By Lemma \ref{lemma:Di-concen-eps2}(ii), 
with large $N$, under a good event denoted as $E_{1,2}$ which happens w.p.$> 1-4N^{-9}$,
we have 
\[
\textcircled{2}_i = 
\sum_{j=1}^N \frac{W_{ij}}{D_j} = 1 + \varepsilon^D_i, 
\quad \max_{1 \le i \le N} |\varepsilon^D_i | = O \Big( \epsilon + \sqrt{ \frac{\log N}{N \epsilon^{d/2}} } \Big).
\]
and recall that $O(\epsilon + \sqrt{ \frac{\log N}{N \epsilon^{d/2}} }) = o(1)$,
we also have $  |\varepsilon_i^D| < 1/2$ for all $i$ with large $N$.
The needed large-$N$ threshold for $E_{1,2}$ 
and $  |\varepsilon_i^D| < 1/2$ depends on $(\calM,p)$.

Meanwhile, the good event $E_{1,2}$ is under the good event needed by Lemma \ref{lemma:Di-concen-eps2}(i), and then we also have
\[
\frac{1}{N} D_i =  \tilde p_\epsilon (x_i) + O \Big( \epsilon^{3/2} + \sqrt{ \frac{\log N}{N \epsilon^{d/2}} } \Big),
\quad 
\tilde p_\epsilon = p +  \epsilon p_1,
\]
where $p_1 =  \omega p + \Delta p \in C^1(\calM)$. 
We introduce \[
\eta: = p_1/ p = \omega + \Delta p /p, \quad \eta \in C^1(\calM),
\] 
and then $p/\tilde p_\epsilon = (1+\epsilon \eta)^{-1}$.
We assume $\epsilon \| \eta \|_\infty < 1/2$, which can be guaranteed by large enough $N$ (depending on $\calM$ and $p$)
 since $\epsilon = o(1)$.
 Then,  uniformly over $\calM$,  
  \[
 1+\epsilon \eta(x) > 1/2, \quad 
 \tilde p_\epsilon(x) = (1+ \epsilon \eta(x)) p(x) > p(x)/2 > p_{\rm min}/2 =: \tilde p_{\rm min}  >0.
 \]

 To proceed, note that as shown in the proof of Lemma \ref{lemma:one-step-C2}, 
by Assumption \ref{assump:sigma-large-N}, $D_i/N = O(1)$, 
and we also have $1/N \ll \sqrt{ \frac{\log N}{N \epsilon^{d/2}} }$.
As a result, we have
\[
\frac{1}{N-1} D_i =  \tilde p_\epsilon (x_i) (1 + \varepsilon_i),
\quad \max_{1 \le i \le N} |\varepsilon_i | = O \Big( \epsilon^{3/2} + \sqrt{ \frac{\log N}{N \epsilon^{d/2}} } \Big),
\quad \max_{1 \le i \le N} |\varepsilon_i| < 1/2,
\]
with large $N$, and the threshold depends on $(\calM, p)$.
This allows us to write
\begin{align*}
\textcircled{1}_i 
& = \frac{1}{N-1}\sum_{j=1}^N W_{ij} \frac{ g(x_j) - g(x_i) }{ \tilde p_\epsilon(x_j) (1+ \varepsilon_j)}  \\
& = \frac{1}{N-1}\sum_{j=1}^N W_{ij} \frac{ g(x_j) - g(x_i) }{ \tilde p_\epsilon(x_j) } (1+ \varepsilon_j') 
 = \textcircled{1}_i'  + \textcircled{3}_i,
\end{align*}
where 
\begin{align*}
\textcircled{1}_i' 
& := \frac{1}{N-1}\sum_{j=1}^N W_{ij} \frac{g(x_j) - g(x_i)}{\tilde p_\epsilon(x_j)}, \\
\textcircled{3}_i
& : = 
\frac{1}{N-1}\sum_{j=1}^N W_{ij} \frac{ g(x_j) - g(x_i) }{ \tilde p_\epsilon(x_j) } \varepsilon_j',
\quad 
\max_{1 \le i \le N} |\varepsilon_i' | = O \Big( \epsilon^{3/2} + \sqrt{ \frac{\log N}{N \epsilon^{d/2}} } \Big),
\end{align*}
by that $\varepsilon_i' = - \varepsilon_i/(1+\varepsilon_i)$
and then $|\varepsilon_i'| \le 2 |\varepsilon_i|$.
We claim that, with large $N$ and under a good event $E_3$ which happens w.p. $>1-2N^{-9}$, for all $i=1,\cdots,N$,
\begin{equation}\label{eq:claim-E3-bound}
\frac{1}{N-1}\sum_{j=1}^N W_{ij} |g(x_j) - g(x_i) |
=\begin{cases} 
 O(\sigma^{1}) \|  g \|_{C^1}, & \quad g \in C^1(\calM),\\
 O(\sigma^{\beta }) \| g \|_{0,\beta}, &\quad g \in C^{0,\beta}(\calM), \\
 O(1) \| g\|_\infty, 	&\quad g \in L^\infty(\calM). \\
\end{cases}
\end{equation}
If true, we then have
\begin{align*}
|\textcircled{3}_i| 
& \le 
\frac{1}{N-1}\sum_{j=1}^N W_{ij} \frac{ |g(x_j) - g(x_i) |}{ \tilde p_\epsilon(x_j) } |\varepsilon_j'|
\le \frac{ \max_{1 \le i \le N} |\varepsilon_i' |}{\tilde p_{\rm min}}\frac{1}{N-1}\sum_{j=1}^N W_{ij} |g(x_j) - g(x_i) | \\
& = \begin{cases} 
O \Big( \sigma^{4} + \sigma \sqrt{ \frac{\log N}{N \sigma^{d}} } \Big)\|  g \|_{C^1}, & \quad g \in C^1(\calM), \\
O \Big( \sigma^{3+\beta} + \sigma^{\beta}\sqrt{ \frac{\log N}{N \sigma^{d}} } \Big) \|g \|_{0,\beta}, &\quad g \in C^{0,\beta}(\calM), \\
O \Big( \sigma^{3} + \sqrt{ \frac{\log N}{N \sigma^{d}} } \Big) \| g\|_\infty, &\quad g \in L^\infty(\calM),
\end{cases}
\end{align*}
and the constants in big-O depend on $(\calM, p)$.

Next, we analyze the main term $\textcircled{1}_i'$.
Conditioning on $x_i$, the summands 
$Y_j := W_{ij} (g(x_j) - g(x_i))/\tilde p_\epsilon(x_j)$ for $j \neq i$ are i.i.d., with
\[
\E [Y_j \mid x_i] 
= \int_\calM K_\epsilon(x_i, y)(g(y) - g(x_i))\frac{p(y)}{\tilde p_\epsilon(y)}\,dV(y).
\]
By definition of $\eta$, we have that for any $x$,
\begin{align*}
& \int_\calM K_\epsilon(x, y)(g(y) - g(x))\frac{p(y)}{\tilde p_\epsilon(y)}\,dV(y) \\
& = \int_\calM K_\epsilon(x, y)(g(y) - g(x)) \frac{1}{1+\epsilon \eta} (y) \,dV(y) \\
& = \underbrace{\int_\calM K_\epsilon(x, y)(g(y) - g(x))  \,dV(y) }_{= (Q_\epsilon g - g)(x) + r(x) }
 	+ \tilde r_1(x) + \tilde r_2(x),
\end{align*}
where
\begin{align*}
\tilde r_1(x) 
& = 	\frac{- \epsilon \eta}{1+\epsilon \eta}  (x)
	 \int_\calM K_\epsilon(x, y)(g(y) - g(x)) \,dV(y), \\
\tilde r_2(x) 
& =  \int_\calM K_\epsilon(x, y)(g(y) - g(x)) 
	 \big( \frac{1}{1+\epsilon \eta (y)}- \frac{1}{1+\epsilon \eta(x)}  \big) \,dV(y).
\end{align*}
We again assume large enough $N$ s.t. $\epsilon < t_1$, so that Lemma~\ref{lemma:Keps-g} applies.
Then,
$ \int_\calM K_\epsilon(x, y)(g(y) - g(x)) \,dV(y) = (Q_\epsilon g - g)(x) + r(x)$,
and  $\|r\|_\infty$ is bound as therein for the three regularity classes. 
Together with that $1 + \epsilon \eta \ge 1/2$ uniformly
and Lemma~\ref{lemma:Qeps-g-estimate} to bound $\|Q_\epsilon g - g\|_\infty$,
we have
\[
\| \tilde r_1 \|_\infty
\le 2 \epsilon \| \eta \|_\infty 
	(  \| Q_\epsilon g - g \|_\infty + \|r \|_\infty )
= \begin{cases}
O( \epsilon^2) \|  g\|_{C^2}, 			& g \in C^2(\calM), \\
O( \epsilon^{1+\beta/2}) \|  g\|_{0,\beta}, 	& g \in C^{0,\beta}(\calM), \\
O( \epsilon) \|g\|_\infty, 			& g \in L^\infty(\calM),
\end{cases}
\]
noting that $\| \eta\|_\infty$ is an $O(1)$ constant depending on $(\calM, p)$
and we also use that $\epsilon <1$.

To bound $| \tilde r_2(x) |$, we adopt the same technique in proving Lemma~\ref{lemma:Keps-g}
by truncating on the geodesic ball $B_{2\delta(\epsilon)}(x)$.
When $y$ is within the ball,  $y = \exp_x (s\theta)$, 
then since $\eta \in C^1(\calM)$, we have
\[
 \big| \frac{1}{1+\epsilon \eta (y)}- \frac{1}{1+\epsilon \eta(x)}  \big|
 \le 4 \epsilon \| \nabla \eta\|_\infty s,
\]
where we used that $ \nabla (1+\epsilon \eta)^{-1}
= -\epsilon (1+\epsilon \eta)^{-2} \nabla \eta  $ and $1+\epsilon \eta > 1/2$ uniformly.
This allows the bound the integral within the ball to be 
$ O(\epsilon^2)  \| \nabla \eta\|_\infty  \| \nabla g\|_\infty$ when $g$ is $C^1$, and 
$ O(\epsilon^{3/2})  \| \nabla \eta\|_\infty  \|  g\|_\infty$ when $g \in L^{\infty}(\calM)$, including when $g$ is in $C^{0,\beta}(\calM)$.
The integral outside $B_{2\delta(\epsilon)}(x)$ can be bounded by $O(\epsilon^2 \| g\|_\infty)$
by that 
$K_\epsilon(x,y) \le (4\pi)^{-d/2} \epsilon^2.$
Thus, we have
\[
\| \tilde r_2 \|_\infty
= \begin{cases}
O( \epsilon^2) \|  g\|_{C^1}, 			& g \in C^1(\calM), \\
O( \epsilon^{3/2})  \|g\|_\infty, 			& g \in L^\infty(\calM).
\end{cases}
\]
Collecting the bounds, we have
\[
\int_\calM K_\epsilon(x, y)(g(y) - g(x))\frac{p(y)}{\tilde p_\epsilon(y)}\,dV(y)  
= (Q_\epsilon g - g)(x) + \tilde r(x), 
\quad \tilde r = r + \tilde r_1 + \tilde r_2,
\]
where
\[
\| \tilde r \|_\infty = 
\begin{cases}
O( \epsilon^2) \| g\|_{C^2}, 			& g \in C^2(\calM), \\
O( \epsilon^{1+\beta/2}) \| g\|_{0,\beta}, 	& g \in C^{0,\beta}(\calM), \\
O( \epsilon) \| g\|_\infty, 			& g \in L^\infty(\calM),
\end{cases}
\]
with the constants in big-O now depending on $(\calM, p)$.

The variance and uniform bounds on $Y_j$ follow the same arguments as in the proof 
of Lemma~\ref{lemma:one-step-C2}, modified by the bounded factor $1/\tilde p_\epsilon \le 1/\tilde p_{\min}$:
\[
\mathrm{Var}(Y_j \mid x_i) \le \nu_Y, \quad |Y_j| \le L_Y,
\]
with $\nu_Y$ and $L_Y$ of the same orders as therein, and constants depending on $(\calM, p)$.
By Bernstein (Lemma~\ref{lemma:bernstein}) with $\tau = \sqrt{40 \log N\, \nu_Y / (N-1)}$, 
the condition $\tau L_Y \le 3 \nu_Y$ reduces to a threshold on $\log N/(N-1)$ depending 
only on $(\calM, p)$ and independent of $g$, satisfied for large $N$ by Assumption~\ref{assump:sigma-large-N}. 
Then, we have $\textcircled{1}_i' = \E [Y_j \mid x_i] + v_i$, where, w.p. $> 1 - 2N^{-10}$, 
\[
|v_i| = 
\begin{cases}
O\Big(\|g\|_{C^1}\,\sigma\sqrt{\frac{\log N}{N\sigma^d}}\Big), & g \in C^2(\calM),\\[2pt]
O\Big(L_{0,\beta}(g)\,\sigma^\beta\sqrt{\frac{\log N}{N\sigma^d}}\Big), & g \in C^{0,\beta}(\calM),\\[2pt]
O\Big(\|g\|_\infty\sqrt{\frac{\log N}{N\sigma^d}}\Big), & g \in L^\infty(\calM),
\end{cases}
\]
matching  the deviation bound in Lemma~\ref{lemma:one-step-C2}
and  all constants in big-O here also depend on $p$. 
By union bound 
over $i$, w.p. $> 1 - 2N^{-9}$ uniformly in $i$, and we call this event $E_4$,
\[
\textcircled{1}_i' = (Q_\epsilon g - g)(x_i) + \tilde r (x_i) + v_i,
\]
with $\| \tilde r\|_\infty$ and $ |v_i|$ uniformly bounded as above.

Combining \eqref{eq:Pij-ratio} with $\textcircled{1}_i = \textcircled{1}_i' + \textcircled{3}_i$ 
and $\textcircled{2}_i = 1 + \varepsilon_i^D$, under $E_{1,2} \cap E_3 \cap E_4$ 
(total tail $\le 8N^{-9}$),
\[
\frac{\textcircled{1}_i}{\textcircled{2}_i}
= (Q_\epsilon g - g)(x_i) + O(|\varepsilon_i^D|)\|Q_\epsilon g - g\|_\infty 
+ O( | \tilde r (x_i) | + |v_i | + |\textcircled{3}_i| ).
\]
The bound $O(|\varepsilon_i^D|)\|Q_\epsilon g - g\|_\infty$ is handled same as in the 
proof of Lemma~\ref{lemma:one-step-C2}, splitting $|\varepsilon_i^D|$ into its $O(\sigma^2)$ 
and $\sqrt{\log N/(N\sigma^{d})}$ parts and bounding $\|Q_\epsilon g - g\|_\infty$ by 
Lemma~\ref{lemma:Qeps-g-estimate}(i) and (ii) respectively for the $C^2$ case, and 
analogously for the other two cases. Together with the bound on $ | \textcircled{3}_i |$ 
from \eqref{eq:claim-E3-bound}, this yields the same final bounds as in 
Lemma~\ref{lemma:one-step-C2} for the three cases of $g$, with constants depending 
on $(\calM, p)$. 

It remains to prove the claim \eqref{eq:claim-E3-bound}. 
For each $i$, conditioning on $x_i$, 
let  $Z_j := W_{ij}|g(x_j) - g(x_i)| \ge 0$ for $j \neq i$, which are i.i.d.\ conditioning on  $x_i$. 
We can bound the expectation
\[
\E[ Z_j \mid x_i] = \int_\calM K_\epsilon(x_i, y)|g(y) - g(x_i)| p(y)\, dV(y)
\]
again via the $B_{2\delta(\epsilon)}(x_i)$ truncation: outside the ball, 
$K_\epsilon (x_i,y) \le (4\pi)^{-d/2}\epsilon^2$ contributes $O(\|g\|_\infty p_{\max}\epsilon^2)$ 
in all cases; inside the ball, the change of variable $u = s/\sqrt\epsilon$ together 
with $|g(y) - g(x_i)| \le \|\nabla g\|_\infty s$, $L_{0,\beta}(g)s^\beta$, or $2\|g\|_\infty$ 
respectively yields
\[
\E [Z_j \mid x_i] = 
\begin{cases}
O( \sigma) \|  g\|_{C^1}, & g \in C^1(\calM),\\
O( \sigma^{\beta}) \| g\|_{0,\beta}, & g \in C^{0,\beta}(\calM),\\
O(1) \|g\|_\infty, & g \in L^\infty(\calM),
\end{cases}
\]
with constants depending on $(\calM, p)$ through $p_{\max}$, $b_V$, $|S^{d-1}|$.

The variance $\mathrm{Var}(Z_j \mid x_i) \le \E[Z_j^2 \mid x_i]$ and uniform bound 
$|Z_j| \le L_Z$ follow same as in the proof of Lemma~\ref{lemma:one-step-C2}: 
$\nu_Z \asymp \| g\|_{C_1}^2 \epsilon^{-d/2+1}$, 
$L_Z \asymp \| g\|_{C^1} \epsilon^{-d/2+1/2}$ for the $C^1$ case, and analogously for 
the other two cases. 
Note that $\nu_Z/L_Z^2 \asymp \epsilon^{d/2}$ with the implicit constant depending 
only on $(\calM, p)$ and the $g$-norm factors cancel.
Bernstein with $\tau = \sqrt{40 \log N\, \nu_Z/(N-1)}$ then requires
$\log N/(N-1) \le (9/40) \nu_Z/L_Z^2 \asymp \epsilon^{d/2}$ for large $N$, satisfied 
under Assumption~\ref{assump:sigma-large-N} with threshold independent from $g$.
Then, w.p. $> 1 - 2N^{-10}$, the deviation is at most $O( \sigma \sqrt{\log N/(N\sigma^{d})}) \| g\|_{C_1}$ 
for $g \in C^1$ and analogous for the other cases,
and each dominated by the corresponding $\E [Z_j \mid x_i]$  under Assumption~\ref{assump:sigma-large-N}. 

By union bound over $i$, w.p.  $> 1 - 2N^{-9}$ uniformly in $i$,
which we call the good event $E_3$, 
 the claim \eqref{eq:claim-E3-bound} holds.
The large-$N$ threshold only depends on $(\calM, p)$ and not on $g$.
\end{proof}

\begin{theorem}\label{thm:pC3}
Under Assumptions \ref{assump:M}, \ref{assump:iid-data}, \ref{assump:sigma-large-N}, \ref{assump:p-C3},
suppose $n \ll N^2$, let $ t = n\sigma^2$.
Then, with large $N$,
we have the same bounds for $ \| \tilde P^n \rho_X f - \rho_X Q_t f \|_\infty$ as in Theorem \ref{thm:p-uniform} under the two cases respectively,
with the same high probability for each case.  
The  constants in big-O only depend on $(\calM, p)$.
\end{theorem}

\begin{proof}[Proof of Theorem \ref{thm:pC3}]
The proof much follows that of Theorem \ref{thm:p-uniform},
replacing $P$ with $\tilde P$,
and Lemma \ref{lemma:one-step-C2} with Lemma \ref{lemma:one-step-C2-pC3}.
In particular, the constant $C_8$ in the proof, which absorbs all big-O terms, now depends on $(\calM, p)$.
Because $\tilde P$ is also row-stochastic, it contracts vector-$\infty$ norm. 
The other parts of the proof are the same as before. 
The tail probability per step changes  from $4N^{-9}$  to $8 N^{-9}$,
giving total tail $8 n N^{-9}$ which is still $\ll N^{-7} $. 
\end{proof}

\subsection{Out-of-sample extension}

We construct a functional estimator of $Q_t f$ on $\calM$ and prove uniform approximation error in $L^\infty(\calM) $.
Given a function $f$ on $\calM$, for any $z \in \calM$, we define 
\begin{equation}
\hat F_n(z) := 
	\frac{ \sum_{j=1}^N \frac{ K_\epsilon( z, x_j) }{D_j}  (\tilde P^{n-1} \rho_X f )_j}
		{\sum_{j=1}^N \frac{ K_\epsilon( z, x_j) }{D_j}},
\end{equation}
and when $z = x_i$, 
\[
\hat F_n(x_i) = \frac{\sum_j \frac{W_{ij}}{D_j} ( \tilde P^{n-1} \rho_X f)_j}{\sum_j \frac{W_{ij}}{D_j} } = ( \tilde P^{n} \rho_X f)_i,
\]
which recovers the on-sample estimator. 
In other words, the estimator $\hat F_n(z)$ extends to the whole manifold using kernel $K_\epsilon(z, x_j)$ in the last step of the $n$-step multiplication. 

The analysis is by standard extension technique: 
utilizing the Lipschitz continuity of $\hat F_n$,
it suffices to prove the pointwise approximation to $Q_t f$ on a covering net on $\calM$, 
and this only asks to extend the one-step analysis in the $n$-th step to be out-of-sample. 
Note that our proof of the one-step analysis, which is based on concentration conditioning on $x_i$, naturally applies when the evaluation location is an arbitrary point $z$ instead of one of the data samples.
Thus, our technique readily derives the out-of-sample error to satisfy the same bound.
In the following theorem, we consider the non-uniform $p$ case for generality.

\begin{theorem}\label{thm:out-of-sample}
Under the same assumption as Theorem \ref{thm:pC3}, $t = n \sigma^2$,
when $N$ is large enough, 
we have the same bounds for
$\| Q_t f  -\hat F_n   \|_\infty $
 as in Theorem \ref{thm:p-uniform} under the two cases respectively,
 w.p.$> 1-  N^{-4}$ for each case.
The  constants in big-O only depend on $(\calM, p)$.
\end{theorem}

\begin{proof}[Proof of Theorem \ref{thm:out-of-sample}]
Recall that $f_m = Q_{m \sigma^2} f$, and $\vec f_m = \rho_X f_m$, $f_0 = f$.
Let $\hat f_0 = \vec f$,
then $\hat f_{k+1} = \tilde P \hat f_{k}$,
then  $\hat f_m = \tilde P^m \vec f$.
We then have 
\[
\hat F_n(z) = 
	\frac{ \sum_{j=1}^N \frac{ K_\epsilon( z, x_j) }{D_j}  ( \hat f_{n-1} )_j}
		{\sum_{j=1}^N \frac{ K_\epsilon( z, x_j) }{D_j}}.
\]
The bound for $\|\hat f_{n-1}  - \vec f_{n-1}  \|_\infty$ inherits from Theorem \ref{thm:pC3}, and we only need to bound the error in the last step. 
Specifically, for any $z$,
\[
\hat F_n(z)  - f_n(z) = 
 \underbrace{\Big( \frac{ \sum_{j=1}^N \frac{ K_\epsilon( z, x_j) }{D_j}  ( \vec f_{n-1} )_j}
		{\sum_{j=1}^N \frac{ K_\epsilon( z, x_j) }{D_j}} - f_n(z) \Big)}_{=: B(z)}
	+ \frac{ \sum_{j=1}^N \frac{ K_\epsilon( z, x_j) }{D_j}  ( \hat f_{n-1} - \vec f_{n-1} )_j}
		{\sum_{j=1}^N \frac{ K_\epsilon( z, x_j) }{D_j}},
\]
and thus
\[
| \hat F_n(z)  - f_n(z) |
\le |B(z)| + \|  \hat f_{n-1} - \vec f_{n-1}\|_\infty.
\]
We claim that 
\begin{equation}\label{eq:claim-Lip-generalization}
\sup_{z \in \calM} |B(z)| = O(\epsilon^2 ) \| f\|_\infty + \text{ the one-step error as in Lemma \ref{lemma:one-step-C2-pC3} applied to $ f_{n-1} $},
\end{equation}
for both cases (i)(ii) of $f$, with tail probability $\ll N^{-4}$; Meanwhile, we can bound
$
\|  \hat f_{n-1}  - \vec f_{n-1}   \|_\infty
\le  \sum_{m=1}^{n-1} \| r_{m}\|_\infty
$ same as before.
This would add another $O(\sigma^4 \| f\|_\infty)$ to the final bound, 
 which is dominated by  $O(t \sigma^2 \| f\|_\infty)$ since $n \ge 1$, for both cases of $f$.
This would prove the theorem.

We first consider the Lipschitz property of $B(z)$ over $\calM$. 
We write $B = \bar F_n - f_n$, where
\[
\bar F_n(z) = 
\frac{ \sum_{j=1}^N \frac{ K_\epsilon( z, x_j) }{D_j}  ( \vec f_{n-1} )_j}
		{\sum_{j=1}^N \frac{ K_\epsilon( z, x_j) }{D_j}}.
\]
First $\| \vec f_{n-1}\|_\infty \le \| f_{n-1}\|_\infty \le \|f\|_\infty $ by Lemma \ref{lemma:Qt-contractive-Linf}.
Again by the row normalization, we have $| \bar F_n(z)| \le \| \vec f_{n-1}\|_\infty  \le \| f \|_\infty$.
By definition, one can verify that 
\[
\nabla_z \bar F_n(z)=
\frac{\sum_{j=1}^N \frac{ \nabla_z K_\epsilon( z, x_j) }{D_j}   ( (\vec f_{n-1} )_j - \bar F_n(z))}
	{  \sum_{j=1}^N \frac{ K_\epsilon( z, x_j) }{D_j}},
\]
this gives
\begin{align*}
\| \nabla_z \bar F_n(z) \|_{op} 
& \le  2 \| f\|_\infty
\frac{\sum_{j=1}^N \frac{ \|  \nabla_z  K_\epsilon( z, x_j) \|_{op} }{D_j}   }
	{  \sum_{j=1}^N \frac{ K_\epsilon( z, x_j) }{D_j}} \\
& \le  2 \| f\|_\infty
\frac{\sum_{j=1}^N  \frac{\| z - x_j\|}{2\epsilon}\frac{  K_\epsilon( z, x_j)  }{D_j}   }
	{  \sum_{j=1}^N \frac{ K_\epsilon( z, x_j) }{D_j}} 
	 =   O(\epsilon^{-1})\| f\|_\infty,
\end{align*}
where in the 2nd inequality used  that 
\[
\| \nabla_z K_\epsilon(z ,x_j)  \|_{op}
= \frac{\| {\rm Proj}_{T_z \calM} (x_j-z) \| }{2\epsilon} K_\epsilon(z ,x_j) 
\le \frac{\| z - x_j\|}{2\epsilon}  K_\epsilon(z ,x_j)
\]
by that $K_\epsilon (z, x_j) = \frac{1}{(4\pi \epsilon)^{d/2}}e^{- \| z - x_j \|^2/4 \epsilon}$.
This shows that $\| \nabla \bar F_n \|_\infty = O(\epsilon^{-1})\| f\|_\infty$.
A refined analysis can bound $\| \nabla \bar F_n \|_\infty$ to be $O(\epsilon^{-1/2})$ instead of $O(\epsilon^{-1})$,
but the latter suffices for our purpose.
At the same time, Lemma \ref{lemma:smoothing-estimates}(ii) gives that 
$
\| \nabla f_n \|_\infty \le C_7 (n \epsilon)^{-1/2} \| f\|_\infty = O(\epsilon^{-1/2}) \| f\|_\infty
$
by that $n \ge 1$. Putting together, we have
\[
\| \nabla B \|_\infty 
\le \| \nabla \bar F_n \|_\infty  + \| \nabla f_n \|_\infty = O( \epsilon^{-1}) \| f\|_\infty.
\]

We aim for $O(\epsilon^2)$ control when comparing $B(z)$ to an anchor point $B(z')$, 
and thus we construct an $\delta$-net $\calN$ of $\calM$ (under geodesic distance), where $\delta = \epsilon^3$.
The cardinal number of $\calN$ can be bounded to be $O(\delta^{-d}) = O(\epsilon^{-3d})$.
Under Assumption \ref{assump:sigma-large-N}, we further have
\[
|\calN| = O(\epsilon^{-3d}) \ll ({N}/{\log N})^6 \ll N^6.
\]
At every point $ z' \in \calN$, 
one can prove the concentration 
$
\sum_{j=1}^N \frac{ K_\epsilon(z', x_j)}{  D_j} 
 =  1+  O \Big( \epsilon + \sqrt{ \frac{\log N}{N \epsilon^{d/2}} } \Big)
$ 
in the same way as Lemma \ref{lemma:Di-concen-eps2}(ii), 
and the analogue of Lemma \ref{lemma:one-step-C2-pC3} to bound 
\[
|\bar F_n(z') - f_{n}(z')|=
\Big| \frac{ \sum_{j=1}^N \frac{ K_\epsilon( z', x_j) }{D_j}  g (x_j)}
		{\sum_{j=1}^N \frac{ K_\epsilon( z', x_j) }{D_j}}
		 - Q_{\sigma^2} g (z') \Big|
\]
in the same way for $g=f_{n-1}$ for all the three cases of $g$, 
with tail probability $O(N^{-10})$ at the point $z'$.
Specifically, when $n \ge 2$, $f_{n-1} \in C^\infty(\calM)$ and use the $g \in C^2$ case;
if $n=1$, use the $g \in C^{0,\beta}(\calM) $ or $L^\infty(\calM)$ case.
Take a union bound, we prove the uniform error bound in the last step with tail probability $\ll N^{-4} $. 
This gives that  $\max_{z \in \calN }|B(z')|$ satisfies the same one-step error bound as in Lemma \ref{lemma:one-step-C2-pC3}.
Then, $\forall z \in \calM$, $\exists z' \in \calN$ s.t. $d_\calM(z, z') \le \delta$, and then
\[
|B(z) | \le |B(z')| + \|\nabla B\|_\infty \delta 
\le \max_{z \in \calN }|B(z')| +  O( \epsilon^2) \| f\|_\infty,
\]
that is, $\sup_{z \in \calM} |B(z)|$ is bounded with another added  $O( \epsilon^2) \| f\|_\infty$.
This proves the claim \eqref{eq:claim-Lip-generalization}.
The tail probability of the first $n-1$ steps is bounded by $N^{-7}$ same as before. The overall tail probability is thus bounded by $N^{-4}$ with large $N$.
\end{proof}

\subsection{Manifold heat kernel estimator}\label{subsec:heat-kernel-estimator}

The previous results also suggest a natural estimator for the manifold heat kernel. 
For each sample point \(x_j\), let
\[
f_{\sigma,j}(x)=K_\epsilon(x,x_j), \qquad \epsilon=\sigma^2.
\]
Then \(W_{\cdot, j}=\rho_X f_{\sigma,j}\). Since \(\tilde P\) approximates the heat semigroup over one time step \(\epsilon\), applying \(\tilde P^{n-1}\) to \(W_{\cdot j}\) gives
\[
\tilde P^{n-1}W_{\cdot j}
=
\tilde P^{n-1}\rho_X f_{\sigma,j}
\approx
\rho_X Q_{(n-1)\epsilon}f_{\sigma,j}.
\]
The ambient Gaussian kernel \(K_\epsilon(\cdot,x_j)\) is a short-time approximation of the manifold heat kernel \( \calH_\epsilon(\cdot,x_j)\), and hence, formally,
\[
Q_{(n-1)\epsilon}f_{\sigma,j}
\approx
Q_{(n-1)\epsilon} \calH_\epsilon(\cdot,x_j)
=
\calH_{n\epsilon}(\cdot,x_j).
\]
This motivates the asymmetric estimator of $\calH_t(x_i, x_j)$ on the $N$ points as
\begin{equation}\label{eq:def-hatKn}
\hat K_n
:=
\tilde P^{n-1}W
=
(D_s^{-1}P^T)^{n-1}W.
\end{equation}

However, Theorem \ref{thm:pC3} does not directly provide an entrywise convergence rate for 
\(\hat K_n \). The theorem applies to fixed test functions \(f\) with bounded regularity norms, whereas here the test functions \(f_{\sigma,j}= K_\epsilon(\cdot,x_j)\) depend on \(\sigma\) and become increasingly localized as \(\sigma \to0 \). 
The function norms of $f_{\sigma,j}$ grow polynomially in \(\sigma^{-1}\), so a direct application of our techniques gives bounds that are too crude for pointwise heat-kernel estimation, particularly in dimensions \(d\ge2\). 
Establishing sharp entrywise convergence rates for heat kernel estimation is left to future work.

Since the population heat kernel is symmetric, we also consider the symmetrized estimator
\begin{equation}\label{eq:def-hatHn}
\hat H_n  = (WD^{-1} D_s^{-1})^{n-1}   W.
\end{equation}
Because \(D^{-1}D_s^{-1}\) is diagonal and \(W=W^T\), this matrix is symmetric.
We numerically evaluate the validity of the proposed estimator in Section \ref{sec:experiment}.

\section{Numerical experiments}\label{sec:experiment}

We numerically compute the estimator of $Q_t f$ and the manifold heat kernel $\calH_t$ using simulated data.
Codes are available at \url{https://github.com/xycheng/heat_semigroup}.

\paragraph{$Q_t f $ from discontinuous $f$}

\begin{figure}[t]
\hspace{-10pt}
\begin{subfigure}{0.25\linewidth}
\hspace{-8pt}
\includegraphics[width=1.05\linewidth]{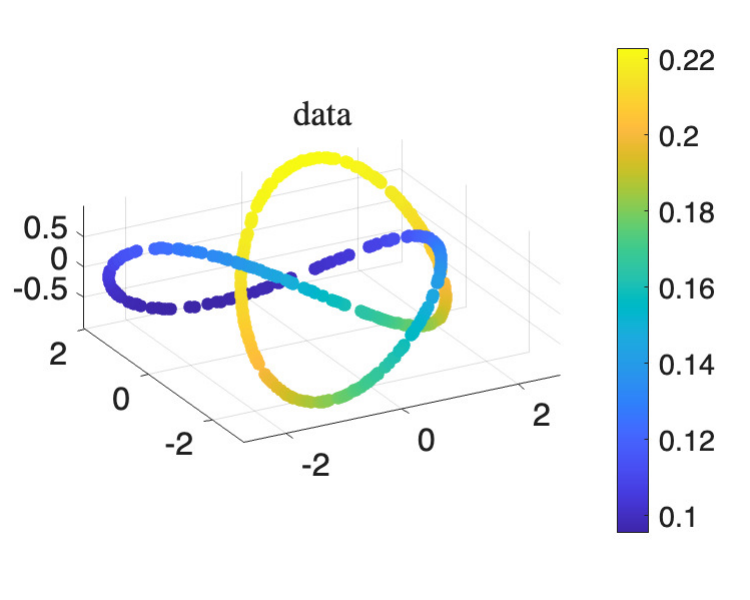}
\subcaption{}
\end{subfigure}
\begin{subfigure}{0.25\linewidth}
\hspace{-8pt}
\includegraphics[width=\linewidth]{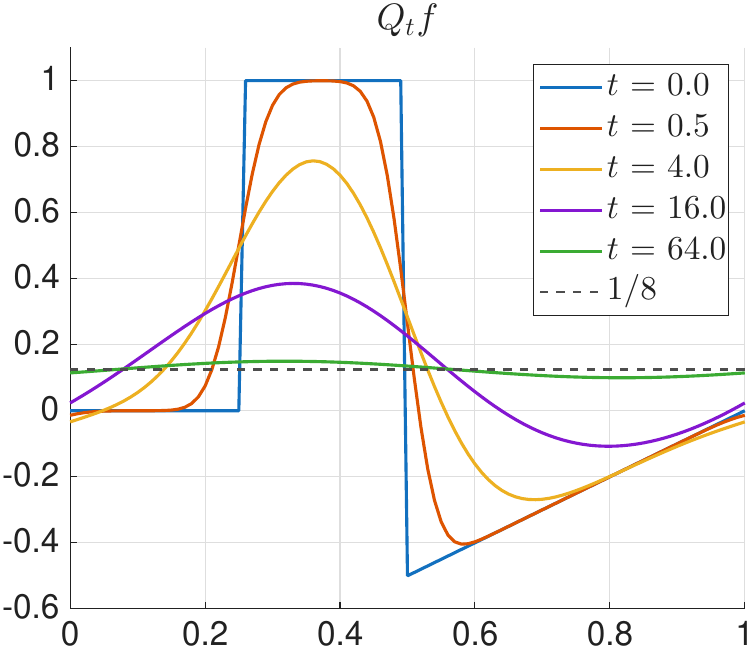}
\subcaption{}
\end{subfigure}
\begin{subfigure}{0.25\linewidth}
\hspace{-8pt}
\includegraphics[width=\linewidth]{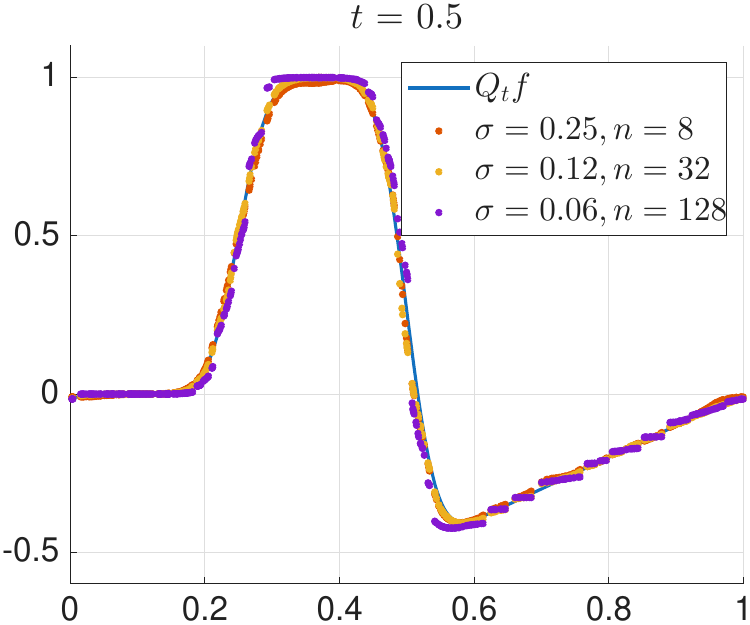}
\subcaption{}
\end{subfigure}
\begin{subfigure}{0.25\linewidth}
\hspace{-8pt}
\includegraphics[width=\linewidth]{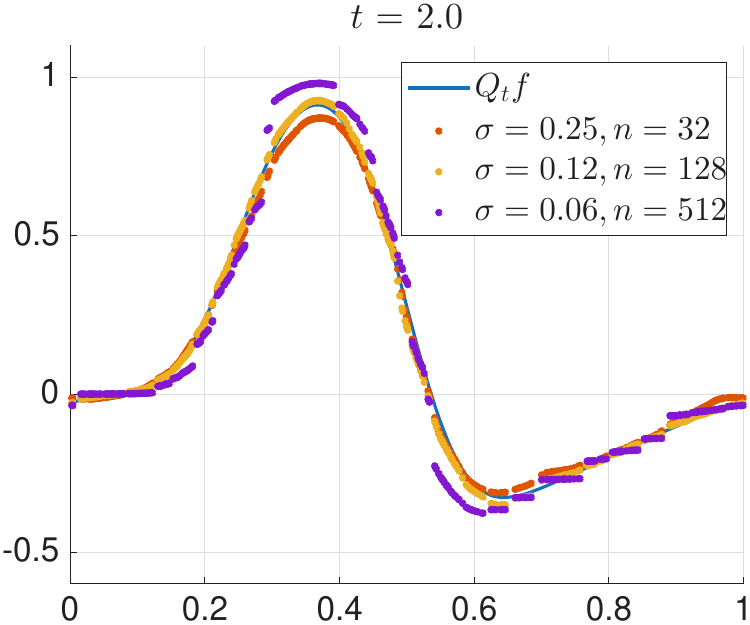}
\subcaption{}
\end{subfigure}
\caption{
(a) Data samples of 1D curve embedded in $\R^3$, colored by sampled density $p$. $N = 500$.
(b) True solutions of $Q_t f$ from an initial value $f$ that has discontinuity. 
(c)(d) Comparison of the estimated $Q_t f $ with the ground truth, at different values of $\sigma$ (and $n$).
}
\label{fig:S1-Qtf}
\end{figure}

The manifold $\calM$ is a closed 1D curve embedded in $\R^3$, sampled at $N=500$ data points from a non-uniform density $p$. 
Figure \ref{fig:S1-Qtf}(a) depicts the data samples colored by the sampling density $p$ evaluated at each point.
The initial function $f$, expressed in the intrinsic (arc-length) coordinate $s \in [0,1]$, is given by
\[
f(s) = \begin{cases}
0, \quad 0 \le s \le 1/4, \\
1, \quad 1/4 < s < 1/2, \\
s-1, \quad 1/2 \le s \le 1. \\
\end{cases}
\]
This $f$ has discontinuity, see (b), and thus $f \in L^\infty(\calM) \backslash C(\calM)$.  
The estimator is computed as $\tilde P^n \rho_X f$ based on Section \ref{subsec:non-uniform-p},
with diffusion time $t = n\sigma^2$ fixed across three $(\sigma,n)$ pairs. 
Plots (c) and (d) show the results for two values of $t$.

The analytical solution of $Q_t f$ can be computed by Fourier series on $S^1$.
As $t$ increases, the solution evolves smoothly from the discontinuous initial $f$ toward the weighted mean $\int_\calM f p dV = 1/8 $ (dashed line in b).
Across the kernel bandwidth choices $\sigma \in \{1/4, 1/8, 1/16\}$ (with $n = t/\sigma^2$), the estimator closely matches the true solution. 
When $\sigma$ is small, the graph approaches the connectivity threshold, leading to local irregularities due to insufficient connectivity. 
In addition, smaller $\sigma$ leads to higher variance across runs (not shown here, as the figure presents a single realization). 
When $t$ increases from $0.5$ to $2.0$, the result with larger $\sigma$ ($\sigma = 0.25$) exhibits more visible bias error.
This is in agreement with the rate bound derived in our theory.

\paragraph{Heat kernel estimator}

\begin{figure}[t]
\centering
\begin{subfigure}{0.275\linewidth}
\includegraphics[width=\linewidth]{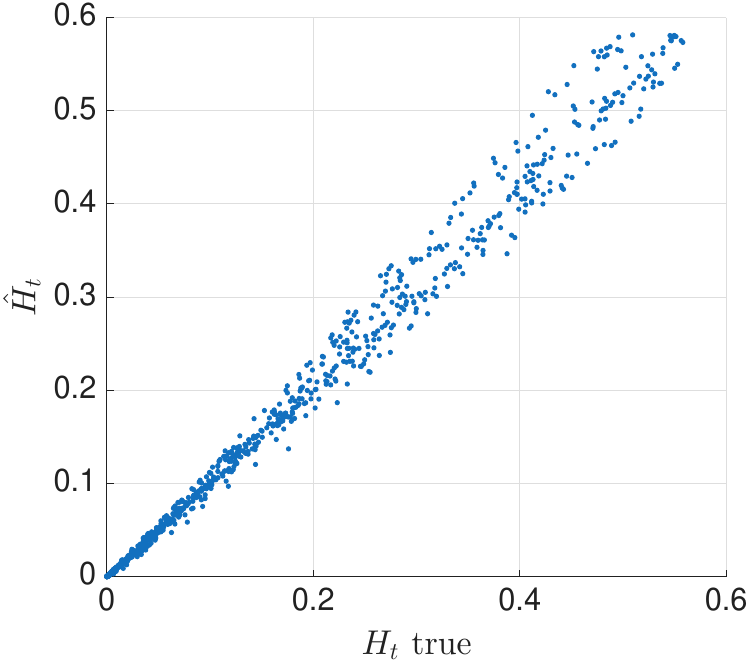}
\subcaption{}
\end{subfigure}
\begin{subfigure}{0.3\linewidth}
\includegraphics[width=\linewidth]{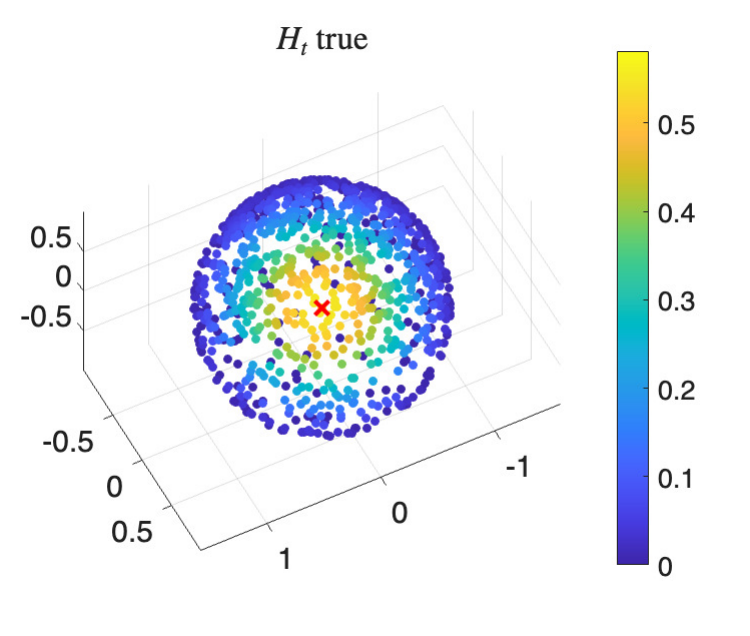}
\subcaption{}
\end{subfigure}
\begin{subfigure}{0.3\linewidth}
\includegraphics[width=\linewidth]{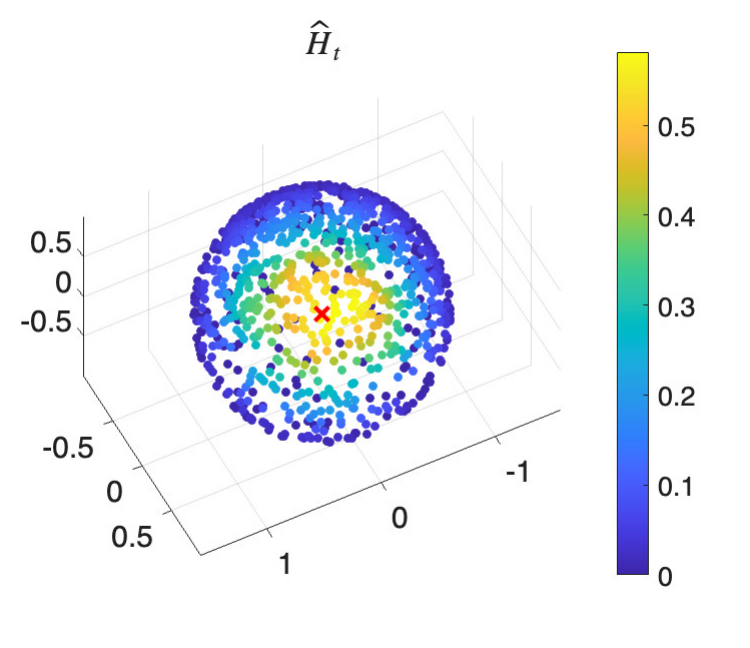}
\subcaption{}
\end{subfigure}
\caption{
Heat kernel estimator on $S^2$ using  $N = 1000$ data points sampled non-uniformly.
The estimator $\hat H_n$ is as in Section \ref{subsec:heat-kernel-estimator}.
}
\label{fig:S2-Ht}
\end{figure}

We sample $N = 1000$ data points on the unit sphere $S^2$ embedded in $\R^3$ from a density $p \propto e^{c x_3 } $, with $c= 1.5$.
The true manifold heat kernel $\calH_t$ is computed by spherical harmonic series.
We set $t = 0.15$ and compute the estimator $\hat H_n$ by equation \eqref{eq:def-hatHn}, where $\sigma = 0.1$ and $n =15$.
The results are shown in Figure \ref{fig:S2-Ht},
where (b) and (c) compare $\rho_X \calH_t(\cdot, x_0)$ 
and the corresponding row of $\hat H_n$ at a single locations $x_0$.
The vector $\infty$-norm error  and 2-norm error are \texttt{9.62e-02}
and \texttt{1.85e-02} respectively, for this case. 

\section*{Acknowledgments}

XC was partially supported by NSF DMS-2237842,
DMS-2031849,
and the Simons Foundation MPS-MODL-00814643.
NW was partially supported by the Simons Foundation MPS-TSM-00002707.

\bibliographystyle{plain}
\bibliography{kernel.bib}

@article{kovachki2023neural,
  title={Neural operator: Learning maps between function spaces with applications to pdes},
  author={Kovachki, Nikola and Li, Zongyi and Liu, Burigede and Azizzadenesheli, Kamyar and Bhattacharya, Kaushik and Stuart, Andrew and Anandkumar, Anima},
  journal={Journal of Machine Learning Research},
  volume={24},
  number={89},
  pages={1--97},
  year={2023}
}

@article{li2021fourier,
  title={Fourier neural operator for parametric partial differential equations},
  author={Li, Zongyi and Kovachki, Nikola and Azizzadenesheli, Kamyar and Liu, Burigede and Bhattacharya, Kaushik and Stuart, Andrew and Anandkumar, Anima},
  journal={International Conference on Learning Representations},
  year={2021}
}

@article{calder2022lipschitz,
  title={Lipschitz regularity of graph {Laplacians} on random data clouds},
  author={Calder, Jeff and Garcia Trillos, Nicolas and Lewicka, Marta},
  journal={SIAM Journal on Mathematical Analysis},
  volume={54},
  number={1},
  pages={1169--1222},
  year={2022},
  publisher={SIAM}
}

@article{cheng2022convergence,
  title={Convergence of graph {Laplacian} with {kNN} self-tuned kernels},
  author={Cheng, Xiuyuan and Wu, Hau-Tieng},
  journal={Information and Inference: A Journal of the IMA},
  volume={11},
  number={3},
  pages={889--957},
  year={2022},
  publisher={Oxford University Press}
}

@article{cheng2024improved,
      title={Improved convergence rate of kNN graph Laplacians: differentiable self-tuned affinity}, 
      author={Xiuyuan Cheng and Yixuan Tan and Nan Wu},
      journal={arXiv preprint arXiv:2410.23212},
      year={2024}
}

@article{shen2022scalability,
  title={Scalability and robustness of spectral embedding: landmark diffusion is all you need},
  author={Shen, Chao and Wu, Hau-Tieng},
  journal={Information and Inference: A Journal of the IMA},
  volume={11},
  number={4},
  pages={1527--1595},
  year={2022},
  publisher={Oxford University Press}
}

@article{cheng2024bi,
  title={Bi-stochastically normalized graph Laplacian: convergence to manifold Laplacian and robustness to outlier noise},
  author={Cheng, Xiuyuan and Landa, Boris},
  journal={Information and Inference: A Journal of the IMA},
  volume={13},
  number={4},
  pages={iaae026},
  year={2024},
  publisher={Oxford University Press}
}

@article{trillos2025minimax,
  title={Minimax Rates for the Estimation of Eigenpairs of Weighted Laplace-Beltrami Operators on Manifolds},
  author={Trillos, Nicol{\'a}s Garc{\'\i}a and Li, Chenghui and Venkatraman, Raghavendra},
  journal={arXiv preprint arXiv:2506.00171},
  year={2025}
}

@article{calder2022improved,
  title={Improved spectral convergence rates for graph {Laplacians} on $\varepsilon$-graphs and {$k$-NN} graphs},
  author={Calder, Jeff and Trillos, Nicolas Garcia},
  journal={Applied and Computational Harmonic Analysis},
  volume={60},
  pages={123--175},
  year={2022},
  publisher={Elsevier}
}

@article{singer2017spectral,
  title={Spectral convergence of the connection {Laplacian} from random samples},
  author={Singer, Amit and Wu, Hau-Tieng},
  journal={Information and Inference: A Journal of the IMA},
  volume={6},
  number={1},
  pages={58--123},
  year={2017},
  publisher={Oxford University Press}
}

@article{hein2007graph,
  title={Graph {Laplacians} and their convergence on random neighborhood graphs.},
  author={Hein, Matthias and Audibert, Jean-Yves and Luxburg, Ulrike von},
  journal={Journal of Machine Learning Research},
  volume={8},
  number={6},
  year={2007}
}

@article{belkin2008towards,
  title={Towards a theoretical foundation for {Laplacian}-based manifold methods},
  author={Belkin, Mikhail and Niyogi, Partha},
  journal={Journal of Computer and System Sciences},
  volume={74},
  number={8},
  pages={1289--1308},
  year={2008},
  publisher={Elsevier}
}

@book{lunardi2012analytic,
  title={Analytic semigroups and optimal regularity in parabolic problems},
  author={Lunardi, Alessandra},
  year={2012},
  publisher={Springer Science \& Business Media}
}

@article{hsu1999estimates,
  title={Estimates of derivatives of the heat kernel on a compact Riemannian manifold},
  author={Hsu, Elton},
  journal={Proceedings of the american mathematical society},
  volume={127},
  number={12},
  pages={3739--3744},
  year={1999}
}

@article{tang2026adaptive,
  title={Adaptive Bayesian regression on data with low intrinsic dimensionality},
  author={Tang, Tao and Wu, Nan and Cheng, Xiuyuan and Dunson, David},
  journal={The Annals of Statistics},
  volume={54},
  number={2},
  pages={1080--1099},
  year={2026},
  publisher={Institute of Mathematical Statistics}
}

@article{cheng2022eigen,
  title={Eigen-convergence of Gaussian kernelized graph Laplacian by manifold heat interpolation},
  author={Cheng, Xiuyuan and Wu, Nan},
  journal={Applied and Computational Harmonic Analysis},
  volume={61},
  pages={132--190},
  year={2022},
  publisher={Elsevier}
}

@article{grigor1997gaussian,
  title={Gaussian upper bounds for the heat kernel on arbitrary manifolds},
  author={Grigor'yan, Alexander},
  journal={Journal of Differential Geometry},
  year={1997},
  volume={45},
  pages={33-52}
}

@book{rosenberg1997laplacian,
  title={The {Laplacian} on a {Riemannian} manifold: An introduction to analysis on manifolds},
  author={Rosenberg, Steven},
  number={31},
  year={1997},
  publisher={Cambridge University Press}
}

@article{dunson2021spectral, 
  title = {Spectral convergence of graph {Laplacian} and heat kernel reconstruction in {$L^\infty$} from random samples},
  author={Dunson, David B and Wu, Hau-Tieng and Wu, Nan},
  journal = {Applied and Computational Harmonic Analysis},
  volume = {55},
  pages = {282-336},
  year = {2021}
}

@article{marshall2019manifold,
  title={Manifold learning with bi-stochastic kernels},
  author={Marshall, Nicholas F and Coifman, Ronald R},
  journal={IMA Journal of Applied Mathematics},
  volume={84},
  number={3},
  pages={455--482},
  year={2019},
  publisher={Oxford University Press}
}

@article{wormell2021spectral,
title = {Spectral convergence of diffusion maps: {Improved} error bounds and an alternative normalization},
author = {Wormell, Caroline L. and Reich, Sebastian},
journal = {SIAM Journal on Numerical Analysis},
volume = {59},
number = {3},
pages = {1687-1734},
year = {2021}
}

@article{long2017landmark,
  title={Landmark diffusion maps (L-dMaps): Accelerated manifold learning out-of-sample extension},
  author={Long, Andrew W and Ferguson, Andrew L},
  journal={Applied and Computational Harmonic Analysis},
  year={2017},
  publisher={Elsevier}
}

@article{berry2016local,
  title={Local kernels and the geometric structure of data},
  author={Berry, Tyrus and Sauer, Timothy},
  journal={Applied and Computational Harmonic Analysis},
  volume={40},
  number={3},
  pages={439--469},
  year={2016},
  publisher={Elsevier}
}

@article{bermanis2016measure,
  title={Measure-based diffusion grid construction and high-dimensional data discretization},
  author={Bermanis, Amit and Salhov, Moshe and Wolf, Guy and Averbuch, Amir},
  journal={Applied and Computational Harmonic Analysis},
  volume={40},
  number={2},
  pages={207--228},
  year={2016},
  publisher={Elsevier}
}

@article{singer2009detecting,
  title={Detecting intrinsic slow variables in stochastic dynamical systems by anisotropic diffusion maps},
  author={Singer, Amit and Erban, Radek and Kevrekidis, Ioannis G and Coifman, Ronald R},
  journal={Proceedings of the National Academy of Sciences},
  volume={106},
  number={38},
  pages={16090--16095},
  year={2009},
  publisher={National Acad Sciences}
}

@article{talmon2013empirical,
  title={Empirical intrinsic geometry for nonlinear modeling and time series filtering},
  author={Talmon, Ronen and Coifman, Ronald R},
  journal={Proceedings of the National Academy of Sciences},
  volume={110},
  number={31},
  pages={12535--12540},
  year={2013},
  publisher={National Acad Sciences}
}

@article{belkin2003laplacian,
  title={{Laplacian} eigenmaps for dimensionality reduction and data representation},
  author={Belkin, Mikhail and Niyogi, Partha},
  journal={Neural Computation},
  volume={15},
  number={6},
  pages={1373--1396},
  year={2003},
  publisher={MIT Press}
}

@article{trillos2020error,
  title={Error estimates for spectral convergence of the graph {Laplacian} on random geometric graphs toward the {Laplace--Beltrami} operator},
  author={Trillos, Nicol{\'a}s Garc{\'\i}a and Gerlach, Moritz and Hein, Matthias and Slep{\v{c}}ev, Dejan},
  journal={Foundations of Computational Mathematics},
  volume={20},
  number={4},
  pages={827--887},
  year={2020},
  publisher={Springer}
}

@article{coifman2006diffusion,
  title={Diffusion maps},
  author={Coifman, Ronald R and Lafon, St{\'e}phane},
  journal={Applied and Computational Harmonic Analysis},
  volume={21},
  number={1},
  pages={5--30},
  year={2006},
  publisher={Elsevier}
}

@article{singer2006graph,
  title={From graph to manifold {Laplacian}: The convergence rate},
  author={Singer, Amit},
  journal={Applied and Computational Harmonic Analysis},
  volume={21},
  number={1},
  pages={128--134},
  year={2006},
  publisher={Elsevier}
}

@inproceedings{ting2010analysis,
  title={An analysis of the convergence of graph {Laplacians}},
  author={Ting, Daniel and Huang, Ling and Jordan, Michael I},
  booktitle={Proceedings of the 27th International Conference on International Conference on Machine Learning},
  pages={1079--1086},
  year={2010}
}

@article{hormander1968spectral,
  title={The spectral function of an elliptic operator},
  author={H{\"o}rmander, Lars},
  journal={Acta mathematica},
  volume={121},
  number={1},
  pages={193--218},
  year={1968},
  publisher={Springer}
}

@article{donnelly2006eigenfunctions,
  title={Eigenfunctions of the {Laplacian} on compact {Riemannian} manifolds},
  author={Donnelly, Harold},
  journal={Asian Journal of Mathematics},
  volume={10},
  number={1},
  pages={115--126},
  year={2006},
  publisher={International Press of Boston}
}

@article{hassannezhad2016eigenvalue,
  title={Eigenvalue inequalities on {Riemannian} manifolds with a lower {Ricci} curvature bound},
  author={Hassannezhad, Asma and Kokarev, Gerasim and Polterovich, Iosif},
  journal={Journal of Spectral Theory},
  volume={6},
  number={4},
  pages={807--835},
  year={2016}
}

@book{gilbarg1998elliptic,
  title={Elliptic partial differential equations of second order},
  author={Gilbarg, David and Trudinger, Neil S},
  volume={2},
  number={1},
  year={1998},
  publisher={Springer}
}

\appendix
\setcounter{table}{0}
\setcounter{figure}{0}
\renewcommand{\thetable}{A.\arabic{table}}
\renewcommand{\thefigure}{A.\arabic{figure}}
\renewcommand{\thelemma}{A.\arabic{lemma}}

\setcounter{assumption}{0} \renewcommand{\theassumption}{A.\arabic{assumption}}

\section{Proofs}

\section{Supporting lemmas}\label{app:lemmas}

\subsection{ Heat kernel parametrix and decay}

The following lemma is based on Chapter 3 of \cite{rosenberg1997laplacian}
and was reproduced as Theorem 2.1 in \cite{cheng2022eigen}.
We largely follow the statement therein. 

\begin{lemma}
\label{lemma:Heat-short-time}
Suppose $\calM$ is as in Assumption \ref{assump:M}, and $m > d/2+2 $ is a positive integer.
Let $G_t(x,y)$ be defined as in \eqref{eq:def-Geps}.
There are  positive constants $t_0 < 1$, $\delta_0 < \xi$ the injective radius of $\calM$, 
both depending on $\calM$, and 
\begin{itemize}
\item[(i)] 
There are positive constants $C_1$, $C_2$  which depending on $\calM$, 
 and  $u_0, \cdots, u_m$ $\in C^{\infty}(\calM  \times \calM )$, such that, 
 when $ t < t_0$,  $\forall x \in \calM$, 
\begin{equation}\label{eq:parametrix-m}
\Big| \calH_t( x,y) - G_t(x,y) \Big( \sum_{l=0}^m t^l u_l(x,y) \Big) \Big| \le C_2 t^{m-d/2+1}, 
\quad
\forall y \in \calM, \, d_\calM( y,x) < \delta_0;
\end{equation}
In addition, using normal coordinates at $x$ to write $y=\exp_x(s\theta)$, where $s = d_\calM(x,y)$ and $\theta \in S^{d-1} \subset T_x \calM$, 
\[
u_0(x,y) 
    = 1+ \frac{1}{12} \texttt{Ric}_{x}(\theta,\theta) s^2
    + r_U(s), 
    \quad |r_U(s)| \le C_1 s^3,
\]
and the bound is uniform in $x$ and $\theta$.

\item[(ii)]
There is positive constant $C_3$ depending on $\calM$ such that, when $ t < t_0$, 
\begin{equation}\label{eq:H-decay}
\calH_t( x,y ) \le  C_3 t^{-d/2} e^{- \frac{ d_\calM( x,y)^2}{ 5 t}},
\quad 
\forall x, y \in \calM.
\end{equation}
\end{itemize}
\end{lemma} 

The verification of \eqref{eq:parametrix-m} and \eqref{eq:H-decay} was carried out in Theorem 2.1 in \cite{cheng2022eigen}.
The expression of $u_0(x,y)$ follows standard calculation, see e.g. Chapter 3.2 of \cite{rosenberg1997laplacian}.

\subsection{Heat kernel derivative bounds}\label{subsec:Ht-der-bounds}

The goal of this section is to prove Lemma \ref{lemma:long-t-grad-Ht}. 
In this section, the constants $c_1, c_2, \cdots$ are local and stand for manifold-dependent constants,
and $g$ stands for the  Riemannian metric tensor.
Recall that $\Delta$  is Laplace-Beltrami operator for manifold functions.
Let $\{ \lambda_l, \phi_l \}_{l=0}^\infty$ be the eigenpairs of $-\Delta$ with $\phi_l$ normalized in $L^2(\mathcal{M})$.  
For manifold tensors, $\Delta$  is the rough Laplacian: 
 $\Delta T = \sum g^{ij}\nabla_i \nabla_j T$ maps a $k$ tensor to a $k$ tensor.

We introduce some notations about tensor innerproduct and norms. 
Suppose that $S$ and  $S'$ are $k$ tensor fields  on $\mathcal{M}$. Then, 
\[
(S,S')=\sum_{i_1, \cdots, i_k, j_1, \cdots, j_k = 1}^d g^{i_1 j_1} \cdots g^{i_kj_k}S_{i_1, \cdots, i_k} S'_{j_1, \cdots, j_k}
\]
 and $|S|^2=(S,S)$. Note that in normal coordinates at $x \in \mathcal{M}$, $g^{ij}(x)=\delta_{ij}$. Then, we have $|S|^2(x)=\sum_{i_1, \cdots, i_k= 1}^d S_{i_1, \cdots, i_k}(x)^2$.
For example, when $k=0$, $| T(x) |$ recovers the absolute value.
When $k=1$, $| T(x)| = (\sum_j T_j(x)^2)^{1/2}$ equals the 2-norm of the vector.
When $k=2$, $| T(x)| = (\sum_{j_1, j_2=1}^d T_{j_1, j_2}(x)^2 )^{1/2}$ equals the Frobenius norm of the matrix. 
We also have  $|T(x) | \ge \| T(x)\|_{op}$.

At last, the Bochner-Weitzenbock equality has that $\Delta|T|^2=2|\nabla T|^2+2(T, \Delta T)$.

\begin{lemma}\label{lemma:long-t-grad-Ht}
Under Assumption \ref{assump:M},
for any $t_2>0$, there exists a constant $C(t_2)>0$  such that $\forall t \geq t_2$,
\begin{align*} 
|\nabla^k_x \mathcal{H}_t(x,y) | \leq C(t_2) e^{- \lambda_1 t /2},
\quad   \forall x, y \in \mathcal{M}, 
\quad k = 1,2,
\end{align*}
where $\lambda_1$ is the first non-zero eigenvalue of $-\Delta$.
The constant $C(t_2)$ depends on $t_2$ and also manifold geometry ($d$, the diameter, volume, injectivity radius, bounds of the Ricci curvature and its covariant derivative,  and the sectional curvature bounds of $\mathcal{M}$).
\end{lemma}

\begin{lemma}[\cite{hormander1968spectral,donnelly2006eigenfunctions}]\label{lemma hormander}
For any integer $l>0$, we have the following bound:  
\begin{align}
\|\phi_l\|_\infty \leq c_1 \lambda_l^{\frac{d-1}{4}} \|\phi_l\|_2= c_1 \lambda_l^{\frac{d-1}{4}}\,,   \nonumber 
\end{align}
where $c_1$ is a constant depending on the injectivity radius and the sectional curvature bounds of the manifold $\mathcal{M}$.
\end{lemma}

{\begin{lemma}[\cite{hassannezhad2016eigenvalue}]\label{laplace eigenvalue lower bound}
For any integer $l \geq 0$, we have
\begin{align}
c_2 l^{2/d} \leq \lambda_l \leq c_3 l^{2/d}, \nonumber 
\end{align}
 where $c_2>0$ depends on $d$, the Ricci curvature lower bound, and the diameter of $\mathcal{M}$ and  $c_3>0$ depends on $d$, the Ricci curvature lower bound, and the volume of $\mathcal{M}$. 
\end{lemma}
 
The next two lemmas are elliptic regularity results for $-\Delta$. When the domain of $-\Delta$ is an open set or the closure of an open set in Euclidean space, such regularity results are developed in \cite[Chapter 6]{gilbarg1998elliptic}. When the domain is an abstract manifold, we prove the following two lemmas using the Bochner-Weitzenbock's identity technique. This approach also allows us to determine explicitly the dependence of the constants appearing in the estimates.

\begin{lemma}\label{elliptic regularity 1}
For any integer $l \geq 0$, we have
\begin{align}
\|\nabla \phi_l \|_\infty \leq c_4 \sqrt{\lambda_l} \| \phi_l\|_\infty,
\end{align}
 where $c_4>0$ depends on $d$, the Ricci curvature lower bound, and the diameter of $\mathcal{M}$.
\end{lemma}
\begin{proof}
We only need to prove the case when $l>0$ as $\phi_0$ is a constant. Suppose  that $\mathcal{M}$ has a Ricci curvature lower bound $\texttt{Ric} \geq -Kg$ with $K>0$.
By the Bochner's formula and $-\Delta \phi_l=\lambda_l \phi_l$, we have
\begin{align*}
\frac{1}{2}\Delta |\nabla \phi_l|^2=&|\nabla^2\phi_l|^2+g(\nabla\phi_l, \nabla \Delta \phi_l)+ \texttt{Ric}(\nabla\phi_l, \nabla\phi_l) \\
=&|\nabla^2\phi_l|^2- \lambda_l | \nabla \phi_l(x)|^2+ \texttt{Ric}(\nabla\phi_l, \nabla\phi_l).
\end{align*}
Hence,
\begin{align*}
|\nabla^2\phi_l|^2=& \frac{1}{2}\Delta |\nabla \phi_l|^2+ \lambda_l| \nabla \phi_l(x)|^2- \texttt{Ric}(\nabla\phi_l, \nabla\phi_l)\leq \frac{1}{2}\Delta |\nabla \phi_l|^2+ \lambda_l| \nabla \phi_l(x)|^2+K g(\nabla\phi_l, \nabla\phi_l) \nonumber \\
=& \frac{1}{2}\Delta |\nabla \phi_l|^2+ (\lambda_l+K)| \nabla \phi_l|^2.
\end{align*}
We conclude that
\begin{align}\label{proof: Bochner bound 2}
|\nabla^2\phi_l|^2-(\lambda_l+K)| \nabla \phi_l|^2 \leq \frac{1}{2}\Delta |\nabla \phi_l|^2. 
\end{align}
Moreover, by the product rule and $-\Delta \phi_l=\lambda_l \phi_l$,
\begin{align}\label{proof: Bochner bound 3}
\Delta (\phi_l^2) =2 \phi \Delta \phi_l +2 g( \nabla\phi_l , \nabla\phi_l)=-2\lambda_l \phi_l^2+2|\nabla\phi_l|^2.
\end{align}
Define $u= |\nabla \phi_l|^2+ B \lambda_l \phi_l^2$ with the constant $B$ to be chosen. Then, by \eqref{proof: Bochner bound 2} and \eqref{proof: Bochner bound 3}, we have
\begin{align*}
\Delta u= \Delta |\nabla \phi_l|^2+ B \lambda_l \Delta (\phi_l^2) \geq& 2|\nabla^2\phi_l|^2-2 (\lambda_l+K)| \nabla \phi_l|^2+ B \lambda_l (-2\lambda_l \phi_l^2+2|\nabla\phi_l|^2) \\
=&2|\nabla^2\phi_l|^2+2 (B\lambda_l-\lambda_l-K)| \nabla \phi_l|^2-2B \lambda_l^2 \phi_l^2
\end{align*}
Since $\mathcal{M}$ is compact, $u$ achieves maximum at $x_0 \in \mathcal{M}$ with $\Delta u (x_0) \leq 0$. Hence, by the above inequality,
\begin{align*}
2|\nabla^2\phi_l|^2(x_0)+2 (B\lambda_l-\lambda_l-K)| \nabla \phi_l|^2(x_0)-2B \lambda_l^2 \phi_l^2(x_0)\leq \Delta u(x_0) \leq 0.
\end{align*}
Since $2|\nabla^2\phi_l|^2(x_0) \geq 0$, we have
\begin{align}\label{proof: Bochner bound 4}
2 (B\lambda_l-\lambda_l-K)| \nabla \phi_l|^2(x_0)-2B \lambda_l^2 \phi_l^2(x_0) \leq 0.
\end{align}
Since $l>0$, $\lambda_l >0$. We specify $B>0$ so that $B\lambda_l-\lambda_l-K = \frac{B\lambda_l}{2}>0$ which is equivalent to $B =2( 1+\frac{K}{\lambda_l})$.  By \eqref{proof: Bochner bound 4}, we have $|\nabla \phi_l|^2(x_0) \leq \frac{ 2B\lambda_l^2 \phi_l^2(x_0)}{2 (B\lambda_l-\lambda_l-K)} \leq 2 \lambda_l \phi_l^2(x_0)$. 

Since $u$ achieves maximum at $x_0$, for any $x \in \mathcal{M}$
\begin{align*}
|\nabla \phi_l|^2(x) \leq |\nabla \phi_l|^2(x)+ B \lambda_l \phi_l^2(x) =u(x) \leq u(x_0)= |\nabla \phi_l|^2(x_0)+ B \lambda_l \phi_l^2(x_0) \leq  2 \lambda_l \phi_l^2(x_0)+ B \lambda_l \phi_l^2(x_0).
\end{align*}
Hence, $|\nabla \phi_l|(x) \leq \sqrt{2+B} \sqrt{\lambda_l} |\phi_l|(x_0)\leq  \sqrt{2+B} \sqrt{\lambda_l} \|\phi_l\|_\infty$. Note that  by Lemma \ref{laplace eigenvalue lower bound},
$$\sqrt{2+B}=\sqrt{2+2( 1+\frac{K}{\lambda_l})} \leq \sqrt{2 (2+\frac{K}{\lambda_1})} \leq \sqrt{2 (2+\frac{K}{c_2})} =c_4.$$
This proves the lemma.
\end{proof}

\begin{lemma}\label{elliptic regularity 2}
For any integer $l \geq 0$, we have
\begin{align}
\sup_{x \in \mathcal{M}}| \nabla^2 \phi_l | (x)
\leq c_5 \sqrt{\lambda_l} 
\sup_{x \in \mathcal{M}} | \nabla \phi_l | (x),
\end{align}
where $c_5>0$ depends on $d$, the diameter, the bounds of the Ricci curvature and its covariant derivative,  and the sectional curvature bounds of $\mathcal{M}$.
\end{lemma}
\begin{proof}
We only need to prove the case when $l>0$ as $\phi_0$ is a constant.  By applying the Bochner-Weitzenbock's identity to $\nabla^2 \phi_l$, we have 
\begin{align}\label{Weitzenbock bound 0}
\Delta|\nabla^2 \phi_l|^2=2|\nabla( \nabla^2 \phi_l)|^2+2(\nabla^2 \phi_l, \Delta \nabla^2 \phi_l)=2| \nabla^3 \phi_l|^2+2(\nabla^2 \phi_l, \Delta\nabla^2 \phi_l).
\end{align}
A straight forward calculation shows that 
\begin{align}\label{Weitzenbock bound 01}
\Delta \nabla_i\nabla_j \phi_l= \nabla_i\nabla_j \Delta\phi_l +\sum_{p,k=1}^d(R_{jp}g_{ik}+R_{ip}g_{jk}-2R_{kipj})\nabla_k\nabla_p\phi_l+\sum_{p=1}^d(\nabla_i R_{jp}+\nabla_j R_{pi}-\nabla_p R_{ij})\nabla_p\phi_l.
\end{align}
where $R_{kipj}$ is the curvature tensor and $R_{pi}$ is the Ricci curvature tensor.

We define the following $2$ tensor fields.  Let $S$ be a $2$ tensor field depending on $\nabla^2 \phi_l$ with $$S_{ij}=\sum_{p,k=1}^d(R_{jp}g_{ik}+R_{ip}g_{jk}-2R_{kipj})\nabla_k\nabla_p\phi_l.$$ 
 Let $S'$ be a $2$ tensor field depending on $\nabla \phi_l$ with 
 $$S'_{ij}=\sum_{p=1}^d(\nabla_i R_{jp}+\nabla_j R_{pi}-\nabla_p R_{ij})\nabla_p\phi_l.$$ 
Substituting \eqref{Weitzenbock bound 01} into \eqref{Weitzenbock bound 0} and using $-\Delta \phi_l=\lambda_l \phi_l$, we have

\begin{align*}
\Delta |\nabla^2 \phi_l|^2=& 2 |\nabla^3\phi_l|^2 +2(\nabla^2 \phi_l, \nabla^2\Delta\phi_l)+2(\nabla^2 \phi_l, S)+2(\nabla^2 \phi_l,S')\\
=& 2 |\nabla^3\phi_l|^2 - 2\lambda_l | \nabla^2 \phi_l|^2+ \mathcal{E}_1(\nabla^2 \phi_l, \nabla^2 \phi_l)+\mathcal{E}_2(\nabla^2 \phi_l, \nabla \phi_l),
\end{align*}
where $\mathcal{E}_1(\nabla^2 \phi_l, \nabla^2 \phi_l)=(\nabla^2 \phi_l, S)$ is a quadratic form on symmetric 2-tensor fields whose coefficients can be expressed by the curvature tensor  and  $\mathcal{E}_2(\nabla^2 \phi_l, \nabla \phi_l)=(\nabla^2 \phi_l, S')$ is a bilinear form whose coefficients can be expressed by the covariant derivative of the Ricci curvature tensor. Since the sectional curvatures determine the curvature tensor, there are constants $c'_1>0$ depending on $d$  and the sectional curvature bounds and  $c'_2>0$ depending on $d$ and the bounds of covariant derivative of the Ricci curvature such that 
\begin{align}\label{Weitzenbock bound 1}
\Delta |\nabla^2 \phi_l|^2 \geq & 2 |\nabla^3\phi_l|^2 - 2\lambda_l | \nabla^2 \phi_l|^2-c'_1 |\nabla^2 \phi_l|^2-c'_2|\nabla \phi_l||\nabla^2 \phi_l| \nonumber \\
=& 2 |\nabla^3\phi_l|^2 - (2\lambda_l +c'_1)| \nabla^2 \phi_l|^2-c'_2|\nabla \phi_l||\nabla^2 \phi_l|.
\end{align}
Define $u= |\nabla^2 \phi_l|^2+ B \lambda_l |\nabla\phi_l|^2$ with the constant $B>0$ to be chosen.  Suppose  that $\mathcal{M}$ has a Ricci curvature lower bound $\texttt{Ric} \geq -Kg$ with $K>0$. Since $ |\nabla^3\phi_l|^2 \geq 0$, by applying \eqref{proof: Bochner bound 2} and \eqref{Weitzenbock bound 1},  we have
\begin{align*}
\Delta u \geq & 2 |\nabla^3\phi_l|^2 - (2\lambda_l +c'_1)| \nabla^2 \phi_l|^2 -c'_2|\nabla \phi_l||\nabla^2 \phi_l|+2 B \lambda_l (|\nabla^2\phi_l|^2-(\lambda_l+K)| \nabla \phi_l|^2) \\
 \geq & (2 B \lambda_l-2\lambda_l-c'_1)|\nabla^2\phi_l|^2 -c'_2|\nabla \phi_l||\nabla^2 \phi_l|- 2 B \lambda_l(\lambda_l+K)| \nabla \phi_l|^2.
\end{align*}
Since $|\nabla \phi_l||\nabla^2 \phi_l| \leq |\nabla \phi_l|^2+|\nabla^2 \phi_l|^2$, we have
$$\Delta u \geq (2 B \lambda_l-2\lambda_l-c'_1-c'_2)|\nabla^2\phi_l|^2- [2 B \lambda_l(\lambda_l+K)+c'_2]| \nabla \phi_l|^2.$$
We will choose $B>0$ so that $2 B \lambda_l-2\lambda_l-c'_1-c'_2>0$ and we have $2 B \lambda_l(\lambda_l+K)+c'_2>0$. Since $\mathcal{M}$ is compact, $u$ achieves maximum at $x_0 \in \mathcal{M}$ with $\Delta u (x_0) \leq 0$. Hence, by the above inequality,
$$ (2 B \lambda_l-2\lambda_l-c'_1-c'_2)|\nabla^2\phi_l|^2(x_0)- [2 B \lambda_l(\lambda_l+K)+c'_2]| \nabla \phi_l|^2(x_0)\leq \Delta u (x_0) \leq 0$$
Then, we have
\begin{align}\label{Weitzenbock bound 2}
|\nabla^2\phi_l|^2(x_0) \leq \frac{2 B \lambda_l(\lambda_l+K)+c'_2}{2 B \lambda_l-2\lambda_l-c'_1-c'_2}| \nabla \phi_l|^2(x_0).
\end{align}
Therefore, by \eqref{Weitzenbock bound 2}, for any $x \in \mathcal {M}$ we have
\begin{align}
|\nabla^2\phi_l|^2(x) \leq &  |\nabla^2 \phi_l|^2(x)+ B \lambda_l |\nabla\phi_l|^2(x)=u(x) \leq u(x_0)= |\nabla^2 \phi_l|^2(x_0)+ B \lambda_l |\nabla\phi_l|^2(x_0) \nonumber\\
\leq& \Big ( \frac{2 B \lambda_l(\lambda_l+K)+c'_2}{2 B \lambda_l-2\lambda_l-c'_1-c'_2}+B \lambda_l \Big) | \nabla \phi_l|^2(x_0)  .\label{Weitzenbock bound 3}
\end{align}
 We choose $B=\frac{c'_1+c'_2+2\lambda_l}{\lambda_l}>0$ so that $2 B \lambda_l-2\lambda_l-c'_1-c'_2=B \lambda_l>0$. For $l>0$, by Lemma \ref{laplace eigenvalue lower bound}, $c_2 \leq \lambda_l$. Therefore,
 \begin{align*} 
\frac{2 B \lambda_l(\lambda_l+K)+c'_2}{2 B \lambda_l-2\lambda_l-c'_1-c'_2}+B \lambda_l =& 2(\lambda_l+K)+\frac{c'_2}{c'_1+c'_2+2\lambda_l}+c'_1+c'_2+2\lambda_l \\
\leq &\lambda_l \left(4+\left(\frac{c'_2}{c'_1+c'_2+2\lambda_l}+c'_1+c'_2+2K\right)/c_2\right)\\
\leq &\lambda_l \left(4+\left(\frac{c'_2}{c'_1+c'_2}+c'_1+c'_2+2K\right)/c_2\right)
\end{align*}
By the above simplification,  \eqref{Weitzenbock bound 3} implies that $\sup_{x \in \mathcal{M}}| \nabla^2 \phi_l  | (x)\leq c_5 \sqrt{\lambda_l} \sup_{x \in \mathcal{M}} | \nabla \phi_l | (x)$ with 
$$c_5=\sqrt{ 4+\left(\frac{c'_2}{c'_1+c'_2}+c'_1+c'_2+2K\right)/c_2}.
$$
\end{proof}

We are ready to prove Lemma \ref{lemma:long-t-grad-Ht}

\begin{proof}[Proof of Lemma \ref{lemma:long-t-grad-Ht}]
Recall that for any $x, y \in \mathcal{M}$ and any $t>0$,
\begin{align*}
\mathcal{H}_t(x, y)=\sum_{i=0}^{\infty}e^{-\lambda_i t} \phi_i(x)\phi_i(y). 
\end{align*}
Therefore for  any $t>0$,
\begin{align}\label{proof: spectral series bound 1}
\sup_{x, y \in \mathcal{M}} |\nabla^k_x \mathcal{H}_t(x,y) | \leq \sum_{i=0}^{\infty}e^{-\lambda_i t/2} e^{-\lambda_i t/2}  \sup_{x \mathcal{M}} |\nabla^k \phi_i |(x) \|\phi_i\|_\infty=&\sum_{i=1}^{\infty}e^{-\lambda_i t/2} e^{-\lambda_i t/2}  \sup_{x \mathcal{M}} |\nabla^k \phi_i |(x) \|\phi_i\|_\infty \nonumber \\
\leq& e^{-\lambda_1 t/2} \sum_{i=1}^{\infty}e^{-\lambda_i t/2}  \sup_{x \mathcal{M}} |\nabla^k \phi_i |(x) \|\phi_i\|_\infty.
\end{align}
where we use that $\phi_0$ is a constant in the second last step and $e^{-\lambda_i t/2} \leq e^{-\lambda_1 t/2} $ for all $i \geq 1$ in the last step.

If we prove that for $t\geq t_2>0$
\begin{align}\label{proof: spectral series bound 3}
\sum_{i=1}^{\infty}e^{-\lambda_i t/2} \sup_{x \in\mathcal{M}} |\nabla^k \phi_i |(x) \|\phi_i\|_\infty \leq C (t_2),
\end{align}
then the lemma is proved by substituting  \eqref{proof: spectral series bound 3} into \eqref{proof: spectral series bound 1}.

\

\noindent
\underline{Proof of \eqref{proof: spectral series bound 3}}: 
By Lemma \ref{lemma hormander}, Lemma \ref{laplace eigenvalue lower bound}, Lemma \ref{elliptic regularity 1}, and Lemma \ref{elliptic regularity 2},
\begin{align*}
&
e^{-\lambda_i t/2} \leq e^{-\frac{1}{2}c_2t_2  i^{2/d}};  
&& \|\phi_i\|_\infty \leq  c_1c_3^{\frac{d-1}{4}}  i^{\frac{d-1}{2d}}; \\
& \sup_{x \mathcal{M}} |\nabla \phi_i |(x) \leq  c_1c_3^{\frac{d+1}{4}}c_4  i^{\frac{d+1}{2d}}; & &  \sup_{x \mathcal{M}} |\nabla^2 \phi_i |(x) \leq  c_1c_3^{\frac{d+3}{4}} c_4c_5  i^{\frac{d+3}{2d}}.
\end{align*}
Therefore, we have
\begin{align*}
& \sum_{i=1}^{\infty}e^{-\lambda_i t/2}  \sup_{x \mathcal{M}} |\nabla \phi_i |(x) \|\phi_i\|_\infty \leq c_1^2 c_3^{d/2} c_4 \sum_{i=1}^{\infty}e^{-\frac{1}{2}c_2t_2  i^{2/d}} i; \\
& \sum_{i=1}^{\infty}e^{-\lambda_i t/2}  \sup_{x \mathcal{M}} |\nabla^2 \phi_i |(x)\|\phi_i\|_\infty \leq c_1^2 c_3^{(d+1)/2} c_4 c_5 \sum_{i=1}^{\infty}e^{-\frac{1}{2}c_2t_2  i^{2/d}} i^{(d+1)/d};
\end{align*}
By the integral test, we have
\begin{align*}
 \sum_{i=1}^{\infty}e^{-\frac{1}{2}c_2t_2  i^{2/d}} i^{(d+1)/d} \leq& e^{-\frac{1}{2}c_2t_2} +\int_{1}^\infty e^{-\frac{1}{2}c_2t_2 u^{2/d}} u^{(d+1)/d} du \leq  1+\int_{0 }^\infty e^{-\frac{1}{2}c_2t_2 u^{2/d}} u^{(d+1)/d} du \\
 =& 1+\frac{d}{2}(c_2 t_2/2)^{-(2d+1)/2} \int_{0}^\infty e^{-v}v^{d-\frac{1}{2}}dv \quad\quad  (\text{Let  } v=\frac{c_2t_2 u^{2/d}}{2}.)\\
 =& 1+\frac{d\Gamma(d+\frac{1}{2})}{2} (\frac{2}{c_2})^{(2d+1)/2} t_2^{-(2d+1)/2} =c'_2(t_2).
\end{align*}
Similarly,  $\sum_{i=1}^{\infty}e^{-\frac{1}{2}c_2t_2  i^{2/d}} i \leq 1+ \frac{d}{2} \Gamma(d) (\frac{2}{c_2})^d t_2^{-d}=c'_1(t_2).$ Finally, let 
$$C(t_2)= \max(c_1^2 c_3^{d/2} c_4 c'_1(t_2), c_1^2 c_3^{(d+1)/2} c_4 c_5 c'_2(t_2))$$
 which proves \eqref{proof: spectral series bound 3}.

 Based on the expression of $C(t_2)$, $C(t_2)$ depends on $t_2$, $d$,  the diameter of $\mathcal{M}$,  the volume of $\mathcal{M}$,  the injectivity radius, bounds of the Ricci curvature and its covariant derivative,  and the sectional curvature bounds of $\mathcal{M}$. 
\end{proof}

\subsection{Heat smoothing estimates on manifold}

For the zero-th derivative, we recall that $Q_t$ is contractive in $L^\infty$ norm.
\begin{lemma}\label{lemma:Qt-contractive-Linf}
For any $f \in L^\infty(\calM)$, $\forall t > 0$
$
\| Q_t f \|_\infty \le \| f \|_\infty.$
\end{lemma}

\begin{proof}
$|Q_t f(x) | \le  \int_\calM \calH_t(x,y) |f(y)|dV(y)  \le \| f\|_\infty \int_\calM \calH_t(x,y) dV(y)$.
\end{proof}

\vspace{5pt}
\noindent
$\bullet$ {\bf Proof of Lemma \ref{lemma:smoothing-estimates}}
\vspace{5pt}

To derive the smoothing estimates for $k$-th derivatives, 
we will utilize  standard Gaussian bounds for the heat kernel and its first two spatial derivatives \cite{grigor1997gaussian,hsu1999estimates} .
This will prove the bound up to some $O(1)$ time, and in the short time (near zero) the power in $t$ is sharp. 
For the long-time decay, since manifold is compact, we utilize a separate analysis based on the bounds in Appendix \ref{subsec:Ht-der-bounds}.

\begin{proof}[Proof of Lemma \ref{lemma:smoothing-estimates}]
For case (i), observe that $\int_\calM \calH_t(x,y) dV(y)=1$, and thus 
\[
\int_\calM \nabla_x \calH_t(x,y) \, dV(y) = 0, 
\quad 
\int_\calM \nabla_x^2 \calH_t(x,y)\, dV(y) = 0.
\]
As a result, for $k=1,2$,
\[
\nabla^k Q_t f(x) 
= \int_\calM \nabla^k_x \calH_t(x,y) f(y) dV(y)
= \int_\calM \nabla^k_x \calH_t(x,y) (f(y) - f(x))dV(y).
\]
Recall that $| T(x)|$ stands for the Frobenius norm of the tensor $T(x)$.
Then, we have 
\begin{equation}\label{eq:bound-der-Qtf-proof-1}
| \nabla^k Q_t f(x) | 
\le \int_\calM  | \nabla^k_x \calH_t(x,y) | | f(y) - f(x)| dV(y), \quad k = 1,2.
\end{equation}
Recall the  constant $t_1$ as defined in \eqref{eq:def-t1M}, $t_1 \le t_0 < 1$. We will consider when $t < t_1$ and $t \ge t_1$ separately.

a) $t < t_1$.
By \cite[Corollary  1.2]{hsu1999estimates}, there are constants $D_1$, $D_2$ that depend on $\calM$ s.t. 
\[
|\nabla_x^k \calH_t(x,y)| 
\le D_k  \big( \frac{d_\calM(x,y)}{t}  + \frac{1}{\sqrt{t}} \big)^k \calH_t(x,y), 
\quad \forall x, y \in \calM, \, \forall 0< t \le 1, \quad  k=1,2.
\]
Combined with the Gaussian bound of $\calH_t$ in Lemma \ref{lemma:Heat-short-time}(ii), 
we have
\[
|\nabla_x^k \calH_t(x,y)| 
\le D_k C_3 
	\big( \frac{d_\calM(x,y)}{t}  + \frac{1}{\sqrt{t}} \big)^k 
	 t^{-d/2} e^{- \frac{ d_\calM( x,y)^2}{ 5 t}}.
\]
Using $(a+b)^2 \leq 2(a^2+b^2)$ and $\sup_{u \geq 0} u^k e^{-u^2/30}  = (15k)^{k/2}e^{-k/2} \leq 12$ for $k=1,2$, one verifies that 
$\big(\frac{s}{t}+\frac{1}{\sqrt{t}}\big)^k e^{-s^2/(5t)} \leq 25\, t^{-k/2} 
e^{-s^2/(6t)}$
 for all $s \geq 0$, $t > 0$.
 As a result, let $C_3':= 25 \max \{ D_1, D_2\} C_3$, we have
\begin{equation}\label{eq:gaussian-envelope-der-Ht-proof-1}
|\nabla_x^k \calH_t(x,y)| \le C_3' t^{-(d+k)/2} e^{- \frac{ d_\calM( x,y)^2}{ 6 t}},
 \quad \forall x, y \in \calM, \, \forall t \le t_0 < 1, \, k=1,2,
\end{equation}
where $C_3'$ is a constant depending on $\calM$.

Because  $ t_1 \le t_0 < 1$, this holds when $t < t_1$.
Back to \eqref{eq:bound-der-Qtf-proof-1}, we then have
\[
| \nabla^k Q_t f(x) | 
\le C_3'  
\int_\calM  t^{-(d+k)/2} e^{- \frac{ d_\calM( x,y)^2}{ 6 t}}
			| f(y) - f(x)| dV(y). 
\]
By definition of $t_1$, when $ t < t_1$, $2 \delta(t) < \delta_0 < \xi$.
We then decompose the integral into when $y \in B_{2\delta(t)}(x)$ and  $y \notin B_{2\delta(t)}(x)$
and bound them respectively, as in the proof of Lemma \ref{lemma:Qeps-g-estimate}(ii).
Specifically, 
for any $y \notin B_{2\delta(t)}(x)$, 
one can verify that $t^{-d/2} e^{- { d_\calM( x,y)^2}/{( 6 t)}} \le t^5$, and then
\[
\int_{\calM \backslash B_{2 \delta(t)}(x)}   t^{-d/2} e^{- \frac{ d_\calM( x,y)^2}{ 6 t}} | f(y) - f(x) | dV(y)
\le 2 \|f\|_\infty \Vol(\calM) t^5;
\]
In addition, similarly as before,
\[
 \int_{ B_{2 \delta(t)}(x)}   t^{-d/2} e^{- \frac{ d_\calM( x,y)^2}{ 6 t}} | f(y) - f(x) | dV(y) 
 \le L_{0,\beta}(f) b_V |S^{d-1}|  m_{d,4} t^{\beta/2},
\]
where $m_{d,4}:= \int_0^{\infty}  e^{-u^2/6} 
	(1+u)u^{d-1} du$ is a constant determined by $d$.
Putting together, this gives
\[
\int_\calM  t^{-d/2} e^{- \frac{ d_\calM( x,y)^2}{ 6 t}}
			| f(y) - f(x)| dV(y)
	\le C_{1,a} \| f\|_{0,\beta}(t^5 + t^{\beta/2}) 
	\le 2 C_{1,a} \| f\|_{0,\beta} t^{\beta/2}, 
\]
since $t<1$, where $C_{1,a}$ is a constant depending on $\calM$ (and $d$). This proves that 
\[
| \nabla^k Q_t f(x) | \le 2  C_3'  C_{1,a}  \| f\|_{0,\beta} t^{-k/2 + \beta/2}, \quad \forall t < t_1, \quad k =1,2.
\]

b) $t \ge t_1$. We apply Lemma \ref{lemma:long-t-grad-Ht} with $t_2= t_1$. 
Note that $\lambda_1 > 0$ is a constant depending on $\calM$. By that for any $c>0$ and $k=1,2$,
 $\sup_{t > 0} t^{k/2} e^{- c t } \le c'$ for some $c'$,
 we have $\forall t \ge t_1$,
\begin{equation}\label{eq:upper-bound-der-Ht-proof-2}
    \left| \nabla_x^k \calH_t(x, y) \right| 
    	\leq C_{b} t^{-k/2},
    \quad 
     \forall x, y \in \calM, 
    \quad k = 1, 2, 
\end{equation}
where $C_{b} $ is a constant depending on $\calM$ (including the dependence on $t_1$).
Then, back to \eqref{eq:bound-der-Qtf-proof-1}, we have
\[
| \nabla^k Q_t f(x) | 
\le 2 \Vol(\calM) C_{b} \|f\|_\infty   t^{-k/2}.
\]
Note that when $t \ge t_1$, $t^{-\beta/2} \le t_1^{-\beta/2} \le t_1^{-1/2}$, we then have 
\[
| \nabla^k Q_t f(x) | 
\le 2 \Vol(\calM) C_{b}  t_1^{-1/2} \|f\|_\infty   t^{-k/2+\beta/2}, \quad \forall t \ge t_1, \quad k = 1,2.
\]

Combining a) and b), 
and by that the Frobenius norm upper bounds the operator norm of a tensor, we have that for any $t>0$,
\[
\|  \nabla^k Q_t f(x) \|_{op}
\le 
| \nabla^k Q_t f(x) |  \le C_{7,1} \| f\|_{0,\beta} t^{-k/2 + \beta/2}, \quad k=1,2,
\]
where $C_{7,1} = 2  \max\{ C_3'  C_{1,a} ,  \Vol(\calM) C_{b}  t_1^{-1/2} \}$ is a constant depending on $\calM$ only.

\vspace{5pt}

For case (ii), by that $\nabla^k Q_t f(x) 
= \int_\calM \nabla^k_x \calH_t(x,y) f(y) dV(y)$, we have
\[
| \nabla^k Q_t f(x) |
\le \| f\|_\infty 
\int_\calM | \nabla^k_x \calH_t(x,y)| dV(y).
\]
Similarly as above, when a) $ t < t_1$, we have  the Gaussian envelope of $|\nabla_x^k \calH_t(x,y)|$
as in \eqref{eq:gaussian-envelope-der-Ht-proof-1},
 and then
\[
\int_\calM | \nabla^k_x \calH_t(x,y)| dV(y)
\le C_3' \int_\calM  t^{-(d+k)/2} e^{- \frac{ d_\calM( x,y)^2}{ 6 t}}
		dV(y),
\]
where by the $B_{2 \delta(t)}$ ball truncation one can verify that
\[
\int_\calM  t^{-d/2} e^{- \frac{ d_\calM( x,y)^2}{ 6 t}}
		dV(y)
\le \Vol(\calM) t^5 + C_{2,a}',
\]
and $C_{2,a}'=b_V |S^{d-1}|  m_{d,4}$ is a constant depending on $\calM$ (and $d$). Putting together, and by that $t < 1$, we have
\[
| \nabla^k Q_t f(x) |
\le C_3' (\Vol(\calM)+ C_{2,a}')  \| f\|_\infty t^{-k/2}, 
\quad \forall t < t_1;
\]
When b) $t> t_1$, we still have \eqref{eq:upper-bound-der-Ht-proof-2}, and then
\[
| \nabla^k Q_t f(x) |
\le  \Vol(\calM) C_{b} \| f\|_\infty  t^{-k/2}, \quad \forall t \ge t_1.
\]
Putting together, we have that for any $t>0$,
\[
\| \nabla^k Q_t f(x) \|_{op} \le 
| \nabla^k Q_t f(x) |  \le C_{7,2} \| f\|_{\infty} t^{-k/2}, \quad k=1,2,
\]
where $C_{7,2} =   \max\{ C_3' (\Vol(\calM)+ C_{2,a}') ,  \Vol(\calM) C_{b} \}$ is a constant depending on $\calM$ only.

\vspace{5pt}

Finally, let $C_7 = \max\{ C_{7,1}, C_{7,2} \}$ finishes the proof of the lemma.
\end{proof}

\subsection{Other lemmas}

\vspace{5pt}
\noindent
$\bullet$ Bernstein inequality
\vspace{5pt}

\begin{lemma}[Bernstein inequality]\label{lemma:bernstein}
	Let $\{Y_j\}_{j=1}^N$ be i.i.d. random variables, $\E Y_j = 0$, 
	and for positive constants $L$ and $\nu$,
	$|Y_j| \leq L$,  $\E Y_j^2 \leq \nu$.
	Then, for any $\tau > 0$,
	\[ 
 \Pr \Big[ \frac{1}{N} \sum_{j=1}^N Y_j  \ge \tau \Big],  
 \Pr \Big[ \frac{1}{N} \sum_{j=1}^N Y_j  \le -\tau  \Big] 
 	\leq \exp \{  - \frac{N \tau^2}{ 2( \nu + {\tau L}/{3}) } \}.   \]
	In particular, when $\tau L \leq 3 \nu$, the right hand side is further bounded by $\exp\{-  \frac{N \tau^2}{4 \nu}\}$.
\end{lemma}

\vspace{5pt}
\noindent
$\bullet$ Metric and volume form comparisons
\vspace{5pt}

The following lemma establishes a crude comparison between Euclidean distance and manifold geodesic distance. 
We follow the statement of Lemma D.11 in \cite{tang2026adaptive}.
\begin{lemma}\label{lemma:manifold-reach}
    Suppose $\calM $ is a $C^2$ manifold isometrically embedded in $\R^D$ with reach $\tau > 0$.
    For any $x, y \in \calM$ with $\| x -  y \|< \tau/2$,
    we have $ d_\calM(x,y) \le 2 \| x  -  y \|$.
\end{lemma}
The refined local comparison of the two metrics is provided in the next lemma.
In addition, we also have a local volume form comparison.
The lemma is adapted from Lemma D.12  of \cite{tang2026adaptive}. Recall that $\xi$ is the injectivity radius of $\calM$.

\begin{lemma}\label{lemma:local-comparisons}
For any $x \in \calM$, we consider the normal coordinates at $x$ provided by $\exp_x:  T_x\calM  \cong  \R^d  \to \calM$, and on $\R^d$ we use the polar coordinates $(s,\theta)$. 

\begin{itemize}

\item[(i)] Local expansion of volume form.

Suppose $\calM$ is $C^2$, then $\forall  0 \le s < \xi$, $\theta \in S^{d-1} \subset T_x \calM$, 
\[
dV( \exp_x (s \theta)) = (1 + R_V(s  , \theta) ) s^{d-1} ds d\theta, \quad |R_V(s  , \theta)| \le c_{V,1} s^2,
\]
where the constant  $c_{V,1}$ depends on $d$ and the uniform bounds of up to the 2nd intrinsic derivative of the Riemannian metric $g$, and $c_{V,1}$ is uniform for $x \in \calM$. 

\item[(ii)] Local expansion of metrics.

Suppose $\calM$ is $C^{ 4}$,
then $\forall  0 \le s <  \min\{ 1, \xi \}$, $\theta \in S^{d-1} \subset T_x \calM$, 
 \begin{align*}
\|  \exp_x( s  \theta )- x \|^2  = s^2+  q_4 (x, \theta) s^4+R_{q,4}(s , \theta),
\quad 
|R_{q,4}(s , \theta)| \leq c_{q, 4} s^{5},
\end{align*}
where 
$q_4( x,\theta) = -\frac{1}{12} \|\Second_x(\theta,\theta)\|^2 $,
the constant $c_{q, 4}$ depends on  the  $\|\cdot\|_\infty$ norm of the up to 2nd covariant derivatives of the second fundamental form $\Second$,
and $c_{q, 4}$ is uniform for $x \in \calM$. 
\end{itemize}
\end{lemma}

\vspace{5pt}
\noindent
$\bullet$ Concentration of the degree 
\vspace{5pt}

Recall that $D_i = \sum_{j=1}^N W_{ij}$ and $\epsilon = \sigma^2$.
The following lemma is adapted from Lemma 6.1 of \cite{cheng2022eigen}. We use the fact that with $h(r) = \frac{1}{(4\pi)^{d/2}} e^{-r /4}$,
$ m_0 = 1$,
$ m_2 = 2$,
and $\tilde{m} = 1$.
The data density $p$ was assumed to be smooth in \cite{cheng2022eigen}, 
while the bound under $p \in C^3(\calM)$ suffices for our purpose, see more in the proof.

\begin{lemma}\label{lemma:Di-concen-eps2}
Under Assumptions \ref{assump:M}, \ref{assump:iid-data}, \ref{assump:sigma-large-N}, \ref{assump:p-C3},
\begin{itemize}
\item[(i)]
 When $N$ is large enough, w.p. $> 1- 2 N^{-9}$,  $D_i > 0$ for all $i$, and 
\begin{equation*}
\frac{1}{N} D_i 
=  \tilde{p}_\epsilon(x_i) +  O \Big( \epsilon^{3/2} + \sqrt{ \frac{\log N}{N \epsilon^{d/2}} } \Big),
\quad 
i =1,\cdots,  N,
\end{equation*}
where $\tilde{p}_\epsilon := p +  \epsilon ( \omega p + \Delta p)$
and 
$\omega \in C^{\infty}(\calM)$ is determined by manifold extrinsic coordinates.
In particular, when $p$ is uniform on $\calM$, we have
\begin{equation*}
\frac{1}{N} D_i 
=  p +  O \Big( \epsilon + \sqrt{ \frac{\log N}{N \epsilon^{d/2}} } \Big),
\quad 
i =1,\cdots,  N,
\end{equation*}

\item[(ii)] 
When $N$ is large enough, w.p. $> 1- 4 N^{-9}$, 
\begin{equation*}
\sum_{j=1}^N W_{i j} \frac{ 1}{  D_j} 
 =  1+  O \Big( \epsilon + \sqrt{ \frac{\log N}{N \epsilon^{d/2}} } \Big) , 
 \quad 
i =1,\cdots,  N.
\end{equation*}
\end{itemize}
In (i)(ii), the large-$N$ thresholds 
and the constants in big-$O$ 
depend on ($\calM, p)$ and are uniform for all $ i$.
The good event in (ii) is under the good event needed by (i).
\end{lemma}

\begin{proof}[Proof of Lemma \ref{lemma:Di-concen-eps2}]
For (i), the proof is the same as in that  of \cite[Lemma 6.1]{cheng2022eigen} except that we use a different estimate for the integral $\int_\calM K_\epsilon(x,y) p(y) dV(y)$: 
when $p \in C^{3}(\calM)$ and thus is in $C^{2,1}(\calM)$, 
and recall that $h(r) = \frac{1}{(4\pi)^{d/2}} e^{-r /4}$, 
by  \cite[Lemma 4.1]{tang2026adaptive}, we have
\[
\int_{\calM} K_\epsilon(x, y) p(y) dV(y) 
= p(x) +  \epsilon( \omega p + \Delta p)(x) + O(\epsilon^{3/2}) \| p\|_{C^3},
\]
where the constant in big-O depends on $\calM$. This gives rise to the $O(\epsilon^{3/2})$ error term in (i), replacing the $O(\epsilon^2)$ error in  \cite[Lemma 6.1]{cheng2022eigen}.

The rest of the proof, including that of (ii)  is the same as in  \cite[Lemma 6.1]{cheng2022eigen}.
In particular, (ii) uses that 
$
\frac{1}{N} D_i = p(x_i) ( 1+ O(\epsilon + \sqrt{ \frac{\log N}{N \epsilon^{d/2}} }))
$ uniformly, 
which  still follows from (i), since  $\tilde p_\epsilon(x) = p(x) + O(\epsilon)$ uniformly.
The bound for uniform $p$ in (i) follows as a special case. 
\end{proof}

\end{document}